\documentclass[a4paper]{cas-dc}

\PassOptionsToPackage{dvipsnames,table,xcdraw}{xcolor}
\usepackage[numbers]{natbib}
\usepackage{graphicx}
\usepackage[T1]{fontenc}
\usepackage{amsmath}
\usepackage{amsfonts}
\usepackage{amsthm}
\usepackage{booktabs}
\usepackage{tcolorbox}
\usepackage{algorithm}
\usepackage[noend]{algpseudocode}
\usepackage{subcaption}
\usepackage{multirow}
\usepackage{array}
\usepackage{makecell}
\usepackage{placeins}
\usepackage{balance}
\usepackage{diagbox}
\usepackage{float}

\def\tsc#1{\csdef{#1}{\textsc{\lowercase{#1}}\xspace}}
\tsc{WGM}
\tsc{QE}

\newtheoremstyle{problemstyle}%
  {1em}   
  {1em}   
  {}      
  {}      
  {\bfseries} 
  {}      
  {.5em}  
  {%
    \thmname{#1}%
    \thmnumber{ #2}%
    \thmnote{ #3}%
  }
\theoremstyle{problemstyle}
\newdefinition{definition}{Definition}
\newtheorem{problem}{Problem}
\newtheorem{theorem}{Theorem}

\newcommand{\methodacronym}{HyAL}

\begin{document}

\let\WriteBookmarks\relax
\def\floatpagepagefraction{1}
\def\textpagefraction{.001}

\shorttitle{}

\shortauthors{Printzios and Chatzilygeroudis}

\title [mode = title]{Hybrid Augmented Lagrangian Method for General Constrained Optimization via Evolutionary Algorithms}

\author[1]{Lampros Printzios}[orcid=0009-0009-9205-9715]
\ead{printzios_lampros@ac.upatras.gr}

\author[1]{Konstantinos Chatzilygeroudis}[orcid=0000-0003-3585-1027]
\ead{costashatz@upatras.gr}

\affiliation[1]{organization={Laboratory of Automation \& Robotics (LAR), Department of Electrical \& Computer Engineering, University of Patras},
                addressline={University Campus},
                city={Patras},
                postcode={GR-26504},
                state={Achaia},
                country={Greece}}

\begin{abstract}
    Constrained Optimization Problems are crucial in fields such as engineering, economics, and robotics, where high-dimensional search spaces and complex objectives and constraints are common. Numerical optimization methods, including Feasible Direction, Interior Point, and Sequential Quadratic Programming, have shown strong performance in finding feasible local optima, but require accurate analytical gradients and effective initialization, which can be challenging in real-world settings. Evolutionary Algorithms, on the other hand, offer gradient-free search and robustness to noisy landscapes, managing to detect global optima more often than numerical methods, but they often suffer from high computational costs and slow convergence. In this work, we propose a Hybrid Augmented Lagrangian (\methodacronym) method that integrates the AL framework's constraint-handling strengths with the exploratory power of population-based search. Our approach employs evolutionary techniques to solve subproblems within the AL iterations, promoting exploration and aiding in the escape from local optima. We conduct extensive experiments on benchmark optimization problems, comparing our method against state-of-the-art optimizers, including IPOPT and CMA-ES, and a standalone evolutionary optimization baseline (with constraint enforcement via penalties). In addition, we evaluate four population-based methods integrated within the AL framework to study the effect of different evolutionary solvers.
    Our results show that \methodacronym\ consistently produces high-quality solutions across the benchmark suite. It outperforms purely evolutionary approaches and scales more effectively to high-dimensional constrained problems, where evolutionary-only methods often struggle. \methodacronym\ also surpasses state-of-the-art numerical optimization algorithms on complex landscapes containing numerous local minima and saddle points.
\end{abstract}



\maketitle

\begin{figure*}[t]
  \centering
  \includegraphics[width=\textwidth]{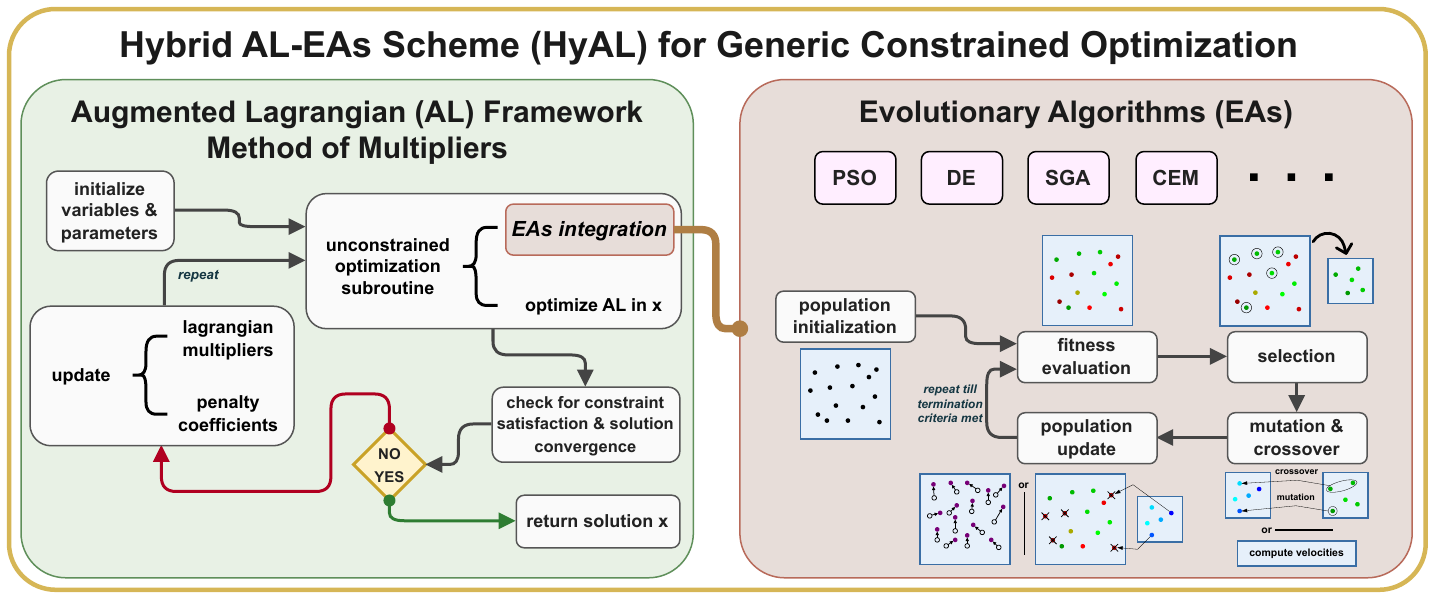}
  \caption{Hybrid Augmented Lagrangian (\methodacronym) Method Overview}
  \label{fig:HyAL_flowdiagram}
\end{figure*}

\section{Introduction \& Related Work} \label{sec:intro}
\emph{Constraint Optimization Problems} (COPs) appear in many diverse research fields and applications, including, among others, structural optimization, engineering design, VLSI design, economics, allocation and location problems, robotics, and optimal control problems \cite{floudas1990collection,himmelblau2018applied,tsiatsianas2025comparative,tsikelis2025ahmp,wensing2023optimization}. All these real-world problems are typically represented by a mathematical model that can contain both binary and continuous variables, as well as a set of linear and non-linear constraints, ranging from simple, low-dimensional numerical functions to high-dimensional noisy estimates.

The \emph{Numerical Optimization} literature \cite{himmelblau2018applied,nocedal2006numerical} has provided us with a wide range of powerful tools to address these problems. Methods such as \emph{Feasible Direction} (FD) \cite{BECK2022517}, \emph{Sequential Quadratic Programming} (SQP) \cite{nocedal2006numerical}, \emph{Generalized Gradient Descent} (GGD) \cite{norkin2020gradients}, and \emph{Interior Point} (IP) \cite{wachter2006implementation} are able to effectively solve COPs even in high dimensions, as well as handle problems in cases where the objective function or any of the constraints is nonlinear or/and non-convex. Additionally, Augmented Lagrangian (AL)-based techniques have emerged as powerful methods for handling constrained optimization problems, ensuring feasibility while maintaining efficient convergence \cite{bambade2022prox,birgin2014practical,jallet2025proxddp}. Despite their success, these methods require the availability of the exact gradient (and hessian) values of both the objective and constraint functions, which can be difficult to obtain in several real-world situations where only approximations are available. Moreover, it is well known that in practice, these methods require good initialization to achieve reasonable convergence rates.

On the other hand, many \emph{Evolutionary Algorithms} (EAs) have been proposed for solving general optimization and COPs \cite{chatzilygeroudis2021quality,chatzilygeroudis2023fast,chootinan2006constraint,d2021gga,elsayed2013particle,jain2013evolutionary,parsopoulos2005unified,sun2022particle,tsakonas2025vector} and have been applied successfully to numerous COPs, even in cases where the corresponding objective function values are corrupted by noise. Most of the methods do not require knowledge of the gradients \cite{jain2013evolutionary,parsopoulos2005unified}, while some of them attempt to improve convergence or performance by using the gradient information \cite{chatzilygeroudis2023fast,chootinan2006constraint}. Nevertheless, EAs are known to require many function evaluations to find high-performing solutions and can often fail to identify the global optimum.

The methods most closely related to our approach are those that integrate the AL framework within a Particle Swarm Optimization (PSO) procedure \cite{jansen2011_PSO,khoukhi2011non,sedlaczek2006_PSO}. However, existing studies are limited to small-scale problems and fail to provide insights into when and why this hybridization is effective. Moreover, these methods typically attempt to insert the AL framework into an evolutionary approach, while we deal with the problem from the opposite perspective.

In particular, we introduce a novel \textbf{Hybrid Augmented Lagrangian (\methodacronym) method} that combines the robustness of the AL framework with the exploratory power of population-based algorithms. Our approach (see Fig.~\ref{fig:HyAL_flowdiagram}) employs an evolutionary or population-based method to solve the unconstrained subproblem at each AL iteration. By leveraging the AL mechanism for constraint enforcement and utilizing stochastic search strategies to escape local optima, our method synergizes the strengths of both paradigms.

We evaluate our approach through extensive experiments on challenging benchmark problems. Furthermore, we compare four population-based methods combined with the AL framework and contrast them with established numerical methods (like IPOPT) and evolution strategies (e.g. CMA-ES). Our results provide a comprehensive analysis of the strengths and weaknesses of each approach, showcasing the conditions under which the proposed hybridization offers a significant performance advantage.

\section{Problem Formulation} \label{sec:problem_formulation}
The general constrained optimization problem is defined as the search for the decision-variable vector $\boldsymbol{x}\in\mathbb{R}^n$, that optimizes an objective function $f(\boldsymbol{x}):\mathbb{R}^n\to\mathbb{R}$. For $\boldsymbol{x}$ to be a feasible solution, it must satisfy a set of equality and inequality constraints, which can be written as vector functions $\boldsymbol{h}(\boldsymbol{x}) = \begin{bmatrix} h_1(\boldsymbol{x}), \ldots, h_l(\boldsymbol{x}) \end{bmatrix}^\top$ and $\boldsymbol{g}(\boldsymbol{x}) = \begin{bmatrix} g_1(\boldsymbol{x}), \ldots, g_m(\boldsymbol{x}) \end{bmatrix}^\top$ respectively. We are interested in constrained minimization problems, which are structured as shown below:
\begin{equation}
    \begin{aligned}
        \boldsymbol{x}_{\star} & = \arg\min_{\boldsymbol{x\in\mathbb{R}^n}} f(\boldsymbol{x}) \\
        \text{s.t.:} \quad     & h_i(\boldsymbol{x}) = 0, \quad i = 1, 2, \dots, l                  \\
                               & g_j(\boldsymbol{x}) \leq 0, \quad j = 1, 2, \dots, m               \\
    \end{aligned}
    \label{eq:constrained_optimization_problem}
\end{equation}

\noindent where $\bigl(\boldsymbol{x}_{\star}, f(\boldsymbol{x}_{\star})\bigr)$ is the global minimum point. We may also assume that $\boldsymbol{x}$ is bounded component-wise by lower limits $\boldsymbol{l_b} = \begin{bmatrix} l_{b,1}, \ldots, l_{b,n} \end{bmatrix}^\top \in (\mathbb{R}\cup\{-\infty\})^n$ and upper limits $\boldsymbol{u_b} = \begin{bmatrix} u_{b,1}, \ldots, u_{b,n} \end{bmatrix}^\top \in (\mathbb{R}\cup\{+\infty\})^n$, meaning that $x_i \in [l_{b,i}, u_{b,i}]$, $\forall i = 1, \ldots, n$, with the corresponding inequalities being $g_{\boldsymbol{l_b},i} = l_{b,i} - x_i \leq 0$ and $g_{\boldsymbol{u_b},i} = x_i - u_{b,i} \leq 0$, $\forall i = 1, \ldots, n$. In our implementation, these box constraints are enforced by projecting each variable onto its admissible interval, resulting in a computationally efficient procedure.

In order to tackle the COP, we construct the corresponding classical Lagrangian function as follows \cite{birgin2014practical,nocedal_numerical_optimization}:
\begin{equation}
    \mathcal{L}(\boldsymbol{x}, \boldsymbol{\lambda}, \boldsymbol{\mu}) = f(\boldsymbol{x}) + \boldsymbol{\lambda}^\top\boldsymbol{h}(\boldsymbol{x}) + \boldsymbol{\mu}^\top\boldsymbol{g}(\boldsymbol{x})
    \label{eq:lagrangian}
\end{equation}

\noindent where the Lagrange multipliers $\boldsymbol{\lambda}=[\lambda_1, \ldots, \lambda_l]^\top\in\mathbb{R}^l$ and $\boldsymbol{\mu}=[\mu_1, \ldots, \mu_m]^\top\in\mathbb{R}^m$ are the dual variables for the equality and inequality constraints.

\begin{theorem}[Karush-Kuhn-Tucker (or Saddle Point) Theorem for constrained optimization]
    If $(\boldsymbol{x}_{\star}, \boldsymbol{\lambda}_{\star}, \boldsymbol{\mu}_{\star})$ is a saddle point for the Lagrangian (local minimum along the primal direction of $\boldsymbol{x}$ and local maximum along the dual directions of $\boldsymbol{\lambda}$, $\boldsymbol{\mu}$), then it is also a local optimum for the constrained optimization problem~\eqref{eq:constrained_optimization_problem} \cite{Tabak1971OptimalCB}.
\end{theorem}

The Augmented Lagrangian \cite{birgin2014practical} can be written as:
\begin{equation}
    \begin{split}
        & \mathcal{L}_A(\boldsymbol{x}, \boldsymbol{\lambda}, \boldsymbol{\mu}; \boldsymbol{\rho}^{\lambda}, \boldsymbol{\rho}^\mu) = \mathcal{L}(\boldsymbol{x}, \boldsymbol{\lambda}, \boldsymbol{\mu})\; + \\
        & + \frac{1}{2}\sum_{i=1}^{l} \rho^\lambda_i h_i^2(\boldsymbol{x}) + \frac{1}{2}\sum_{j=1}^{m} \rho^\mu_j\max(0, g_j(\boldsymbol{x}))^2
    \end{split} \label{eq:augmented_lagrangian}
\end{equation}

\noindent where $\boldsymbol{\rho}^\lambda = \begin{bmatrix} \rho_1^\lambda, \ldots, \rho_l^\lambda \end{bmatrix}^\top \in \mathbb{R}^l$, $\boldsymbol{\rho}^\mu = \begin{bmatrix} \rho_1^\mu, \ldots, \rho_m^\mu \end{bmatrix}^\top\in\mathbb{R}^m$ denote the penalty parameters associated with the equality and inequality constraints, respectively. For sufficiently large penalty coefficients $\boldsymbol{\rho}$, these terms strongly penalize constraint violations and encourage feasible solutions.

\section{Hybrid Augmented Lagrangian Method} \label{sec:HyAL_method}

Our \methodacronym\ approach consists of two main parts: a) the \textbf{Augmented Lagrangian Framework} (AL) \emph{outer loop}, and b) the \textbf{Evolutionary Algorithms} (EAs) \emph{inner loop}. These are shown in Fig.~\ref{fig:HyAL_flowdiagram}, and are further discussed in the following sections. Contrary to previous methods \cite{chatzilygeroudis2023fast,jansen2011_PSO,khoukhi2011non,sedlaczek2006_PSO}, instead of incorporating ideas from the numerical optimization literature inside the evolutionary pipeline, we take the inverse approach and we embed evolutionary algorithms inside the AL framework.

In more detail, while the first step of the AL framework \eqref{eq:solution_estimation} is typically solved using a numerical optimization method, in \methodacronym, we propose to employ an evolutionary algorithm to address the unconstrained minimization subproblems (see Algo.~\ref{alg:AL_outer_loop_pseudocode}). Our approach leverages the fundamental property of the AL framework: \emph{if the unconstrained minimization subproblems are solved accurately, the overall method is guaranteed to converge to a feasible and, under appropriate assumptions, optimal point of the original constrained minimization problem} \cite{birgin2014practical}. We will experiment with four different evolutionary methods for solving the unconstrained subproblems.


\subsection{Augmented Lagrangian Framework} \label{sec:al_framework}
For the construction of our Hybrid Augmented Lagrangian Framework, we will rely on the Method of Multipliers \cite{birgin2014practical,hestenes1969_AL,nocedal_numerical_optimization,powell1978_AL}. Having the already stated Augmented Lagrangian formula, \eqref{eq:augmented_lagrangian}, we start by choosing some initial values $\boldsymbol{x}_0$, $\boldsymbol{\lambda}_0$, and $\boldsymbol{\mu}_0$ for the decision and dual variables. Then, $\forall k = 0, \ldots, iter^{AL}_{max}$, where $iter^{AL}_{max}$ is a predetermined maximum number of iterations, we:
\begin{enumerate}
    \item Calculate a new estimate for $\boldsymbol{x}$, solving the following unconstrained minimization subproblem. It can be approximately solved by a sufficiently accurate (meta)heuristic method, whose population may be initialized near the previous solution $\boldsymbol{x}_k$ to leverage \emph{warm-start} information.
        \begin{equation}
            \boldsymbol{x}_{k+1} \approx \arg\min_{\boldsymbol{x}}\mathcal{L}_A(\boldsymbol{x}, \boldsymbol{\lambda}_k, \boldsymbol{\mu}_k; \boldsymbol{\rho}_k^{\lambda}, \boldsymbol{\rho}_k^{\mu})
            \label{eq:solution_estimation}
        \end{equation}

        \noindent \item Update the Lagrange multipliers for equality and inequality constraints ($\odot$ denotes component-wise multiplication):
        \begin{equation}
          \begin{aligned}
            \boldsymbol{\lambda}_{k+1} & = \boldsymbol{\lambda}_k + \boldsymbol{\rho}_k^{\lambda} \odot \boldsymbol{h}(\boldsymbol{x}_{k+1})                \\
            \boldsymbol{\mu}_{k+1}     & = \max(\boldsymbol{0},\ \boldsymbol{\mu}_k + \boldsymbol{\rho}_k^{\mu} \odot \boldsymbol{g}(\boldsymbol{x}_{k+1}))
          \end{aligned} \label{eq:lagrange_multipliers_update}
        \end{equation}

        \noindent \item Update the penalty parameters using a scaling factor $\beta > 1$ and a predefined constant $\varepsilon_{\rho} > 0$ to check for violations of equality and inequality constraints (we also want to keep them bounded between some $\rho_{min}$ and $\rho_{max}$ limits).

        \begin{equation}
            \begin{aligned}
                \rho_{i,k+1}^{\lambda} & = \begin{cases} \beta \cdot \rho_{i,k}^{\lambda},   & \lvert h_i(\boldsymbol{x}_{k+1}) \rvert > \varepsilon_{\rho}   \\ 
                  \frac{\rho_{i,k}^{\lambda}}{\beta}, & \lvert h_i(\boldsymbol{x}_{k+1}) \rvert \le \varepsilon_{\rho} \\ 
                                           \end{cases} \\
                \rho_{j,k+1}^{\mu}     & = \begin{cases} \beta \cdot \rho_{j,k}^{\mu},   & \max(0,\ g_j(\boldsymbol{x}_{k+1})) > \varepsilon_{\rho}    \\ 
                  \frac{\rho_{j,k}^{\mu}}{\beta}, & \max(0,\ g_j(\boldsymbol{x}_{k+1})) \leq \varepsilon_{\rho} \\ 
                                           \end{cases} \\
                &\forall i = 1, \ldots, l \quad \text{and} \quad \forall j = 1, \ldots, m
            \end{aligned} \label{eq:penalty_parameters_update}
        \end{equation}
        
\end{enumerate}

The AL outer loop is terminated either when a maximum number of iterations $iter^{AL}_{max}$ is reached, or in a number of $k \leq iter^{AL}_{max}$ steps when the practical conditions below hold for $stag^{AL}_{max}$ consecutive iterations:
\begin{equation}
    \left|f(\boldsymbol{x}_k)-f(\boldsymbol{x}_{k-1})\right| \leq \varepsilon_s \quad (\text{or}\; \|\boldsymbol{x}_k-\boldsymbol{x}_{k-1}\|_2 \leq \varepsilon_s)
    \label{eq:solution_convergence_conditions}
\end{equation}
\begin{equation}
    \begin{aligned}
        \mathcal{V}_{c}(\boldsymbol{x}_k) & =
        \left\| \begin{bmatrix} \boldsymbol{h}(\boldsymbol{x}_k) \\ \max(\boldsymbol{0}, \boldsymbol{g}(\boldsymbol{x}_k)) \end{bmatrix} \right\|_2 \leq \\
        & \leq \sum_{i=1}^{l}\left|h_i(\boldsymbol{x}_k)\right|
        +
        \sum_{j=1}^{m}\max\!\left(0,g_j(\boldsymbol{x}_k)\right)
        \leq \varepsilon_c \\
    \end{aligned} \label{eq:constraint_violation_conditions}
\end{equation}

\noindent Here, $\varepsilon_s$ and $\varepsilon_c$ are sufficiently small tolerances to ensure that objective / solution - based convergence is achieved, Eq.~\eqref{eq:solution_convergence_conditions}, and that constraint violations do not exceed a certain threshold, Eq.~\eqref{eq:constraint_violation_conditions}. To avoid premature termination, a stagnation counter is used: when the convergence conditions are met, the counter is incremented, and the algorithm stops only if this occurs $stag^{AL}_{max}$ times in a row. An analogous stagnation mechanism also exists in the EAs inner loop.

To understand how the AL method of multipliers can lead to a candidate solution $\boldsymbol{x}_{\star}$ of the original COP in \eqref{eq:constrained_optimization_problem}, we need to look at the \textit{Karush-Kuhn-Tucker (KKT) first-order necessary optimality conditions}, and how our \methodacronym\ approach attempts to enforce them for a point $\boldsymbol{x}_{\star}$ inside the search space\footnote{In the theoretical KKT formulation below, the box bounds are treated explicitly as ordinary inequality constraints. In our practical implementation, however, they are enforced by projection onto the admissible intervals, while preserving the underlying theoretical formulation.}.

\begin{itemize}
    \item {\textbf{Stationarity} (at convergence to a feasible point, the subproblem \eqref{eq:solution_estimation} is solved to sufficient accuracy, so the Lagrangian gradient / stationarity residual vanishes):}
    $$
        \nabla_{\boldsymbol{x}} \mathcal{L}(\boldsymbol{x}_{\star}, \boldsymbol{\lambda}_{\star}, \boldsymbol{\mu}_{\star})
        = \nabla_{\boldsymbol{x}} f + \nabla_{\boldsymbol{x}} \boldsymbol{h} \cdot \boldsymbol{\lambda}_{\star} + \nabla_{\boldsymbol{x}} \boldsymbol{g} \cdot \boldsymbol{\mu}_{\star} \Big|_{\boldsymbol{x}_{\star}}
        = \boldsymbol{0}
    $$
    \item {\textbf{Primal Feasibility} (expresses constraint satisfaction):}
    $$
        \boldsymbol{h}(\boldsymbol{x}_{\star}) = \boldsymbol{0}, \qquad
        \boldsymbol{g}(\boldsymbol{x}_{\star}) \leq \boldsymbol{0}
    $$
    \item {\textbf{Dual Feasibility} (it is maintained by the projection in the update of $\boldsymbol{\mu}$ in \eqref{eq:lagrange_multipliers_update}):}
    $$
        \boldsymbol{\mu}_{\star} \geq \boldsymbol{0}
    $$
    \item {\textbf{Complementary Slackness} (inactive inequalities have zero multipliers, enforced via updates of $\boldsymbol{\mu}$ and $\boldsymbol{\rho^\mu}$ in \eqref{eq:lagrange_multipliers_update},~\eqref{eq:penalty_parameters_update}, as constraints are gradually satisfied):}
    $$
        \boldsymbol{\mu}_{\star} \odot \boldsymbol{g}(\boldsymbol{x}_{\star}) = \boldsymbol{0}
    $$
\end{itemize}

\begin{theorem}[KKT first-order optimality conditions under regularity conditions]
    Let $\boldsymbol{x}_\star$ be a local minimizer of the COP \eqref{eq:constrained_optimization_problem}, and assume that all objective and constraint functions are continuously differentiable. If a suitable regularity condition holds (such as LICQ, see Appendix~\ref{appendix:regularity_conditions}) at $\boldsymbol{x}_\star$, then there exist multipliers $\boldsymbol{\lambda}_\star\in\mathbb{R}^l$ and $\boldsymbol{\mu}_\star\in\mathbb{R}^m$ such that the KKT first-order optimality conditions hold \cite{birgin2014practical,nocedal_numerical_optimization}.
\end{theorem}


So, our \methodacronym\ approach aims to drive the iterates toward a KKT point of the original COP \eqref{eq:constrained_optimization_problem}. However, it is important to note that the KKT conditions are only necessary and not generally sufficient for local optimality. A KKT point may correspond to a local minimizer, but it may also be a local maximizer, or even a saddle point.
Our proposed approach is motivated by the observation that the AL outer loop progressively enforces feasibility, while the EAs inner loop provides effective exploration and exploitation of the search space. In practice, \methodacronym\ method often produces feasible stationary points, and the experimental results in Sec.~\ref{sec:results} show that these solutions usually coincide with locally optimal, and even globally optimal, benchmark solutions.

\begin{algorithm}
    \small
    \caption{\textcolor{blue}{Augmented Lagrangian (AL)} outer loop}
    \label{alg:AL_outer_loop_pseudocode}
    \begin{algorithmic}[1]
        \State \textbf{Input:} $\boldsymbol{x}_0,\ \boldsymbol{\lambda}_0,\ \boldsymbol{\mu}_0,\ \boldsymbol{\rho}^\lambda_0,\ \boldsymbol{\rho}^\mu_0$,\ \textcolor{ForestGreen}{\textbf{\textit{EA}}}
        \State \textbf{init.:} $(\boldsymbol{x}, \boldsymbol{\lambda}, \boldsymbol{\mu}) \gets (\boldsymbol{x}_0, \boldsymbol{\lambda}_0, \boldsymbol{\mu}_0),\ (\boldsymbol{\rho}^\lambda, \boldsymbol{\rho}^\mu) \gets (\boldsymbol{\rho}^\lambda_0, \boldsymbol{\rho}^\mu_0),\ stag \gets 0$
        \State \textbf{set:} $iter^{AL}_{max},\ stag^{AL}_{max},\ \beta,\ \varepsilon_\rho,\ \varepsilon_s,\ \varepsilon_c$
        \For{$k = 0$ to $iter^{AL}_{max}$}
            \State \textcolor{purple}{$\boldsymbol{x}_{new} \gets \textcolor{ForestGreen}{\textbf{EA}} \Big[\mathcal{L}_A(\boldsymbol{x}, \boldsymbol{\lambda}, \boldsymbol{\mu}; \boldsymbol{\rho}^\lambda, \boldsymbol{\rho}^\mu)\Big]$} $\quad \rightarrow \text {to solve \eqref{eq:solution_estimation}}$
            \State $\boldsymbol{\lambda}_{new},\ \boldsymbol{\mu}_{new} \gets \text {updated via \eqref{eq:lagrange_multipliers_update}}$
            \State $\boldsymbol{\rho}_{new}^{\lambda},\; \boldsymbol{\rho}_{new}^{\mu} \gets \text {updated via \eqref{eq:penalty_parameters_update}}$
            \If{conditions \eqref{eq:solution_convergence_conditions} and \eqref{eq:constraint_violation_conditions} are satisfied}
                \State $stag \gets stag + 1$
                \If{$stag \geq stag^{AL}_{max}$}
                    \State \textbf{break} \quad $\rightarrow$ \textit{convergence \& feasibility achieved}
                \EndIf
            \Else
                \State $stag \gets 0$
            \EndIf
            \State \textbf{update}: $\boldsymbol{x} \gets \boldsymbol{x}_{new},\ \boldsymbol{\lambda} \gets \boldsymbol{\lambda}_{new},\ \boldsymbol{\rho}^{\lambda} \gets \boldsymbol{\rho}_{new}^{\lambda},\ \boldsymbol{\rho}^{\mu} \gets \boldsymbol{\rho}_{new}^{\mu}$
        \EndFor
        \State \textbf{return} $\boldsymbol{x}_{new}$
    \end{algorithmic}
\end{algorithm}

\begin{algorithm}
    \small
    \caption{\textcolor{ForestGreen}{Evolutionary Algorithm (EA)} inner loop}
    \label{alg:EAs_inner_loop_pseudocode}
    \begin{algorithmic}[1]
        \State \textbf{Input:} $\textcolor{blue}{\mathcal{L}_A}(\boldsymbol{x})$
        \State \textbf{init.}: $P \gets \{\boldsymbol{x}_1,\ldots,\boldsymbol{x}_M\},\ (\boldsymbol{x}_{best},\phi) \gets (\boldsymbol{x}_1,+\infty),\ stag \gets 0$
        \State \textbf{set}: $iter^{EA}_{max},\ stag^{EA}_{max},\ \varepsilon_{EA}$
        \For{$k = 0$ to $iter^{EA}_{max}$}
            \If{$\min_{\boldsymbol{x}\in P}\textcolor{blue}{\mathcal{L}_A}(\boldsymbol{x})$ < $\phi$}
                \State \textcolor{purple}{$\boldsymbol{x}_{best} \gets \arg\min_{\boldsymbol{x}\in P}\textcolor{blue}{\mathcal{L}_A}(\boldsymbol{x})$} $\quad \rightarrow \text {to solve \eqref{eq:solution_estimation}}$
            \EndIf
            \State $\phi_{best} \gets \textcolor{blue}{\mathcal{L}_A}(\boldsymbol{x}_{best})$
            \If{$k > 0$ and $|\phi_{best} - \phi| \leq \varepsilon_{EA}$}
                \State $stag \gets stag + 1$
                \If{$stag \geq stag^{EA}_{max}$}
                    \State \textbf{break} \quad $\rightarrow$ \textit{\textcolor{ForestGreen}{EA} convergence achieved}
                \EndIf
            \Else
                \State $stag \gets 0$
            \EndIf
            \State \textbf{update}: $\phi \gets \phi_{best},\ P \gets {\text{\textcolor{ForestGreen}{EA} rule on current P}}$
        \EndFor
        \State \textbf{return} $\boldsymbol{x}_{best}$
    \end{algorithmic}
\end{algorithm}

\subsection{Evolutionary Algorithms} \label{sec:evolutionary_algorithms}

In the following, we briefly analyze the main aspects and properties of the Evolutionary Algorithms (EAs), namely SGA, CEM, DE, and PSO, that will be utilized to solve the inner loop's unconstrained optimization subroutines \eqref{eq:solution_estimation}. Our choices were made to include strong representatives from various subfields of evolutionary computation. The unconstrained subproblem arising in each AL iteration is solved only approximately by the EAs, as these methods are metaheuristics and cannot provide an exact, analytically derived minimizer.

In our implementation (see Algo.~\ref{alg:EAs_inner_loop_pseudocode}), the EAs inner loop is terminated when the best ever found AL objective, also called fitness value $\phi$, changes by less than a small tolerance $\varepsilon_{EA}$ for a fixed number of consecutive iterations $stag^{EA}_{max}$, or when a maximum number $iter^{EA}_{max}$ of iterations is reached.

\subsubsection{\textbf{Simple Genetic Algorithm}} \label{sec:SGA}

The \emph{Simple Genetic Algorithm} (SGA)~\cite{reed2000_SGA} is a stochastic EA inspired by the natural selection theory. It begins with a random $N$-dimensional, $M$-sized population of candidate solutions $P=\{\boldsymbol{x}_1(0), \ldots, \boldsymbol{x}_M(0)\}$. Then, it evaluates their fitness, and applies selection, mutation, and crossover processes to produce new individuals, favoring fitter members, while exploring new areas of the solution space via random variations. For the $k$-th iteration the algorithmic steps are:
\begin{enumerate}
  \item Evaluate each individual in the population using a fitness function (the Augmented Lagrangian $\mathcal{L}_A$ here), and construct the evaluation set $$\mathcal{F} = \{\mathcal{L}_A(\boldsymbol{x}_1(k)), \dots, \mathcal{L}_A(\boldsymbol{x}_M(k))\}.$$
  \item Sort the evaluation set $\mathcal{F}$. Then, select the best individuals ${P}_{elite} = \{\boldsymbol{x}_1^{elite}(k), \dots, \boldsymbol{x}_\Xi^{elite}(k)\}$ (smallest evaluations) for reproduction and the worst ones ${P}_{weak} = \{\boldsymbol{x}_1^{weak}(k), \dots, \boldsymbol{x}_\Xi^{weak}(k)\}$ (largest evaluations) for substitution, where $\Xi \leq M$.
  \item For each individual $\boldsymbol{x}_i^{elite}(k) \in \mathcal{P}_{elite}$:
        \begin{enumerate}
          \item Randomly select a different $\boldsymbol{x_j^{elite}}(k) \in \mathcal{P}_{elite}$ ($i \neq j \in \{1,2,\ldots,\Xi\}$).
          \item Calculate an offspring child individual by additive mutation and subtractive crossover on the chosen elite parent individuals (we use mutation and crossover operators inspired by \cite{vassiliades2018discovering}):
                \[
                    \begin{aligned}
                        \boldsymbol{x}_i^{offspring}(k) = &\ \boldsymbol{x}_i^{elite}(k) + \boldsymbol{\varepsilon}_{mut,k} + \\
                        & + \varepsilon_{cross,k} \; (\boldsymbol{x_j}^{elite}(k) - \boldsymbol{x}_i^{elite}(k))
                    \end{aligned}
                \]
                where $\boldsymbol{\varepsilon}_{mut,k} \in \mathbb{R}^N \sim \mathcal{N}(\boldsymbol{0}, \boldsymbol{\sigma}^2_{mut,k})$ is the mutation strength, and $\varepsilon_{cross,k} \in \mathbb{R} \sim \mathcal{N}(0, \sigma^2_{cross,k})$ scales the crossover.
          \item Replace $\boldsymbol{x}_i^{weak}$ with $\boldsymbol{x}_i^{offspring}$.
        \end{enumerate}
  \item Increase $k$ by 1 and repeat the steps until the stopping criterion is met.
\end{enumerate}
SGA excels at exploring the search space and producing diverse solutions. Adjusting the standard deviations $\boldsymbol{\sigma}_{mut} \in \mathbb{R}^N$, and $\sigma_{cross} \in \mathbb{R}$ we can control the amount of variation and, consequently, the exploration range for each iteration.

\subsubsection{\textbf{Cross Entropy Method}} \label{sec:CEM}

The \emph{Cross Entropy Method} (CEM)~\cite{botev2013_CEM} is a distribution-shaping evolutionary method that generates a completely new population each iteration, with stochastic properties determined by the previous generation's elite sampled members. The algorithmic process for the $k$-th iteration is:
\begin{enumerate}
  \item Generate a random population $P=\{\boldsymbol{x}_1(k), \ldots, \boldsymbol{x}_M(k)\}$ of dimension $N$ and size $M$ obeying the normal distribution $\mathcal{N}(\boldsymbol{\mu}_k, \boldsymbol{\sigma}^2_k)$, with mean value vector $\boldsymbol{\mu}_k$ and standard deviation vector $\boldsymbol{\sigma}_k$.
  \item Evaluate each individual in the population using a fitness function (the Augmented Lagrangian $\mathcal{L}_A$ here), and construct the evaluation set $$\mathcal{F} = \{\mathcal{L}_A(\boldsymbol{x}_1(k)), \dots, \mathcal{L}_A(\boldsymbol{x}_M(k))\}.$$
  \item Sort the evaluation set $\mathcal{F}$. Then, select the best individuals ${P}_{elite} = \{\boldsymbol{x}_1^{elite}(k), \dots, \boldsymbol{x}_\Xi^{elite}(k)\}$ with the smallest evaluations (where $\Xi \leq M$).
  \item Update the mean values vector for the next generation: $\boldsymbol{\mu}_{k+1} = \frac{1}{\Xi} \sum_{i}^{\Xi} \boldsymbol{x}_i^{elite}(k)$ (initially we have $\boldsymbol{\mu}_0 = \boldsymbol{x}_0$).
  \item Update (element-wise) the standard deviation for the next generation: $\boldsymbol{\sigma}_{k+1} = \sqrt{\frac{1}{\Xi} \sum_{i}^{\Xi} (\boldsymbol{x}_i^{elite}(k) - \boldsymbol{\mu}_k)^2}$.
  \item Increase $k$ by 1 and repeat the steps until the stopping criterion is met.
\end{enumerate}

The primary advantage of CEM is its focused search, favoring the exploitation of high-quality solutions over broad exploration, which enables precise identification of local minima. But this is also a limitation, since CEM needs to start close enough to global minima in order to detect them.

\subsubsection{\textbf{Differential Evolution}} \label{sec:DE}

\emph{Differential Evolution} (DE)~\cite{deng2021_DE} is a simple and effective population-based EA designed for globally optimizing continuous functions. It iteratively evolves the population by first generating trial vectors via mutation, using scaled differences between candidate solutions and other population members. It then applies crossover between each candidate and its corresponding trial vector, followed by a greedy selection step that retains only the improved solutions.

Formally, we start with a random initial population of $M$ candidate solutions $\boldsymbol{x}_q(0) \in \mathbb{R}^N$. At iteration $k$ they have positions $\boldsymbol{x}_q(k) \in \mathbb{R}^N$, with $q\in\{1, 2, \ldots, M\}$. For the mutation step we may follow many approaches. Here, we chose a \emph{DE/rand-to-best/1}-style update \cite{deng2021_DE}, that generates a new trial vector for each candidate by combining the current candidate and the best ever candidate, together with two other randomly selected individuals:
\[
    \boldsymbol{y}_q(k) = \boldsymbol{x}_q(k) + \lambda \left(\boldsymbol{x}^{best}(k) - \boldsymbol{x}_q(k)\right) + w \left(\boldsymbol{x}_{r_1}(k) - \boldsymbol{x}_{r_2}(k)\right)
\]
where $\boldsymbol{x}^{best}(k)$ denotes the best ever individual found so far, $r_1, r_2 \neq q$ are distinct random indices, $\lambda \in \mathbb{R}$ controls the pull toward the current best ever solution, and $w \in \mathbb{R}$ is the differential weight.

In the crossover step, a new trial vector $\boldsymbol{v}_q(k)$ is produced by recombining the mutant vector with the current solution:
\[
    v_{q,j}(k) = \begin{cases}
    y_{q,j}(k), & u_j \leq c_r \text{ or } j = j_r \\ x_{q,j}(k), & \text{otherwise}
  \end{cases},\quad j=0,1,\ldots,N
\]
where $c_r \in [0,1]$ is the crossover probability, $u_j \sim \mathcal{U}(0,1)$ is an independently sampled uniform random variable for each dimension $j$, and $j_r \in \{1,2,\ldots,N\}$ is a randomly selected index that ensures that at least one component is inherited from the mutant trial vector $\boldsymbol{y}_q(k)$.

Finally, the selection step updates the population by evaluating the fitness of the crossover trial $\boldsymbol{v}_q(k)$ and retaining the better solution, using a greedy replacement rule:
\[
  \boldsymbol{x}_q(k+1) = \begin{cases}
    \boldsymbol{v}_q(k), & f(\boldsymbol{v}_q(k)) \leq f(\boldsymbol{x}_q(k)) \\ \boldsymbol{x}_q(k), & \text{otherwise}
  \end{cases},\quad k=0,1,\ldots
\]

The above iterative process balances exploration and exploitation, allowing DE to effectively search complex optimization landscapes. The algorithm's performance can be fine-tuned by adjusting the parameters $\lambda$, $w$ and $c_r$ to suit the problem's domain.

\subsubsection{\textbf{Particle Swarm Optimization}} \label{sec:PSO}

\emph{Particle Swarm Optimization} (PSO)~\cite{vrahatispso,vrahatisbook} is inspired by collective swarming behavior, in which randomly initialized particles iteratively evaluate the objective function and update their positions based on their own experience and that of neighboring particles or the entire swarm. The two main strategies are \emph{local PSO} (PSO-LS) and \emph{global PSO} (PSO-GS), where particles track their personal best ever position and also, either their neighborhood best ever or the global best ever position, respectively.

Formally, the algorithm initializes $M$ particles with random positions $\boldsymbol{x}_q(0)\in\mathbb{R}^N$ and velocities $\boldsymbol{v}_q(0)\in\mathbb{R}^N$, with $q\in\{1, 2, \ldots, M\}$, that can explore the search space. The position update for particle $q$ at iteration $k+1$ is defined as:
\[
  \boldsymbol{x}_q(k+1) = \boldsymbol{x}_q(k) + \boldsymbol{v}_q(k+1),\quad k=0,1,\ldots
\]

\noindent In \emph{local PSO}, the velocity update is:
{\small{
    $$
    \boldsymbol{v}^l_q(k+1) = \chi\left[\boldsymbol{v}_q(k) + c_1 r^l_1(\boldsymbol{x}^b_q(k)-\boldsymbol{x}_q(k)) + c_2 r^l_2(\boldsymbol{x}^{lb}_q(k)-\boldsymbol{x}_q(k))\right]
    $$
}}

\noindent whereas for \emph{global PSO}, the velocity update becomes:
{\small{
    $$
        \boldsymbol{v}^g_q(k+1) = \chi\left[\boldsymbol{v}_q(k) + c_1 r^g_1(\boldsymbol{x}^b_q(k)-\boldsymbol{x}_q(k)) + c_2 r^g_2(\boldsymbol{x}^{gb}(k)-\boldsymbol{x}_q(k))\right]
    $$
}}

\noindent Here, $\boldsymbol{x}^b_q(k)$ denotes the particle’s personal best position, $\boldsymbol{x}^{lb}_q(k)$ its neighborhood best position, and $\boldsymbol{x}^{gb}(k)$ the global best position ever found by the entire swarm. The parameters $c_1, c_2 \in \mathbb{R}_{>0}$ are the cognitive and social acceleration coefficients respectively, $r^l_1, r^l_2, r^g_1, r^g_2 \sim U(0,1)$ are uniformly distributed random variables, and $\chi \in \mathbb{R}_{>0}$ is an inertia/constriction factor analogous to a gradient-descent learning rate.

The \emph{Unified Particle Swarm Optimizer} (UPSO) \cite{parsopoulos2005unified} merges local and global strategies into a single framework, so the total velocity update is computed as the following convex combination:
\[
  \boldsymbol{v}_q(k+1) = u \cdot \boldsymbol{v}^g_q(k+1) + (1-u) \cdot \boldsymbol{v}^l_q(k+1)
\]
where $u \in [0,1]$ is a parameter that blends the two strategies ($u = 0$ yields PSO-LS, while $u = 1$ yields PSO-GS). In our implementation, we add a small Gaussian noise term to avoid stagnation, and we clamp the velocities within reasonable bounds to ensure stability. We adopt the unified approach to balance exploration and exploitation.

\section{Experimental Setup} \label{sec:experiments}
Our experimental and evaluation framework is designed to investigate three main questions:
\begin{enumerate}
  \item Can \methodacronym\ effectively tackle diverse, challenging, and high-dimensional constrained optimization problems, in terms of both solution quality and runtime?
  \item Which EA performs best in this \methodacronym\ setting?
  \item Can \methodacronym\ outperform state-of-the-art numerical optimization and evolution strategies on complex COPs?
\end{enumerate}

\subsection{Benchmark Problems} \label{sec:benchmark_problems}

To address the evaluation objectives listed above, we assess \methodacronym\ on a diverse suite of 10 constrained optimization benchmark problems drawn from the literature \cite{handbook_optimization_test_problems,sfu_optimization_test_functions,wiki_optimization_test_functions}, as they are written in Appendix~\ref{appendix:benchmark_problems}, where both the global minima and the optimal Lagrange multipliers associated with the non-box inequality constraints are reported for each COP. We report only these multipliers because variable bounds are handled separately through projection onto the prescribed intervals. The selected problems span a wide range of difficulties and structural properties, exhibiting challenging and complex search spaces due to:
\begin{itemize}
  \item highly nonlinear and nonconvex objectives (e.g. Prob. \ref{prob:rosenbrock}, \ref{prob:mishra_bird}, \ref{prob:egg_holder}),
  \item nonconvex feasible regions arising from high-order polynomials or trigonometric constraints (e.g. Prob. \ref{prob:quadratic}, \ref{prob:linear_polynomial}, \ref{prob:gomez_levy}),
  \item highly multimodal landscapes with multiple local minima (e.g. Prob.~\ref{prob:drop_wave},~\ref{prob:griewank}),
  \item equality-constrained manifolds restricting the search (e.g. Prob.~\ref{prob:power},~\ref{prob:drop_wave},~\ref{prob:double_integrator}),
  \item high-dimensionality with strongly coupled variables (e.g. Prob.~\ref{prob:double_integrator}).
\end{itemize}




\subsection{Evaluation Methodology} \label{sec:testing_evaluation_methodology}

For each constrained optimization problem, we evaluate the proposed evolutionary methods and compare them against established baselines. In particular, we consider:

\begin{itemize}
    \item \textbf{AL-EAs comparison:} the four evolutionary optimizers used within the augmented Lagrangian framework, namely \textit{SGA-AL}, \textit{CEM-AL}, \textit{DE-AL}, and \textit{PSO-AL}, in order to identify the most effective evolutionary search strategy among the chosen representative EAs.
    \item \textbf{Penalty-based baseline:} a standalone PSO implementation with a penalty formulation, \textit{PSO-PEN}, to check the evolutionary search behavior when the Augmented Lagrangian mechanism is not leveraged.
    \item \textbf{Classical Numerical baseline:} the \textit{IPOPT} (Interior Point Optimizer) solver \cite{wachter2006_IPOPT}, included as a representative state-of-the-art gradient-based and nonlinear interior point numerical optimization method.
    \item \textbf{Evolution Strategy baseline:} \textit{CMA-ES} \cite{hansen2006cma}
    (Covariance Matrix Adaptation - Evolution Strategy), included to contrast AL-EAs against a state-of-the-art derivative-free and evolutionary optimizer.
\end{itemize}

The comparison focuses on convergence to the global minimum and satisfaction of the constraints. Each constrained problem is evaluated over 100 independent replicates/runs initialized from distinct starting points, except for Problem~\ref{prob:double_integrator}, which we evaluate over 25 independent runs. For two-dimensional problems, initial points are arranged on a rectangular grid of evenly spaced values, whereas for higher-dimensional problems, they are randomly sampled from a Gaussian distribution within the prescribed bounds.

Convergence is assessed using the thresholds $\varepsilon_s = 10^{-3}$ and $\varepsilon_c = 10^{-4}$, reflecting acceptable solution accuracy while prioritizing feasibility (in many cases, the final solution accuracy and constraint satisfaction exceed these nominal thresholds before termination, due to the inserted stagnation phases $stag^{AL}_{max} = 50$ and $stag^{EA}_{max} = 100$). For EAs we set the population size to 4096 individuals and the convergence tolerance to $\varepsilon_{EA} = 10^{-6}$. The maximum number of iterations for the AL outer loop and EAs inner loop are set to $iter_{\max}^{\text{AL}} = 2{,}000$ and $iter_{\max}^{\text{EA}} = 5{,}000$, respectively (even though EAs inner loop usually terminates much earlier, especially for the simpler COPs). Any further tunable parameters mentioned in Sec.~\ref{sec:evolutionary_algorithms} related to the exact implementations of the EAs, as well as the \methodacronym\ framework implementation, can be found in the source codes available, respectively, at \url{https://github.com/NOSALRO/algevo} and \url{https://github.com/NOSALRO/halm}. Additional details regarding the other test optimizers considered in our comparisons, are provided in Sec.~\ref{sec:comparing_diff_optimizers}. Crucially, all methods are evaluated using the same external success criteria $(\varepsilon_s, \varepsilon_c)$, ensuring that error metrics, convergence times, and success rates are directly comparable across all optimizers.

Experiments were conducted on a workstation equipped with an Intel Core Ultra 9 285K processor (24 physical cores, 36 MiB L3 cache, 64GB RAM)
running Arch Linux. All methods were evaluated under identical hardware and software conditions to ensure a fair comparison.

\section{Results} \label{sec:results}

Using the parameters and criteria defined in Sec.~\ref{sec:testing_evaluation_methodology}, we can now evaluate the various optimizers on all benchmark problems in Sec.~\ref{sec:benchmark_problems}, with known global minima $\boldsymbol{x}_\star$. An estimated solution $\bar{\boldsymbol{x}}_{opt}$ will be considered $\varepsilon_s$-optimal and $\varepsilon_c$-feasible, if the conditions \eqref{eq:solution_convergence_conditions} ($\lvert f(\boldsymbol{x}_\star) - f(\bar{\boldsymbol{x}}_{opt}) \rvert \leq \varepsilon_s$ or $\|\boldsymbol{x}_\star - \bar{\boldsymbol{x}}_{opt}\|_2 \leq \varepsilon_s$, depending on the chosen convergence criterion) and \eqref{eq:constraint_violation_conditions} ($\mathcal{V}_c(\bar{\boldsymbol{x}}_{opt}) \leq \varepsilon_c$) are satisfied, respectively. For the comparisons coming next, we also define $\boldsymbol{x}_{f}$ as the final solution estimation computed by the current optimizer.


\subsection{Comparing EAs inside the AL Framework} \label{sec:comparing_AL_EAs}

\begin{figure}
    \centering
    \begin{subfigure}{0.49\columnwidth}
        \centering
        \includegraphics[width=\linewidth]{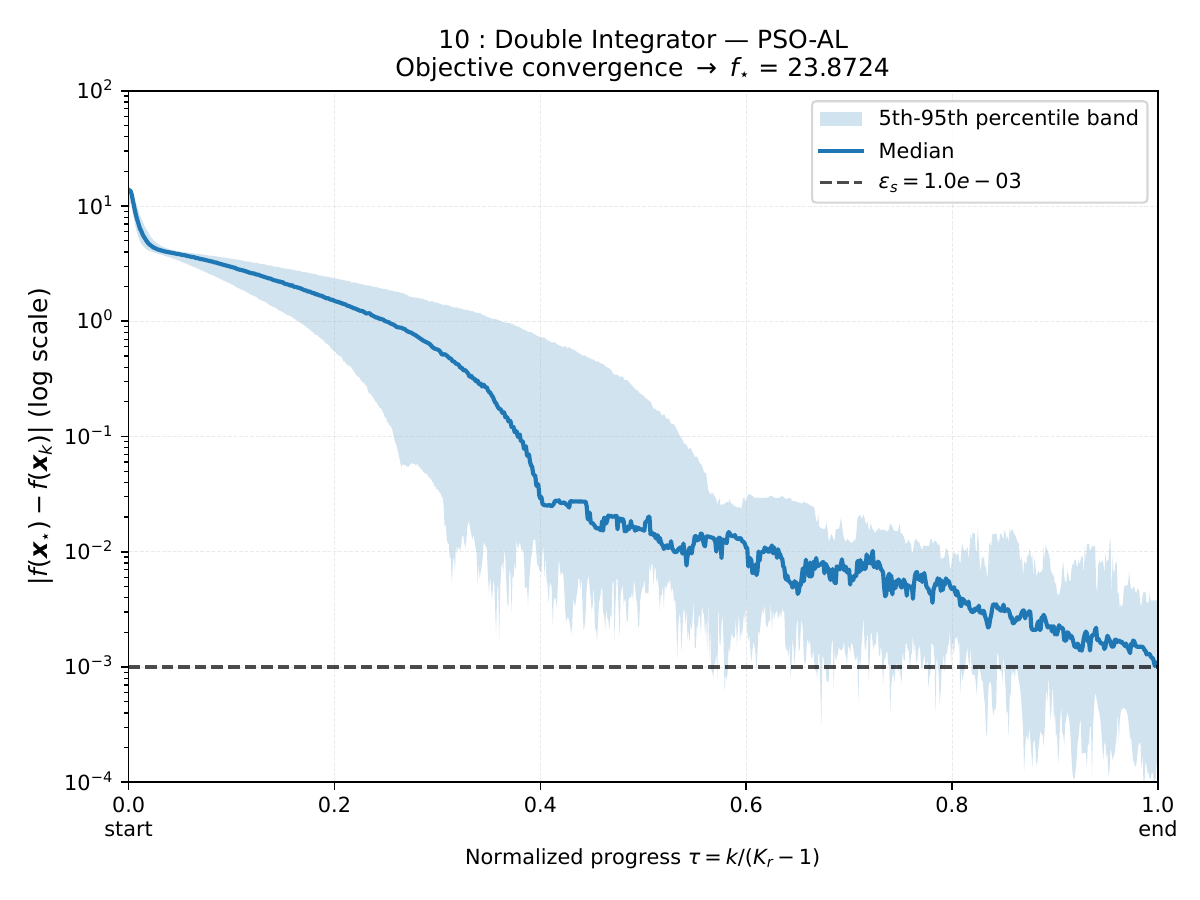}
        \caption{Objective error}
    \end{subfigure}
    \begin{subfigure}{0.49\columnwidth}
        \centering
        \includegraphics[width=\linewidth]{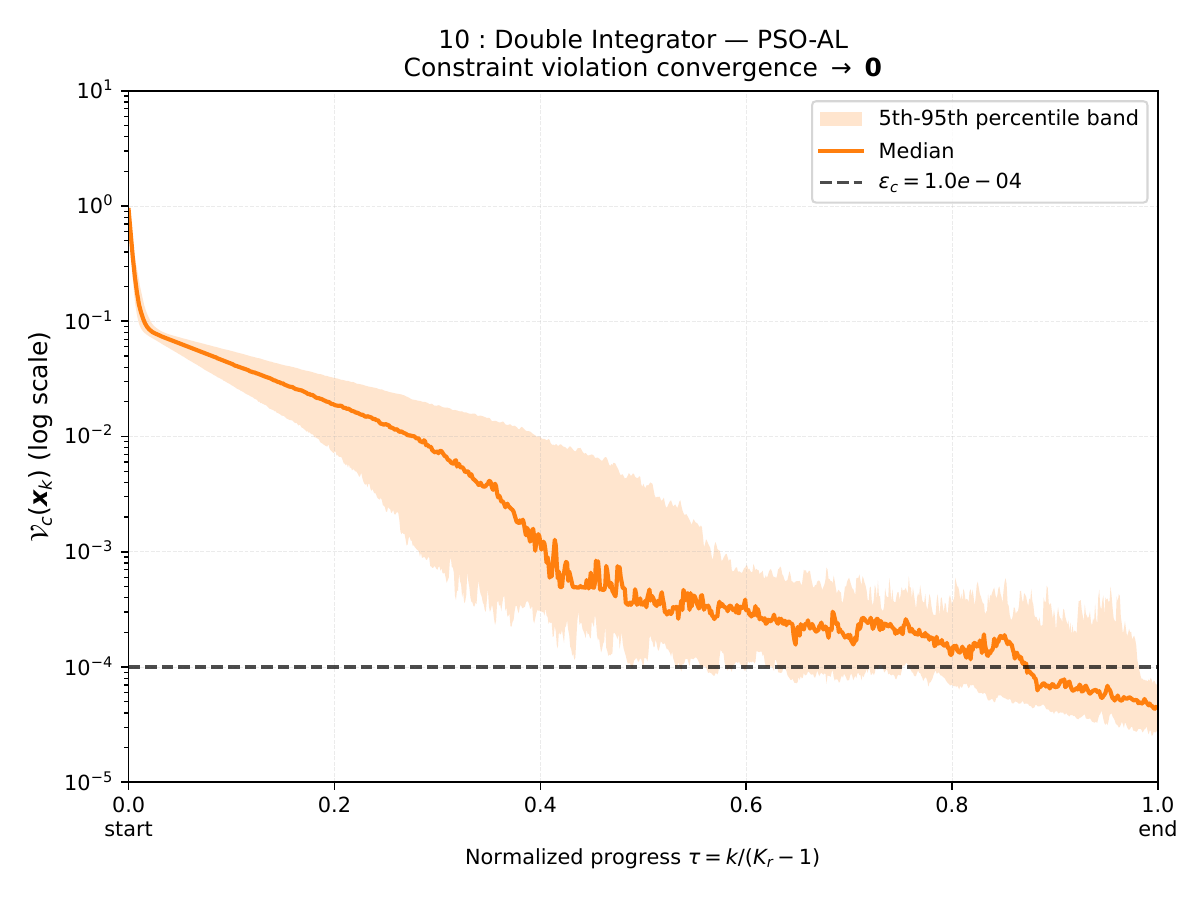}
        \caption{Constraint violation}
    \end{subfigure}
    \caption{Convergence behavior of PSO-AL on Prob.~\ref{prob:double_integrator}. The plots show the evolution of the objective error and constraint violation versus normalized progress, $\tau=k/(K_r-1)\in[0,1]$, where $k$ denotes the current AL outer loop iteration and $K_r$ the final iteration of run $r$, thus, $\tau=0$ denotes the start and $\tau=1$ the end of a run. The solid line shows the median across runs, the shaded region the 5th-95th percentile band, and the dashed horizontal line the nominal convergence threshold.}
    \label{fig:double_integrator_convergence}
\end{figure}

\begin{table*}[t]
\centering
\caption{\textbf{Aggregated performance comparison of all AL-EA optimizers (PSO-AL, DE-AL, SGA-AL, CEM-AL)} on the whole benchmark set (10 COPs), reporting the \textit{5th}, \textit{median}, and \textit{95th} percentiles of \textit{final solution absolute error}, \textit{objective absolute error} and \textit{constraint violation}. In addition, \textit{objective-value success} is evaluated under tolerances $\varepsilon_s=10^{-3}$ and $\varepsilon_c=10^{-4}$: \emph{Ever Hits} counts runs that satisfy the criteria at least once, \emph{$1^{\text{st}}$-Hit Med. Time [msec]} reports the median first-satisfaction time, and \emph{Final Hits} counts runs meeting the criteria at the final iteration. Best and worst results per problem are highlighted in green and red, respectively. For presentation purposes, numbers with absolute value less than $10^{-10}$ are considered insignificantly small and reported as $0$.}
\footnotesize
\setlength{\tabcolsep}{2pt}
\renewcommand{\arraystretch}{1.0}
\begin{tabular}{c|c|ccc|ccc|ccc|ccc}
\toprule
\multirow{2}{*}[-3pt]{\textbf{Problem}} & \multirow{2}{*}[-3pt]{\textbf{Optimizer}} & \multicolumn{3}{c|}{$\boldsymbol{\|x_\star - x_f\|_2}$} & \multicolumn{3}{c|}{$\boldsymbol{|f(x_\star) - f(x_f)|}$} & \multicolumn{3}{c|}{$\boldsymbol{\mathcal{V}_c(x_f)}$} & \multicolumn{3}{c}{\makecell{\textbf{\quad Obj. value Success \quad} \\ \textbf{$(\varepsilon_s = 10^{-3},\; \varepsilon_c = 10^{-4})$}}} \\
\cmidrule(lr){3-5}\cmidrule(lr){6-8}\cmidrule(lr){9-11}\cmidrule(lr){12-14}
& & \textbf{5\%}  & \textbf{Med.} & \textbf{95\%}
& \textbf{5\%}    & \textbf{Med.} & \textbf{95\%}
& \textbf{5\%}    & \textbf{Med.} & \textbf{95\%}
& \makecell{\textbf{Ever} \\ \textbf{Hits}} & \makecell{\textbf{$1^{\text{st}}$-Hit Med.} \\ \textbf{Time [msec]}} & \makecell{\textbf{Final} \\ \textbf{Hits}} \\
\midrule
\multirow{4}{*}{\makecell{\textbf{(\ref{prob:rosenbrock})}\\\textbf{Rosenbrock}\\\textbf{2D}}} & \textbf{PSO-AL} & 0 & \textcolor{ForestGreen}{\textbf{0}} & 0 & 0 & \textcolor{ForestGreen}{\textbf{0}} & 0 & 0 & \textcolor{ForestGreen}{\textbf{0}} & 2.8e-10 & \textcolor{ForestGreen}{\textbf{100/100}} & 7 & \textcolor{ForestGreen}{\textbf{100/100}} \\
 & \textbf{DE-AL} & 0 & \textcolor{ForestGreen}{\textbf{0}} & 0 & 0 & \textcolor{ForestGreen}{\textbf{0}} & 0 & 0 & \textcolor{ForestGreen}{\textbf{0}} & 0 & \textcolor{ForestGreen}{\textbf{100/100}} & \textcolor{ForestGreen}{\textbf{6}} & \textcolor{ForestGreen}{\textbf{100/100}} \\
 & \textbf{SGA-AL} & 0 & \textcolor{red}{\textbf{5.0e-6}} & 3.2e-5 & 0 & \textcolor{red}{\textbf{1.9e-9}} & 1.3e-8 & 1.9e-7 & \textcolor{red}{\textbf{4.8e-6}} & 3.4e-5 & \textcolor{ForestGreen}{\textbf{100/100}} & \textcolor{red}{\textbf{89.8}} & \textcolor{ForestGreen}{\textbf{100/100}} \\
 & \textbf{CEM-AL} & 0 & \textcolor{ForestGreen}{\textbf{0}} & 0 & 0 & \textcolor{ForestGreen}{\textbf{0}} & 0 & 0 & 6.1e-9 & 3.8e-8 & \textcolor{ForestGreen}{\textbf{100/100}} & 58.7 & \textcolor{ForestGreen}{\textbf{100/100}} \\
\midrule
\multirow{4}{*}{\makecell{\textbf{(\ref{prob:quadratic})}\\\textbf{Quadratic}}} & \textbf{PSO-AL} & 0 & \textcolor{ForestGreen}{\textbf{0}} & 6 & 0 & \textcolor{ForestGreen}{\textbf{0}} & 16 & 0 & \textcolor{ForestGreen}{\textbf{0}} & 0 & 78/100 & 20.4 & 78/100 \\
 & \textbf{DE-AL} & 0 & \textcolor{ForestGreen}{\textbf{0}} & 0 & 0 & \textcolor{ForestGreen}{\textbf{0}} & 0 & 0 & \textcolor{ForestGreen}{\textbf{0}} & 0 & \textcolor{ForestGreen}{\textbf{100/100}} & \textcolor{ForestGreen}{\textbf{20.2}} & \textcolor{ForestGreen}{\textbf{100/100}} \\
 & \textbf{SGA-AL} & 2.2e-5 & \textcolor{red}{\textbf{1.2e-4}} & 0.001 & 0.003 & \textcolor{red}{\textbf{0.016}} & 0.031 & 0 & \textcolor{ForestGreen}{\textbf{0}} & 1.1e-5 & \textcolor{red}{\textbf{2/100}} & \textcolor{red}{\textbf{2.0e+3}} & \textcolor{red}{\textbf{2/100}} \\
 & \textbf{CEM-AL} & 0 & \textcolor{ForestGreen}{\textbf{0}} & 11.225 & 0 & \textcolor{ForestGreen}{\textbf{0}} & 146 & 0 & \textcolor{red}{\textbf{4.5e-8}} & 4.4e-7 & 53/100 & 1.2e+3 & 53/100 \\
\midrule
\multirow{4}{*}{\makecell{\textbf{(\ref{prob:power})}\\\textbf{Power-Terms}\\\textbf{Objective}}} & \textbf{PSO-AL} & 3.3e-6 & \textcolor{ForestGreen}{\textbf{3.3e-6}} & 3.3e-6 & 0 & \textcolor{ForestGreen}{\textbf{0}} & 0 & 1.4e-9 & 4.3e-8 & 2.3e-7 & \textcolor{ForestGreen}{\textbf{100/100}} & 17.3 & \textcolor{ForestGreen}{\textbf{100/100}} \\
 & \textbf{DE-AL} & 3.3e-6 & \textcolor{ForestGreen}{\textbf{3.3e-6}} & 3.3e-6 & 0 & \textcolor{ForestGreen}{\textbf{0}} & 0 & 3.6e-10 & \textcolor{ForestGreen}{\textbf{1.5e-9}} & 2.0e-9 & \textcolor{ForestGreen}{\textbf{100/100}} & \textcolor{ForestGreen}{\textbf{16}} & \textcolor{ForestGreen}{\textbf{100/100}} \\
 & \textbf{SGA-AL} & 3.3e-6 & \textcolor{red}{\textbf{2.0e-5}} & 4.2e-5 & 0 & \textcolor{red}{\textbf{5.5e-5}} & 1.6e-4 & 6.9e-6 & \textcolor{red}{\textbf{3.6e-5}} & 9.3e-5 & \textcolor{ForestGreen}{\textbf{100/100}} & \textcolor{red}{\textbf{501.8}} & \textcolor{ForestGreen}{\textbf{100/100}} \\
 & \textbf{CEM-AL} & 3.3e-6 & \textcolor{ForestGreen}{\textbf{3.3e-6}} & 3.3e-6 & 0 & \textcolor{ForestGreen}{\textbf{0}} & 0 & 1.8e-8 & 2.0e-7 & 1.2e-6 & \textcolor{ForestGreen}{\textbf{100/100}} & 228.7 & \textcolor{ForestGreen}{\textbf{100/100}} \\
\midrule
\multirow{4}{*}{\makecell{\textbf{(\ref{prob:linear_polynomial})}\\\textbf{Linear}\\\textbf{Objective}}} & \textbf{PSO-AL} & 0 & \textcolor{ForestGreen}{\textbf{0}} & 0 & 0 & \textcolor{ForestGreen}{\textbf{0}} & 0 & 1.5e-10 & 1.6e-9 & 6.4e-9 & \textcolor{ForestGreen}{\textbf{100/100}} & 13.5 & \textcolor{ForestGreen}{\textbf{100/100}} \\
 & \textbf{DE-AL} & 0 & \textcolor{ForestGreen}{\textbf{0}} & 0 & 0 & \textcolor{ForestGreen}{\textbf{0}} & 0 & 0 & \textcolor{ForestGreen}{\textbf{4.2e-10}} & 1.2e-9 & \textcolor{ForestGreen}{\textbf{100/100}} & \textcolor{ForestGreen}{\textbf{11.4}} & \textcolor{ForestGreen}{\textbf{100/100}} \\
 & \textbf{SGA-AL} & 1.0e-5 & \textcolor{red}{\textbf{2.0e-5}} & 4.0e-5 & 9.5e-6 & \textcolor{red}{\textbf{2.0e-5}} & 5.0e-5 & 5.8e-6 & \textcolor{red}{\textbf{2.6e-5}} & 6.0e-5 & \textcolor{ForestGreen}{\textbf{100/100}} & \textcolor{red}{\textbf{299.3}} & \textcolor{ForestGreen}{\textbf{100/100}} \\
 & \textbf{CEM-AL} & 0 & \textcolor{ForestGreen}{\textbf{0}} & 0 & 0 & \textcolor{ForestGreen}{\textbf{0}} & 0 & 8.1e-9 & 3.2e-8 & 6.8e-8 & \textcolor{ForestGreen}{\textbf{100/100}} & 93.5 & \textcolor{ForestGreen}{\textbf{100/100}} \\
\midrule
\multirow{4}{*}{\makecell{\textbf{(\ref{prob:mishra_bird})}\\\textbf{Mishra's}\\\textbf{Bird}}} & \textbf{PSO-AL} & 3.9e-6 & \textcolor{ForestGreen}{\textbf{3.9e-6}} & 3.9e-6 & 4.6e-4 & \textcolor{ForestGreen}{\textbf{4.6e-4}} & 4.6e-4 & 0 & \textcolor{ForestGreen}{\textbf{0}} & 0 & \textcolor{ForestGreen}{\textbf{100/100}} & 8.3 & \textcolor{ForestGreen}{\textbf{100/100}} \\
 & \textbf{DE-AL} & 3.9e-6 & \textcolor{ForestGreen}{\textbf{3.9e-6}} & 3.9e-6 & 4.6e-4 & \textcolor{ForestGreen}{\textbf{4.6e-4}} & 4.6e-4 & 0 & \textcolor{ForestGreen}{\textbf{0}} & 0 & \textcolor{ForestGreen}{\textbf{100/100}} & \textcolor{ForestGreen}{\textbf{7}} & \textcolor{ForestGreen}{\textbf{100/100}} \\
 & \textbf{SGA-AL} & 3.9e-6 & \textcolor{ForestGreen}{\textbf{3.9e-6}} & 1.0e-5 & 4.6e-4 & \textcolor{ForestGreen}{\textbf{4.6e-4}} & 4.6e-4 & 0 & \textcolor{ForestGreen}{\textbf{0}} & 0 & \textcolor{ForestGreen}{\textbf{100/100}} & \textcolor{red}{\textbf{74}} & \textcolor{ForestGreen}{\textbf{100/100}} \\
 & \textbf{CEM-AL} & 3.9e-6 & \textcolor{ForestGreen}{\textbf{3.9e-6}} & 5.928 & 4.6e-4 & \textcolor{ForestGreen}{\textbf{4.6e-4}} & 85.928 & 0 & \textcolor{ForestGreen}{\textbf{0}} & 2.8e-9 & \textcolor{red}{\textbf{79/100}} & 57.3 & \textcolor{red}{\textbf{79/100}} \\
\midrule
\multirow{4}{*}{\makecell{\textbf{(\ref{prob:gomez_levy})}\\\textbf{Constrained}\\\textbf{Gomez-Levy}}} & \textbf{PSO-AL} & 4.0e-7 & \textcolor{ForestGreen}{\textbf{4.0e-7}} & 4.0e-7 & 1.5e-6 & \textcolor{ForestGreen}{\textbf{1.5e-6}} & 1.5e-6 & 0 & \textcolor{ForestGreen}{\textbf{0}} & 0 & \textcolor{ForestGreen}{\textbf{100/100}} & 8.5 & \textcolor{ForestGreen}{\textbf{100/100}} \\
 & \textbf{DE-AL} & 4.0e-7 & \textcolor{ForestGreen}{\textbf{4.0e-7}} & 4.0e-7 & 1.5e-6 & \textcolor{ForestGreen}{\textbf{1.5e-6}} & 1.5e-6 & 0 & \textcolor{ForestGreen}{\textbf{0}} & 0 & \textcolor{ForestGreen}{\textbf{100/100}} & \textcolor{ForestGreen}{\textbf{7.2}} & \textcolor{ForestGreen}{\textbf{100/100}} \\
 & \textbf{SGA-AL} & 9.9e-7 & \textcolor{red}{\textbf{4.0e-6}} & 9.2e-6 & 1.5e-6 & \textcolor{ForestGreen}{\textbf{1.5e-6}} & 1.5e-6 & 0 & \textcolor{ForestGreen}{\textbf{0}} & 0 & \textcolor{ForestGreen}{\textbf{100/100}} & \textcolor{red}{\textbf{86.4}} & \textcolor{ForestGreen}{\textbf{100/100}} \\
 & \textbf{CEM-AL} & 4.0e-7 & \textcolor{ForestGreen}{\textbf{4.0e-7}} & 4.0e-7 & 1.5e-6 & \textcolor{ForestGreen}{\textbf{1.5e-6}} & 1.5e-6 & 0 & \textcolor{ForestGreen}{\textbf{0}} & 0 & \textcolor{ForestGreen}{\textbf{100/100}} & 61.9 & \textcolor{ForestGreen}{\textbf{100/100}} \\
\midrule
\multirow{4}{*}{\makecell{\textbf{(\ref{prob:drop_wave})}\\\textbf{Drop-Wave}}} & \textbf{PSO-AL} & 0 & \textcolor{ForestGreen}{\textbf{0}} & 0 & 2.0e-6 & \textcolor{ForestGreen}{\textbf{2.0e-6}} & 2.0e-6 & 2.7e-10 & \textcolor{ForestGreen}{\textbf{5.3e-9}} & 2.2e-8 & \textcolor{ForestGreen}{\textbf{100/100}} & 8.4 & \textcolor{ForestGreen}{\textbf{100/100}} \\
 & \textbf{DE-AL} & 0 & \textcolor{ForestGreen}{\textbf{0}} & 0 & 2.0e-6 & \textcolor{ForestGreen}{\textbf{2.0e-6}} & 2.0e-6 & 8.7e-10 & 6.6e-9 & 1.8e-8 & \textcolor{ForestGreen}{\textbf{100/100}} & \textcolor{ForestGreen}{\textbf{8}} & \textcolor{ForestGreen}{\textbf{100/100}} \\
 & \textbf{SGA-AL} & 0 & \textcolor{red}{\textbf{2.2e-5}} & 9.3e-5 & 2.0e-6 & \textcolor{ForestGreen}{\textbf{2.0e-6}} & 2.0e-6 & 2.0e-6 & \textcolor{red}{\textbf{1.8e-5}} & 7.3e-5 & \textcolor{ForestGreen}{\textbf{100/100}} & \textcolor{red}{\textbf{87}} & \textcolor{ForestGreen}{\textbf{100/100}} \\
 & \textbf{CEM-AL} & 0 & \textcolor{ForestGreen}{\textbf{0}} & 0 & 2.0e-6 & \textcolor{ForestGreen}{\textbf{2.0e-6}} & 2.0e-6 & 1.6e-9 & 2.8e-8 & 1.4e-7 & \textcolor{ForestGreen}{\textbf{100/100}} & 68.6 & \textcolor{ForestGreen}{\textbf{100/100}} \\
\midrule
\multirow{4}{*}{\makecell{\textbf{(\ref{prob:egg_holder})}\\\textbf{Egg-Holder}}} & \textbf{PSO-AL} & 2.0e-4 & \textcolor{ForestGreen}{\textbf{2.0e-4}} & 2.0e-4 & 3.4e-4 & \textcolor{ForestGreen}{\textbf{3.4e-4}} & 3.4e-4 & 0 & \textcolor{ForestGreen}{\textbf{0}} & 0 & \textcolor{ForestGreen}{\textbf{100/100}} & 11.3 & \textcolor{ForestGreen}{\textbf{100/100}} \\
 & \textbf{DE-AL} & 2.0e-4 & \textcolor{ForestGreen}{\textbf{2.0e-4}} & 2.0e-4 & 3.4e-4 & \textcolor{ForestGreen}{\textbf{3.4e-4}} & 3.4e-4 & 0 & \textcolor{ForestGreen}{\textbf{0}} & 0 & \textcolor{ForestGreen}{\textbf{100/100}} & \textcolor{ForestGreen}{\textbf{6.5}} & \textcolor{ForestGreen}{\textbf{100/100}} \\
 & \textbf{SGA-AL} & 2.0e-4 & \textcolor{ForestGreen}{\textbf{2.0e-4}} & 2.0e-4 & 3.4e-4 & \textcolor{ForestGreen}{\textbf{3.4e-4}} & 3.4e-4 & 0 & \textcolor{ForestGreen}{\textbf{0}} & 0 & \textcolor{ForestGreen}{\textbf{100/100}} & \textcolor{red}{\textbf{82.5}} & \textcolor{ForestGreen}{\textbf{100/100}} \\
 & \textbf{CEM-AL} & 2.0e-4 & \textcolor{red}{\textbf{190.203}} & 334.287 & 3.4e-4 & \textcolor{red}{\textbf{70.692}} & 652.473 & 0 & \textcolor{ForestGreen}{\textbf{0}} & 0 & \textcolor{red}{\textbf{14/100}} & 57.3 & \textcolor{red}{\textbf{14/100}} \\
\midrule
\multirow{4}{*}{\makecell{\textbf{(\ref{prob:griewank})}\\\textbf{Griewank}\\\textbf{2D}}} & \textbf{PSO-AL} & 3.8e-9 & \textcolor{ForestGreen}{\textbf{1.2e-8}} & 1.7e-8 & 0 & \textcolor{ForestGreen}{\textbf{0}} & 0 & 0 & \textcolor{ForestGreen}{\textbf{0}} & 0 & \textcolor{ForestGreen}{\textbf{100/100}} & \textcolor{ForestGreen}{\textbf{8.5}} & \textcolor{ForestGreen}{\textbf{100/100}} \\
 & \textbf{DE-AL} & 6.3e-6 & 3.3e-5 & 7.4e-5 & 0 & 3.9e-10 & 1.9e-9 & 0 & \textcolor{ForestGreen}{\textbf{0}} & 0 & \textcolor{ForestGreen}{\textbf{100/100}} & 11.5 & \textcolor{ForestGreen}{\textbf{100/100}} \\
 & \textbf{SGA-AL} & 5.2e-6 & 1.7e-5 & 3.6e-5 & 0 & 1.0e-10 & 4.6e-10 & 0 & \textcolor{ForestGreen}{\textbf{0}} & 0 & \textcolor{ForestGreen}{\textbf{100/100}} & \textcolor{red}{\textbf{105.6}} & \textcolor{ForestGreen}{\textbf{100/100}} \\
 & \textbf{CEM-AL} & 6.2e-9 & \textcolor{red}{\textbf{5.437}} & 16.441 & 0 & \textcolor{red}{\textbf{0.007}} & 0.068 & 0 & \textcolor{ForestGreen}{\textbf{0}} & 0 & \textcolor{red}{\textbf{45/100}} & 87.7 & \textcolor{red}{\textbf{45/100}} \\
\midrule
\multirow{4}{*}{\makecell{\textbf{(\ref{prob:double_integrator})}\\\textbf{Double}\\\textbf{Integrator}}} & \textbf{PSO-AL} & 0.011 & \textcolor{ForestGreen}{\textbf{0.016}} & 0.024 & 1.0e-4 & \textcolor{ForestGreen}{\textbf{0.001}} & 0.004 & 2.8e-5 & \textcolor{ForestGreen}{\textbf{4.4e-5}} & 7.0e-5 & \textcolor{ForestGreen}{\textbf{25/25}} & \textcolor{ForestGreen}{\textbf{4.7e+4}} & \textcolor{ForestGreen}{\textbf{11/25}} \\
 & \textbf{DE-AL} & 32.121 & \textcolor{red}{\textbf{41.193}} & 45.908 & 76.963 & \textcolor{red}{\textbf{286.209}} & 607.748 & 4.0266 & \textcolor{red}{\textbf{5.7035}} & 7.6766 & \textcolor{red}{\textbf{0/25}} & - & \textcolor{red}{\textbf{0/25}} \\
 & \textbf{SGA-AL} & 2.074 & 2.301 & 3.001 & 0.03 & 0.352 & 0.914 & 0.0654 & 0.076 & 0.0857 & \textcolor{red}{\textbf{0/25}} & - & \textcolor{red}{\textbf{0/25}} \\
 & \textbf{CEM-AL} & 5.328 & 7.875 & 10.9 & 1.187 & 3.048 & 6.414 & 0.0018 & 0.0039 & 0.0098 & \textcolor{red}{\textbf{0/25}} & - & \textcolor{red}{\textbf{0/25}} \\
\bottomrule
\end{tabular}
\label{tab:metrics_stats_AL_EAs}
\end{table*}

Tab.~\ref{tab:metrics_stats_AL_EAs} summarizes the performance of all AL-EA optimizers across the whole benchmark set, presenting useful quantitative metrics. For each problem, we report the $5^{\text{th}}$, $50^{\text{th}}$ (median), and $95^{\text{th}}$ percentiles of key metrics, such as the final \textit{solution error} $\|\boldsymbol{x}_\star - \boldsymbol{x}_{f}\|_2$, the \textit{objective error} $\lvert f(\boldsymbol{x}_\star) - f(\boldsymbol{x}_{f}) \rvert$, and the \textit{total constraint violation} $\mathcal{V}_c(\boldsymbol{x}_{f})$ \eqref{eq:constraint_violation_conditions}. Additionally, Tab.~\ref{tab:metrics_stats_AL_EAs} provides objective-value success metrics under the nominal $(\varepsilon_s, \varepsilon_c) = (10^{-3}, 10^{-4})$ thresholds. Specifically, we record the number of runs that satisfy the convergence criteria at least once (\emph{Ever Hits}), the median time required to first meet these criteria (\emph{First-Hit Median Time}), and the number of runs that satisfy the criteria until the final iteration (\emph{Final Hits}). Best and worst results for each problem-metric pair are highlighted in green and red, respectively, facilitating immediate visual comparison across all AL-EA optimizers.

In general, we conclude that PSO-AL and DE-AL achieve the smallest solution and constraint errors, with very low variance and almost perfect success rates across all benchmarks (exceptions include Prob.~\ref{prob:quadratic} for PSO-AL and Prob.~\ref{prob:double_integrator} for DE-AL). Meanwhile, CEM-AL and SGA-AL typically yield higher solution errors and lower convergence success (Prob.~\ref{prob:quadratic},~\ref{prob:mishra_bird},~\ref{prob:egg_holder},~\ref{prob:griewank},~\ref{prob:double_integrator} being characteristic examples for CEM-AL and Prob.~\ref{prob:quadratic},~\ref{prob:double_integrator} for SGA). This is to be expected, since SGA-AL is a less sophisticated algorithm than PSO-AL and DE-AL, featuring simpler crossover and mutation rules, and CEM-AL is inherently more exploitative than exploratory, making it prone to local minima. Notably, in the high-dimensional case of Problem~\ref{prob:double_integrator}, only PSO-AL achieves decent performance, revealing its superior exploration capabilities relative to the other AL-EAs. In this case, the lower number of \textit{Final Hits} relative to \textit{Ever Hits} is caused only by the solution accuracy requirement, as the constraint violation threshold is satisfied in all runs. More generally, a gap between \textit{Ever Hits} and \textit{Final Hits} may reflect oscillatory or unstable behavior, where the optimizer identifies the global optimum but loses track of it and fails to retain it until termination. This phenomenon is not observed for the proposed AL-EAs in any of the considered benchmark problems (with the only exception as we said of Prob.~\ref{prob:double_integrator} for PSO-AL, where still the solution accuracy remains very good), further showcasing their robust and stable convergence behavior.

\begin{table*}
\centering
\caption{\textbf{Aggregated time comparison of all AL-EA optimizers} on the whole benchmark set (10 COPs), reporting the \textit{median total running time [msec]} for each optimizer-problem pair. Best and worst times per problem are highlighted in green and red, respectively.}
\footnotesize
\setlength{\tabcolsep}{4pt}
\renewcommand{\arraystretch}{1.0}
\resizebox{\textwidth}{!}{%
\begin{tabular}{l|l|cccccccccc}
\toprule
\multirow{5}{*}{\rotatebox[origin=c]{90}{\makecell{\textbf{Total Running}\\\textbf{Times [msec]}\\\textbf{(Median)}}}}
& \diagbox{\textbf{Optimizer}}{\textbf{Problem}} & \makecell{\textbf{(\ref{prob:rosenbrock})}\\\textbf{Rosenbrock}\\\textbf{2D}} & \makecell{\textbf{(\ref{prob:quadratic})}\\\textbf{Quadratic}} & \makecell{\textbf{(\ref{prob:power})}\\\textbf{Power-Terms}\\\textbf{Objective}} & \makecell{\textbf{(\ref{prob:linear_polynomial})}\\\textbf{Linear}\\\textbf{Objective}} & \makecell{\textbf{(\ref{prob:mishra_bird})}\\\textbf{Mishra's}\\\textbf{Bird}} & \makecell{\textbf{(\ref{prob:gomez_levy})}\\\textbf{Constrained}\\\textbf{Gomez-Levy}} & \makecell{\textbf{(\ref{prob:drop_wave})}\\\textbf{Drop-Wave}} & \makecell{\textbf{(\ref{prob:egg_holder})}\\\textbf{Egg-Holder}} & \makecell{\textbf{(\ref{prob:griewank})}\\\textbf{Griewank}\\\textbf{2D}} & \makecell{\textbf{(\ref{prob:double_integrator})}\\\textbf{Double}\\\textbf{Integrator}} \\
\cmidrule(l){2-12}
& \makecell{\textbf{PSO-AL}} & 158.2 & \textcolor{ForestGreen}{\textbf{208.7}} & \textcolor{ForestGreen}{\textbf{200.1}} & \textcolor{ForestGreen}{\textbf{199.8}} & 184.5 & 200.5 & \textcolor{ForestGreen}{\textbf{184.1}} & 205.4 & \textcolor{ForestGreen}{\textbf{184.9}} & \textcolor{ForestGreen}{\textbf{5.4e+4}} \\
& \makecell{\textbf{DE-AL}} & \textcolor{ForestGreen}{\textbf{147}} & 307.3 & 214 & 205 & \textcolor{ForestGreen}{\textbf{170.4}} & \textcolor{ForestGreen}{\textbf{175.7}} & 190.8 & \textcolor{ForestGreen}{\textbf{158}} & 279.5 & 3.8e+5 \\
& \makecell{\textbf{SGA-AL}} & \textcolor{red}{2.0e+3} & \textcolor{red}{3.4e+3} & \textcolor{red}{6.9e+3} & \textcolor{red}{2.0e+3} & \textcolor{red}{1.7e+3} & \textcolor{red}{2.1e+3} & \textcolor{red}{2.0e+3} & \textcolor{red}{1.9e+3} & \textcolor{red}{1.7e+3} & 9.9e+5 \\
& \makecell{\textbf{CEM-AL}} & 1.4e+3 & 3.4e+3 & 5.1e+3 & 1.6e+3 & 1.4e+3 & 1.5e+3 & 1.5e+3 & 1.3e+3 & 1.4e+3 & \textcolor{red}{5.9e+6} \\
\bottomrule
\end{tabular}
}
\label{tab:running_times_AL_EAs}
\end{table*}

\begin{table*}[!htb]
\centering
\caption{\textbf{Aggregated AL framework performance comparison of all AL-EA optimizers} on all 10 COPs, reporting the \textit{median final Lagrange-multipliers $\ell_2$-norm errors}. Best and worst median errors per problem are highlighted in green and red, respectively.}
\footnotesize
\setlength{\tabcolsep}{4pt}
\renewcommand{\arraystretch}{1.0}
\resizebox{\textwidth}{!}{%
\begin{tabular}{l|l|cccccccccc}
\toprule
\multirow{5}{*}{\rotatebox[origin=c]{90}{\makecell{\textbf{$\ell^2$-Norm Errors}\\\textbf{(Median) of}\\\textbf{Lag. Multipliers}}}}
& \diagbox{\textbf{Optimizer}}{\textbf{Problem}} & \makecell{\textbf{(\ref{prob:rosenbrock})}\\\textbf{Rosenbrock}\\\textbf{2D}} & \makecell{\textbf{(\ref{prob:quadratic})}\\\textbf{Quadratic}} & \makecell{\textbf{(\ref{prob:power})}\\\textbf{Power-Terms}\\\textbf{Objective}} & \makecell{\textbf{(\ref{prob:linear_polynomial})}\\\textbf{Linear}\\\textbf{Objective}} & \makecell{\textbf{(\ref{prob:mishra_bird})}\\\textbf{Mishra's}\\\textbf{Bird}} & \makecell{\textbf{(\ref{prob:gomez_levy})}\\\textbf{Constrained}\\\textbf{Gomez-Levy}} & \makecell{\textbf{(\ref{prob:drop_wave})}\\\textbf{Drop-Wave}} & \makecell{\textbf{(\ref{prob:egg_holder})}\\\textbf{Egg-Holder}} & \makecell{\textbf{(\ref{prob:griewank})}\\\textbf{Griewank}\\\textbf{2D}} & \makecell{\textbf{(\ref{prob:double_integrator})}\\\textbf{Double}\\\textbf{Integrator}} \\
\cmidrule(l){2-12}
& \makecell{\textbf{PSO-AL}} & 4.5e-9 & \textcolor{red}{2.3e+3} & \textcolor{ForestGreen}{\textbf{0.011}} & \textcolor{ForestGreen}{\textbf{5.8e-4}} & \textcolor{ForestGreen}{\textbf{0}} & \textcolor{ForestGreen}{\textbf{0}} & 7.1e-9 & \textcolor{ForestGreen}{\textbf{0}} & \textcolor{ForestGreen}{\textbf{0}} & \textcolor{ForestGreen}{\textbf{0.213}} \\
& \makecell{\textbf{DE-AL}} & \textcolor{ForestGreen}{\textbf{0}} & \textcolor{red}{2.3e+3} & \textcolor{ForestGreen}{\textbf{0.011}} & \textcolor{ForestGreen}{\textbf{5.8e-4}} & \textcolor{ForestGreen}{\textbf{0}} & \textcolor{ForestGreen}{\textbf{0}} & \textcolor{ForestGreen}{\textbf{3.6e-9}} & \textcolor{ForestGreen}{\textbf{0}} & \textcolor{ForestGreen}{\textbf{0}} & \textcolor{red}{4.1e+5} \\
& \makecell{\textbf{SGA-AL}} & \textcolor{red}{0.001} & \textcolor{red}{2.3e+3} & \textcolor{red}{0.467} & \textcolor{red}{0.087} & \textcolor{ForestGreen}{\textbf{0}} & \textcolor{ForestGreen}{\textbf{0}} & \textcolor{red}{6.6e-4} & \textcolor{ForestGreen}{\textbf{0}} & \textcolor{ForestGreen}{\textbf{0}} & 1.0e+4 \\
& \makecell{\textbf{CEM-AL}} & 1.5e-6 & \textcolor{ForestGreen}{\textbf{15.246}} & \textcolor{ForestGreen}{\textbf{0.011}} & 7.4e-4 & \textcolor{ForestGreen}{\textbf{0}} & \textcolor{ForestGreen}{\textbf{0}} & 3.2e-8 & \textcolor{ForestGreen}{\textbf{0}} & \textcolor{ForestGreen}{\textbf{0}} & 1.1e+3 \\
\bottomrule
\end{tabular}
}
\label{tab:multipliers_errors_AL_EAs}
\end{table*}

Regarding the execution times of the AL-EAs, we observe that, first of all, EAs reach the desired solution precision rapidly, in a matter of milliseconds, as indicated by the \textit{$1^{\text{st}}$-Hit Time Median} metric in Tab.~\ref{tab:metrics_stats_AL_EAs}. However, their overall running times, as reported in Tab.~\ref{tab:running_times_AL_EAs}, tend to be much higher due to the stagnation phases incorporated in both the AL outer loop and the EAs inner loop. Median total times highlight that PSO-AL and DE-AL not only converge more reliably, but also consistently faster than the other evolutionary approaches SGA-AL and CEM-AL, typically in around 0.2 seconds (for all COPs instead of Prob.~\ref{prob:double_integrator}, where the \textit{"curse of dimensionality"} increases the required fitness function evaluations and the computational cost).

We have established that the evolutionary methods are effective at recovering the global minimum, some of them more efficiently than the others. However, we have not yet verified whether our AL-EAs scheme itself is functioning as intended. In particular, we must assess whether the AL outer loop contributes meaningfully to the optimization process by accurately estimating the Lagrange multipliers and driving them toward their optimal values. This is examined in Tab.~\ref{tab:multipliers_errors_AL_EAs}, which reports the $\ell^2$-norm errors $\left\|\begin{bmatrix} \boldsymbol{\lambda}_{\star} - \boldsymbol{\lambda}_{f} \\ \boldsymbol{\mu}_{\star} - \boldsymbol{\mu}_{f} \end{bmatrix}\right\|_2$ between the final estimated multipliers $(\boldsymbol{\lambda}_f, \boldsymbol{\mu}_f)$ and their optimal counterparts $(\boldsymbol{\lambda}_{\star}, \boldsymbol{\mu}_{\star})$. We report only the multiplier errors associated with the active non-box constraints, since we dealt with variable bounds via projection onto the prescribed intervals. The results indicate that the AL-EAs estimate the Lagrange multipliers with very small errors, even in the most challenging cases where the optimal multiplier vector is nonzero (Prob.~\ref{prob:quadratic},~\ref{prob:power},~\ref{prob:linear_polynomial}, and~\ref{prob:double_integrator}), with the only exception of Quadratic Prob.~\ref{prob:quadratic}. For this COP, the comparatively larger multiplier estimation errors are consistent with the lower overall success rates observed for the majority of EAs relative to the remaining benchmark problems, suggesting that accurate multiplier estimation is closely linked to successful convergence. It's important to note that, for Prob.~\ref{prob:double_integrator}, the total multiplier estimation error of PSO-AL is an order of magnitude smaller than any individual component of the high-dimensional optimal multiplier vector, demonstrating strong numerical accuracy. Overall, PSO-AL and DE-AL once again outperform the remaining AL-EAs. It should also be noted that the accuracy of the estimated Lagrange multipliers is influenced by the update strategy of the AL framework, including the initialization, scaling factor, and maximum values of the penalty parameters in \eqref{eq:penalty_parameters_update}, which affect the conditioning of the subproblems and the convergence of the multiplier estimates. In our experiments we have used the same update setup for all COPs and AL-EAs.

To further inspect the internal behavior of our \methodacronym\ scheme, Fig.~\ref{fig:double_integrator_convergence} shows the convergence of the best AL-EA optimizer, PSO-AL, on the Double Integrator Prob.~\ref{prob:double_integrator}, with respect to two metrics: the \textit{objective error} and \textit{constraint violation}. In these plots, the horizontal axis is the normalized progress of each run, $\tau \in [0,1]$, where $\tau=0$ and $\tau=1$ denote the start and end of the run, respectively. To enable comparison across runs of different lengths, all trajectories are linearly interpolated onto a common normalized progress grid before computing the median and the 5th-95th percentile band. The solid curve represents the median across runs at each $\tau$ value, while the shaded region indicates the corresponding 5th-95th percentile band. The dashed horizontal line marks the nominal convergence threshold ($\varepsilon_s$ or $\varepsilon_c$). The plots reveal a consistent convergence trend across both metrics, with the objective error and constraint violation decreasing rapidly during the early stages of the optimization, and then almost flattening near their limiting values as the runs approach completion. These observations are consistent with the statistics reported in Tab.~\ref{tab:metrics_stats_AL_EAs}.

\subsection{Comparing different Kinds of Optimizers} \label{sec:comparing_diff_optimizers}

Before evaluating the last set of benchmark optimizers, we need to summarize briefly their main functionality. First, we have \textit{\textbf{PSO-PEN}}, which is the penalty-based alternative to \textit{PSO-AL}. We chose to use PSO, as it was experimentally (Sec.~\ref{sec:comparing_AL_EAs}) the best AL-EA. PSO-PEN maintains the usual PSO movement rule (Sec.~\ref{sec:PSO}) with the same exact hyper-parameters, but the population's quality at $k$-th iteration is judged by a cost function $J$ consisting of the objective function plus an increasing penalty term, computed from the equality and inequality violations (this is inspired from~\cite{parsopoulos2002particle}): $J(\boldsymbol{x}) = f(\boldsymbol{x}) + t(k) \cdot p(\boldsymbol{x})$, where $p(\boldsymbol{x}) = \left(\sum_{i=1}^{l} \xi(|h_i(\boldsymbol{x})|) + \sum_{j=1}^{m} \xi(\max(0,\ g_j(\boldsymbol{x})))\right)$, $t(k)$ is an iteration-dependent growing weight and $\xi(\cdot)$ is some monotonous increasing function. So, the only key difference from PSO-AL is that there is no AL outer loop, instead, feasibility is handled directly inside the penalty-shaped fitness function in a single PSO run.

The second optimizer is \textit{\textbf{IPOPT}} \cite{wachter2006_IPOPT}, a deterministic primal-dual barrier interior-point algorithm with a filter line search, for nonlinear programming. We worked with python's \texttt{cyipopt} (https://github.com/mechmotum/cyipopt) interface, where we passed, for each COP, the objective function and its gradient, the constraints and their Jacobians (in dense form), and the variable bounds. We configured the solver with a strict internal convergence tolerance of $10^{-7}$ (for the bounds of the scaled KKT residual, that combines stationarity and equality/inequality feasibility violation terms) and a maximum of $2{,}000$ iterations to ensure reliable and stable termination.

The final benchmark optimizer is \textit{\textbf{CMA-ES}} \cite{hansen2006cma}, which is a Covariance Matrix Adaptation evolution strategy. We used the standard Python's \texttt{pycma} (https://github.com/CMA-ES/pycma) library, developed by the authors of the method. In general, CMA-ES samples a Gaussian population around a current mean, evaluates the candidates, and then adapts the mean, step size (standard deviation), and covariance matrix from the best solutions, similarly to the CEM optimizer (Sec.~\ref{sec:CEM}). Here, we employed the Augmented-Lagrangian approach AL-CMA-ES, proposed in \cite{dufosse2021augmented,girardin2025augmented}, in order to handle the COP constraints. Following the recommended default choice in \cite{hansen2006cma}, we used the $(\mu/\mu_w,\lambda)$-CMA-ES version, with initial step size $\sigma_0 = 0.25$ and offsprings population size $\lambda = 4 + \lfloor 3\ln(n) \rfloor$ ($n$ is the COP's dimension), while the remaining strategy parameters were automatically derived from $\lambda$, including the parent population size $\mu$ (typically $\lfloor \lambda/2 \rfloor$), the recombination weights $w$ which determine the variance-effective selection mass $\mu_w$, and the step-size adaptation. We capped the base run at 5{,}000 generations and used an IPOP-CMA-ES restart scheme \cite{auger2005restart}, sharing the corresponding evaluation budget across the initial run and up to $3$ restarts. A restart was launched only if the current run terminated early according to CMA-ES's stopping criteria, and each restart used the remaining budget with a doubled $\lambda$. Additionally, we scaled the decision variables to comparable ranges in order for the step size to act more uniformly across coordinates and to prevent ill-conditioned covariances. Furthermore, we decided to follow a "best-feasible tracking" approach, retaining the best feasible point observed across all generations and restarts. Without it, CMA-ES's reported solutions were considerably less robust, i.e. the \textit{Final Hits} column would be far lower than the \textit{Ever Hits} column in Tab.~\ref{tab:metrics_stats_diff_optimizers}, indicating loss of global minima tracking. We avoided $(1+1)$-CMA-ES, as it is ''a poor global optimizer and not usefully applied to multimodal problems'' \cite{arnold20121}, and also ``on unconstrained multimodal problems $(\mu/\mu_w,\lambda)$-CMA-ES vastly outperforms the $(1+1)$-CMA-ES'' \cite{dufosse2021augmented}.

\begin{table*}
\centering
\caption{\textbf{Aggregated performance comparison of optimizers PSO-AL (best of the AL-EAs), PSO-PEN, IPOPT, and CMA-ES}, on the whole benchmark set (10 COPs). The metrics reported here and the presentation of the results are the same as in Tab.~\ref{tab:metrics_stats_AL_EAs}.}
\footnotesize
\setlength{\tabcolsep}{2pt}
\renewcommand{\arraystretch}{1.0}
\begin{tabular}{c|c|ccc|ccc|ccc|ccc}
\toprule
\multirow{2}{*}[-3pt]{\textbf{Problem}} & \multirow{2}{*}[-3pt]{\textbf{Optimizer}} & \multicolumn{3}{c|}{$\boldsymbol{\|x_\star - x_f\|_2}$} & \multicolumn{3}{c|}{$\boldsymbol{|f(x_\star) - f(x_f)|}$} & \multicolumn{3}{c|}{$\boldsymbol{\mathcal{V}_c(x_f)}$} & \multicolumn{3}{c}{\makecell{\textbf{\quad Obj. value Success \quad} \\ \textbf{$(\varepsilon_s = 10^{-3},\; \varepsilon_c = 10^{-4})$}}} \\
\cmidrule(lr){3-5}\cmidrule(lr){6-8}\cmidrule(lr){9-11}\cmidrule(lr){12-14}
& & \textbf{5\%}  & \textbf{Med.} & \textbf{95\%}
& \textbf{5\%}    & \textbf{Med.} & \textbf{95\%}
& \textbf{5\%}    & \textbf{Med.} & \textbf{95\%}
& \makecell{\textbf{Ever} \\ \textbf{Hits}} & \makecell{\textbf{$1^{\text{st}}$-Hit Med.} \\ \textbf{Time [msec]}} & \makecell{\textbf{Final} \\ \textbf{Hits}} \\
\midrule
\multirow{4}{*}{\makecell{\textbf{(\ref{prob:rosenbrock})}\\\textbf{Rosenbrock}\\\textbf{2D}}} & \textbf{PSO-AL} & 0 & \textcolor{ForestGreen}{\textbf{0}} & 0 & 0 & \textcolor{ForestGreen}{\textbf{0}} & 0 & 0 & \textcolor{ForestGreen}{\textbf{0}} & 2.8e-10 & \textcolor{ForestGreen}{\textbf{100/100}} & 7 & \textcolor{ForestGreen}{\textbf{100/100}} \\
 & \textbf{PSO-PEN} & 0 & 1.0e-6 & 3.0e-6 & 0 & 5.0e-10 & 3.2e-9 & 0 & \textcolor{ForestGreen}{\textbf{0}} & 2.9e-7 & \textcolor{ForestGreen}{\textbf{100/100}} & 9.2 & \textcolor{ForestGreen}{\textbf{100/100}} \\
 & \textbf{IPOPT} & 1.2e-6 & \textcolor{red}{\textbf{1.411}} & 1.411 & 7.3e-10 & \textcolor{red}{\textbf{0.999}} & 0.999 & 0 & 9.9e-9 & 1.0e-8 & \textcolor{red}{\textbf{18/100}} & \textcolor{ForestGreen}{\textbf{2.6}} & \textcolor{red}{\textbf{18/100}} \\
 & \textbf{CMA-ES} & 7.6e-10 & 1.411 & 1.411 & 0 & 0.999 & 0.999 & 0 & \textcolor{red}{\textbf{9.7e-5}} & 1.0e-4 & 35/100 & \textcolor{red}{\textbf{514.4}} & 35/100 \\
\midrule
\multirow{4}{*}{\makecell{\textbf{(\ref{prob:quadratic})}\\\textbf{Quadratic}}} & \textbf{PSO-AL} & 0 & \textcolor{ForestGreen}{\textbf{0}} & 6 & 0 & \textcolor{ForestGreen}{\textbf{0}} & 16 & 0 & \textcolor{ForestGreen}{\textbf{0}} & 0 & \textcolor{ForestGreen}{\textbf{78/100}} & 20.4 & \textcolor{ForestGreen}{\textbf{78/100}} \\
 & \textbf{PSO-PEN} & 12.317 & \textcolor{red}{\textbf{24.657}} & 38.931 & 39.408 & \textcolor{red}{\textbf{1.5e+3}} & 6.6e+3 & 0.1241 & \textcolor{red}{\textbf{16.2035}} & 46.5129 & \textcolor{red}{\textbf{0/100}} & - & \textcolor{red}{\textbf{0/100}} \\
 & \textbf{IPOPT} & 3.8 & 7.874 & 12.728 & 11.4 & 126 & 162.6 & 9.3e-9 & 1.4e-8 & 1.4e-8 & 5/100 & \textcolor{ForestGreen}{\textbf{3.1}} & 5/100 \\
 & \textbf{CMA-ES} & 7.0e-5 & 6 & 6.126 & 0.01 & 11.988 & 70.968 & 0 & 9.6e-5 & 1.0e-4 & 7/100 & \textcolor{red}{\textbf{964.8}} & \textcolor{red}{\textbf{0/100}} \\
\midrule
\multirow{4}{*}{\makecell{\textbf{(\ref{prob:power})}\\\textbf{Power-Terms}\\\textbf{Objective}}} & \textbf{PSO-AL} & 3.3e-6 & 3.3e-6 & 3.3e-6 & 0 & \textcolor{ForestGreen}{\textbf{0}} & 0 & 1.4e-9 & 4.3e-8 & 2.3e-7 & \textcolor{ForestGreen}{\textbf{100/100}} & \textcolor{ForestGreen}{\textbf{17.3}} & \textcolor{ForestGreen}{\textbf{100/100}} \\
 & \textbf{PSO-PEN} & 6.7e-6 & 6.7e-6 & 18.128 & 1.0e-5 & 1.0e-5 & 69.437 & 3.8e-6 & 5.0e-6 & 46.7052 & 76/100 & 64 & 76/100 \\
 & \textbf{IPOPT} & 5.1e-8 & \textcolor{ForestGreen}{\textbf{5.7e-8}} & 1.856 & 1.7e-6 & 1.7e-6 & 1.631 & 0 & \textcolor{ForestGreen}{\textbf{7.5e-9}} & 1.0e-8 & 91/100 & 70 & 91/100 \\
 & \textbf{CMA-ES} & 4.7e-5 & \textcolor{red}{\textbf{4.216}} & 4.216 & 1.5e-4 & \textcolor{red}{\textbf{4.514}} & 4.514 & 7.5e-5 & \textcolor{red}{\textbf{9.8e-5}} & 1.0e-4 & \textcolor{red}{\textbf{47/100}} & \textcolor{red}{\textbf{412.8}} & \textcolor{red}{\textbf{47/100}} \\
\midrule
\multirow{4}{*}{\makecell{\textbf{(\ref{prob:linear_polynomial})}\\\textbf{Linear}\\\textbf{Objective}}} & \textbf{PSO-AL} & 0 & \textcolor{ForestGreen}{\textbf{0}} & 0 & 0 & \textcolor{ForestGreen}{\textbf{0}} & 0 & 1.5e-10 & \textcolor{ForestGreen}{\textbf{1.6e-9}} & 6.4e-9 & \textcolor{ForestGreen}{\textbf{100/100}} & 13.5 & \textcolor{ForestGreen}{\textbf{100/100}} \\
 & \textbf{PSO-PEN} & 1.0e-5 & 1.0e-5 & 0.823 & 1.0e-5 & 1.0e-5 & 0.609 & 3.4e-6 & 5.0e-6 & 2.3383 & 88/100 & 29.5 & 88/100 \\
 & \textbf{IPOPT} & 3.1e-6 & \textcolor{red}{\textbf{0.813}} & 1.738 & 3.3e-6 & \textcolor{red}{\textbf{1.088}} & 1.454 & 1.4e-8 & 1.4e-8 & 1.4e-8 & \textcolor{red}{\textbf{42/100}} & \textcolor{ForestGreen}{\textbf{1.7}} & \textcolor{red}{\textbf{42/100}} \\
 & \textbf{CMA-ES} & 6.3e-5 & 7.3e-5 & 7.6e-5 & 6.8e-5 & 7.6e-5 & 7.9e-5 & 8.6e-5 & \textcolor{red}{\textbf{9.7e-5}} & 1.0e-4 & \textcolor{ForestGreen}{\textbf{100/100}} & \textcolor{red}{\textbf{269.7}} & \textcolor{ForestGreen}{\textbf{100/100}} \\
\midrule
\multirow{4}{*}{\makecell{\textbf{(\ref{prob:mishra_bird})}\\\textbf{Mishra's}\\\textbf{Bird}}} & \textbf{PSO-AL} & 3.9e-6 & 3.9e-6 & 3.9e-6 & 4.6e-4 & 4.6e-4 & 4.6e-4 & 0 & \textcolor{ForestGreen}{\textbf{0}} & 0 & \textcolor{ForestGreen}{\textbf{100/100}} & 8.3 & \textcolor{ForestGreen}{\textbf{100/100}} \\
 & \textbf{PSO-PEN} & 3.9e-6 & 3.9e-6 & 3.9e-6 & 4.6e-4 & 4.6e-4 & 4.6e-4 & 0 & \textcolor{ForestGreen}{\textbf{0}} & 0 & \textcolor{ForestGreen}{\textbf{100/100}} & 8.8 & \textcolor{ForestGreen}{\textbf{100/100}} \\
 & \textbf{IPOPT} & 2.3e-8 & \textcolor{red}{\textbf{4.619}} & 7.899 & 4.9e-8 & \textcolor{red}{\textbf{85.246}} & 108.252 & 0 & \textcolor{ForestGreen}{\textbf{0}} & 1.0e-8 & \textcolor{red}{\textbf{46/100}} & \textcolor{ForestGreen}{\textbf{3.2}} & \textcolor{red}{\textbf{46/100}} \\
 & \textbf{CMA-ES} & 1.6e-8 & \textcolor{ForestGreen}{\textbf{2.3e-8}} & 3.1e-8 & 4.9e-8 & \textcolor{ForestGreen}{\textbf{4.9e-8}} & 4.9e-8 & 0 & \textcolor{ForestGreen}{\textbf{0}} & 0 & 99/100 & \textcolor{red}{\textbf{251.4}} & 99/100 \\
\midrule
\multirow{4}{*}{\makecell{\textbf{(\ref{prob:gomez_levy})}\\\textbf{Constrained}\\\textbf{Gomez-Levy}}} & \textbf{PSO-AL} & 4.0e-7 & 4.0e-7 & 4.0e-7 & 1.5e-6 & 1.5e-6 & 1.5e-6 & 0 & \textcolor{ForestGreen}{\textbf{0}} & 0 & \textcolor{ForestGreen}{\textbf{100/100}} & 8.5 & \textcolor{ForestGreen}{\textbf{100/100}} \\
 & \textbf{PSO-PEN} & 4.0e-7 & 6.3e-7 & 1.0e-6 & 1.5e-6 & 1.5e-6 & 1.5e-6 & 0 & \textcolor{ForestGreen}{\textbf{0}} & 0 & \textcolor{ForestGreen}{\textbf{100/100}} & 8.2 & \textcolor{ForestGreen}{\textbf{100/100}} \\
 & \textbf{IPOPT} & 1.1e-8 & \textcolor{red}{\textbf{1.383}} & 1.413 & 4.7e-8 & \textcolor{red}{\textbf{0.044}} & 0.667 & 0 & \textcolor{red}{\textbf{9.3e-9}} & 9.9e-9 & \textcolor{red}{\textbf{43/100}} & \textcolor{ForestGreen}{\textbf{2.7}} & \textcolor{red}{\textbf{43/100}} \\
 & \textbf{CMA-ES} & 9.3e-9 & \textcolor{ForestGreen}{\textbf{1.3e-8}} & 2.3e-8 & 4.7e-8 & \textcolor{ForestGreen}{\textbf{4.7e-8}} & 4.7e-8 & 0 & \textcolor{ForestGreen}{\textbf{0}} & 0 & 98/100 & \textcolor{red}{\textbf{282.7}} & 98/100 \\
\midrule
\multirow{4}{*}{\makecell{\textbf{(\ref{prob:drop_wave})}\\\textbf{Drop-Wave}}} & \textbf{PSO-AL} & 0 & \textcolor{ForestGreen}{\textbf{0}} & 0 & 2.0e-6 & \textcolor{ForestGreen}{\textbf{2.0e-6}} & 2.0e-6 & 2.7e-10 & 5.3e-9 & 2.2e-8 & \textcolor{ForestGreen}{\textbf{100/100}} & \textcolor{ForestGreen}{\textbf{8.4}} & \textcolor{ForestGreen}{\textbf{100/100}} \\
 & \textbf{PSO-PEN} & 0 & 1.0e-5 & 3.3e-5 & 2.0e-6 & \textcolor{ForestGreen}{\textbf{2.0e-6}} & 2.0e-6 & 2.2e-8 & 1.4e-7 & 5.6e-7 & 96/100 & 11.8 & 96/100 \\
 & \textbf{IPOPT} & 4.4e-4 & \textcolor{red}{\textbf{1.629}} & 6.977 & 3.8e-6 & \textcolor{red}{\textbf{0.193}} & 0.332 & 0 & \textcolor{ForestGreen}{\textbf{0}} & 46.8162 & \textcolor{red}{\textbf{8/100}} & 81.6 & \textcolor{red}{\textbf{7/100}} \\
 & \textbf{CMA-ES} & 1.7e-6 & 0.613 & 2.119 & 2.3e-6 & 0.08 & 0.216 & 0 & \textcolor{red}{\textbf{1.6e-7}} & 9.4e-5 & 37/100 & \textcolor{red}{\textbf{772.2}} & 37/100 \\
\midrule
\multirow{4}{*}{\makecell{\textbf{(\ref{prob:egg_holder})}\\\textbf{Egg-Holder}}} & \textbf{PSO-AL} & 2.0e-4 & 2.0e-4 & 2.0e-4 & 3.4e-4 & 3.4e-4 & 3.4e-4 & 0 & \textcolor{ForestGreen}{\textbf{0}} & 0 & \textcolor{ForestGreen}{\textbf{100/100}} & \textcolor{ForestGreen}{\textbf{11.3}} & \textcolor{ForestGreen}{\textbf{100/100}} \\
 & \textbf{PSO-PEN} & 2.0e-4 & 2.0e-4 & 41.226 & 3.4e-4 & 3.4e-4 & 2.723 & 0 & \textcolor{ForestGreen}{\textbf{0}} & 0 & 81/100 & 12.1 & 81/100 \\
 & \textbf{IPOPT} & 5.2e-6 & \textcolor{red}{\textbf{44.749}} & 344.504 & 2.0e-5 & \textcolor{red}{\textbf{37.539}} & 873.608 & 0 & \textcolor{ForestGreen}{\textbf{0}} & 1.0e-8 & \textcolor{red}{\textbf{43/100}} & 255.5 & \textcolor{red}{\textbf{43/100}} \\
 & \textbf{CMA-ES} & 5.0e-6 & \textcolor{ForestGreen}{\textbf{5.1e-6}} & 41.227 & 2.7e-6 & \textcolor{ForestGreen}{\textbf{2.7e-6}} & 6.396 & 0 & \textcolor{ForestGreen}{\textbf{0}} & 0 & 81/100 & \textcolor{red}{\textbf{648.3}} & 81/100 \\
\midrule
\multirow{4}{*}{\makecell{\textbf{(\ref{prob:griewank})}\\\textbf{Griewank}\\\textbf{2D}}} & \textbf{PSO-AL} & 3.8e-9 & \textcolor{ForestGreen}{\textbf{1.2e-8}} & 1.7e-8 & 0 & \textcolor{ForestGreen}{\textbf{0}} & 0 & 0 & \textcolor{ForestGreen}{\textbf{0}} & 0 & \textcolor{ForestGreen}{\textbf{100/100}} & \textcolor{ForestGreen}{\textbf{8.5}} & \textcolor{ForestGreen}{\textbf{100/100}} \\
 & \textbf{PSO-PEN} & 2.7e-7 & 1.0e-6 & 2.3e-6 & 0 & \textcolor{ForestGreen}{\textbf{0}} & 0 & 0 & \textcolor{ForestGreen}{\textbf{0}} & 0 & \textcolor{ForestGreen}{\textbf{100/100}} & 9.5 & \textcolor{ForestGreen}{\textbf{100/100}} \\
 & \textbf{IPOPT} & 2.8e-4 & \textcolor{red}{\textbf{15.629}} & 27.184 & 2.3e-8 & \textcolor{red}{\textbf{0.089}} & 0.736 & 0 & \textcolor{ForestGreen}{\textbf{0}} & 1.0e-8 & 13/100 & 28.6 & 13/100 \\
 & \textbf{CMA-ES} & 5.3e-8 & 5.496 & 10.874 & 0 & 0.01 & 0.03 & 0 & \textcolor{ForestGreen}{\textbf{0}} & 0 & \textcolor{red}{\textbf{9/100}} & \textcolor{red}{\textbf{262.2}} & \textcolor{red}{\textbf{9/100}} \\
\midrule
\multirow{4}{*}{\makecell{\textbf{(\ref{prob:double_integrator})}\\\textbf{Double}\\\textbf{Integrator}}} & \textbf{PSO-AL} & 0.011 & 0.016 & 0.024 & 1.0e-4 & 0.001 & 0.004 & 2.8e-5 & 4.4e-5 & 7.0e-5 & \textcolor{ForestGreen}{\textbf{25/25}} & \textcolor{red}{\textbf{4.7e+4}} & 11/25 \\
 & \textbf{PSO-PEN} & 11.099 & 14.688 & 21.146 & 6.345 & 10.508 & 24.522 & 0.0034 & \textcolor{red}{\textbf{0.0232}} & 0.7571 & \textcolor{red}{\textbf{0/25}} & - & \textcolor{red}{\textbf{0/25}} \\
 & \textbf{IPOPT} & 2.4e-5 & \textcolor{ForestGreen}{\textbf{2.4e-5}} & 2.5e-5 & 4.4e-6 & \textcolor{ForestGreen}{\textbf{4.4e-6}} & 4.4e-6 & 0 & \textcolor{ForestGreen}{\textbf{0}} & 0 & \textcolor{ForestGreen}{\textbf{25/25}} & \textcolor{ForestGreen}{\textbf{16.4}} & \textcolor{ForestGreen}{\textbf{25/25}} \\
 & \textbf{CMA-ES} & 22.379 & \textcolor{red}{\textbf{27.816}} & 32.419 & 25.62 & \textcolor{red}{\textbf{39.585}} & 55.008 & 9.4e-5 & 9.8e-5 & 9.9e-5 & \textcolor{red}{\textbf{0/25}} & - & \textcolor{red}{\textbf{0/25}} \\
\bottomrule
\end{tabular}
\label{tab:metrics_stats_diff_optimizers}
\end{table*}

The new comparison metrics are gathered in Tab.~\ref{tab:metrics_stats_diff_optimizers}, and the time data in Tab.~\ref{tab:running_times_diff_optimizers}. Initially, we notice that PSO-PEN performs at best comparably to, and more often worse than PSO-AL, across all 10 COPs, regarding both solution convergence and constraint satisfaction. Moreover, PSO-PEN is unable to handle higher-dimensional COPs, such as Prob.~\ref{prob:quadratic} and~\ref{prob:double_integrator}, where it totally fails. This showcases that the AL framework is crucial in ensuring constraint satisfaction. We also tested a second PSO-PEN variant, obtained by keeping identical the PSO-AL framework but removing the Lagrange multiplier updates \eqref{eq:lagrange_multipliers_update}. On the COPs for which the optimal vector is $(\boldsymbol{\lambda}_{\star}, \boldsymbol{\mu}_{\star}) = \boldsymbol{0}$, this simple quadratic penalty method retains the same high success rate as PSO-AL, since the multipliers are initialized at zero and remain at their optimal values throughout the optimization. However, on prob.~\ref{prob:quadratic},~\ref{prob:power},~\ref{prob:linear_polynomial}, and~\ref{prob:double_integrator}, where $(\boldsymbol{\lambda}_{\star}, \boldsymbol{\mu}_{\star}) \neq \boldsymbol{0}$, this approach neither achieves solution convergence nor ensures constraint satisfaction, and it completely fails in every run.

\begin{table*}
\centering
\caption{\textbf{Aggregated time comparison of optimizers PSO-AL (best of the AL-EAs), PSO-PEN, IPOPT, and CMA-ES}, on the whole benchmark set (10 COPs), reported in a similar way as in Tab.~\ref{tab:running_times_AL_EAs}.}
\footnotesize
\setlength{\tabcolsep}{4pt}
\renewcommand{\arraystretch}{1.0}
\resizebox{\textwidth}{!}{%
\begin{tabular}{l|l|cccccccccc}
\toprule
\multirow{5}{*}{\rotatebox[origin=c]{90}{\makecell{\textbf{Total Running}\\\textbf{Times [msec]}\\\textbf{(Median)}}}}
& \diagbox{\textbf{Optimizer}}{\textbf{Problem}} & \makecell{\textbf{(\ref{prob:rosenbrock})}\\\textbf{Rosenbrock}\\\textbf{2D}} & \makecell{\textbf{(\ref{prob:quadratic})}\\\textbf{Quadratic}} & \makecell{\textbf{(\ref{prob:power})}\\\textbf{Power-Terms}\\\textbf{Objective}} & \makecell{\textbf{(\ref{prob:linear_polynomial})}\\\textbf{Linear}\\\textbf{Objective}} & \makecell{\textbf{(\ref{prob:mishra_bird})}\\\textbf{Mishra's}\\\textbf{Bird}} & \makecell{\textbf{(\ref{prob:gomez_levy})}\\\textbf{Constrained}\\\textbf{Gomez-Levy}} & \makecell{\textbf{(\ref{prob:drop_wave})}\\\textbf{Drop-Wave}} & \makecell{\textbf{(\ref{prob:egg_holder})}\\\textbf{Egg-Holder}} & \makecell{\textbf{(\ref{prob:griewank})}\\\textbf{Griewank}\\\textbf{2D}} & \makecell{\textbf{(\ref{prob:double_integrator})}\\\textbf{Double}\\\textbf{Integrator}} \\
\cmidrule(l){2-12}
& \makecell{\textbf{PSO-AL}} & 158.2 & 208.7 & 200.1 & 199.8 & 184.5 & 200.5 & 184.1 & 205.4 & 184.9 & \textcolor{red}{5.4e+4} \\
& \makecell{\textbf{PSO-PEN}} & 9.2 & 300 & \textcolor{ForestGreen}{\textbf{63.3}} & 29 & 8.8 & 8.2 & \textcolor{ForestGreen}{\textbf{11.8}} & \textcolor{ForestGreen}{\textbf{11.7}} & \textcolor{ForestGreen}{\textbf{9.5}} & 4.1e+3 \\
& \makecell{\textbf{IPOPT}} & \textcolor{ForestGreen}{\textbf{5.4}} & \textcolor{ForestGreen}{\textbf{2.7}} & 76.6 & \textcolor{ForestGreen}{\textbf{1.9}} & \textcolor{ForestGreen}{\textbf{3.5}} & \textcolor{ForestGreen}{\textbf{3.3}} & 195.5 & \textcolor{red}{6.3e+3} & 56.5 & \textcolor{ForestGreen}{\textbf{23.6}} \\
& \makecell{\textbf{CMA-ES}} & \textcolor{red}{1.2e+3} & \textcolor{red}{2.7e+3} & \textcolor{red}{2.0e+3} & \textcolor{red}{1.4e+3} & \textcolor{red}{1.2e+3} & \textcolor{red}{1.2e+3} & \textcolor{red}{1.6e+3} & 1.3e+3 & \textcolor{red}{1.2e+3} & 3.5e+4 \\
\bottomrule
\end{tabular}
}
\label{tab:running_times_diff_optimizers}
\end{table*}

On the other hand, IPOPT appears to fail in reliably reaching the global minimum in almost all COPs (less than $50\%$ success, except Prob.~\ref{prob:power} and~\ref{prob:double_integrator}), displaying worse performance than PSO-PEN, and high variability in solution and objective errors. Although IPOPT converges very rapidly, in a matter of milliseconds (Tab.~\ref{tab:running_times_diff_optimizers}) due to its gradient-based updates, its inability to consistently locate the global minimum highlights the advantage of the AL-based population framework in addressing nonlinear, nonconvex, and multimodal problems. The comparatively weak performance of IPOPT can be intuitively explained by the structural characteristics of the benchmark landscapes considered in this study, that pose significant challenges to gradient-based or Newton-based local search methods, which are inherently prone to converging to suboptimal local stationary points. Classic examples such as the flat, banana-shaped Rosenbrock valley (Prob. ~\ref{prob:rosenbrock}) and the highly multimodal structures of Drop-Wave (Prob. ~\ref{prob:drop_wave}) and Griewank (Prob. ~\ref{prob:griewank}) illustrate these difficulties. As shown in the convergence contour plots of Fig.~\ref{fig:convergence_contour_maps}, IPOPT often gets trapped in local minima.
For the Rosenbrock problem, it converges near $(0,0)$ rather than to the global minimum at $(1,1)$, located at the intersection of the two active nonlinear inequality constraints, due to the narrow and curved geometry of the valley. For the Griewank problem, IPOPT often terminates at interior points within the feasible region instead of reaching the boundary global minimum on the active circular inequality constraint, primarily because of the large number of local minima. Similarly, for the Drop-Wave problem, it converges to local minima along the circular equality constraint, as these stationary points are difficult to distinguish from the global minimum using local information alone.
In contrast, population-based EAs leverage exploration and diversity to navigate such complex, nonconvex topographies more effectively, maintaining progress toward the true global optima (for example, PSO-AL converges to the global minima from any initial point, for all the cases described above).

Up to this point, we have shown that our \methodacronym\ approach handles the benchmark black-box COPs substantially more effectively than the state-of-the-art numerical optimizer IPOPT. The last benchmark optimizer is the state-of-the-art evolution strategy AL-$(\mu/\mu_w,\lambda)$-CMA-ES, which in general converges better than IPOPT. But it exhibits much weaker performance relative to \methodacronym\, on the higher-dimensional Prob.~\ref{prob:quadratic}, and~\ref{prob:double_integrator}, and on the multimodal Prob.~\ref{prob:drop_wave}, and~\ref{prob:griewank}. Also, it seems that it struggles in these specific COPs more than the others, which is consistent with the constrained-CMA-ES literature: \cite{dufosse2021augmented} restricts its evaluation to unimodal test problems, explicitly noting that ''...The performance of the CMA-ES on multimodal test problems has not been investigated...'', while \cite{girardin2025augmented} evaluates its proposed variant mainly on simple sphere and ellipsoid problems with linear constraints, rather than on the more complex, multiply constrained, high-dimensional and multimodal search spaces considered here. Overall, bringing CMA-ES to a level competitive with \methodacronym\ required considerable tuning, whereas our approach requires much less. CMA-ES was also noticeably slower in wall-clock time than the other optimizers (Tab.~\ref{tab:running_times_diff_optimizers}), likely due to the additional restarts and because \texttt{pycma} is implemented in pure Python, unlike the AL-EAs and IPOPT, which rely on compiled C++ backends.

\begin{figure}
    \centering
    \begin{subfigure}{0.49\columnwidth}
        \centering
        \includegraphics[width=\linewidth]{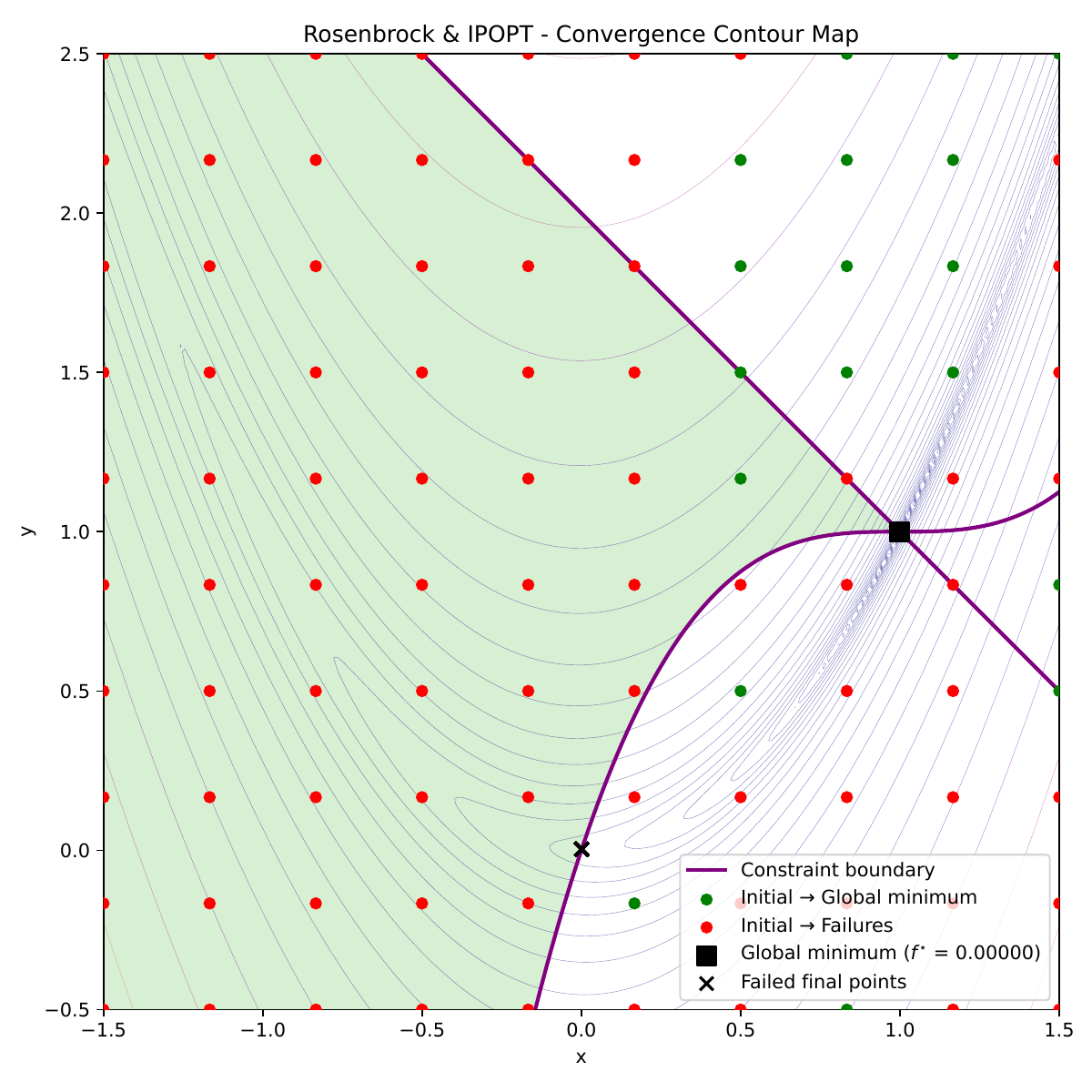}
        \caption{Rosenbrock 2D}
    \end{subfigure}
    \begin{subfigure}{0.49\columnwidth}
        \centering
        \includegraphics[width=\linewidth]{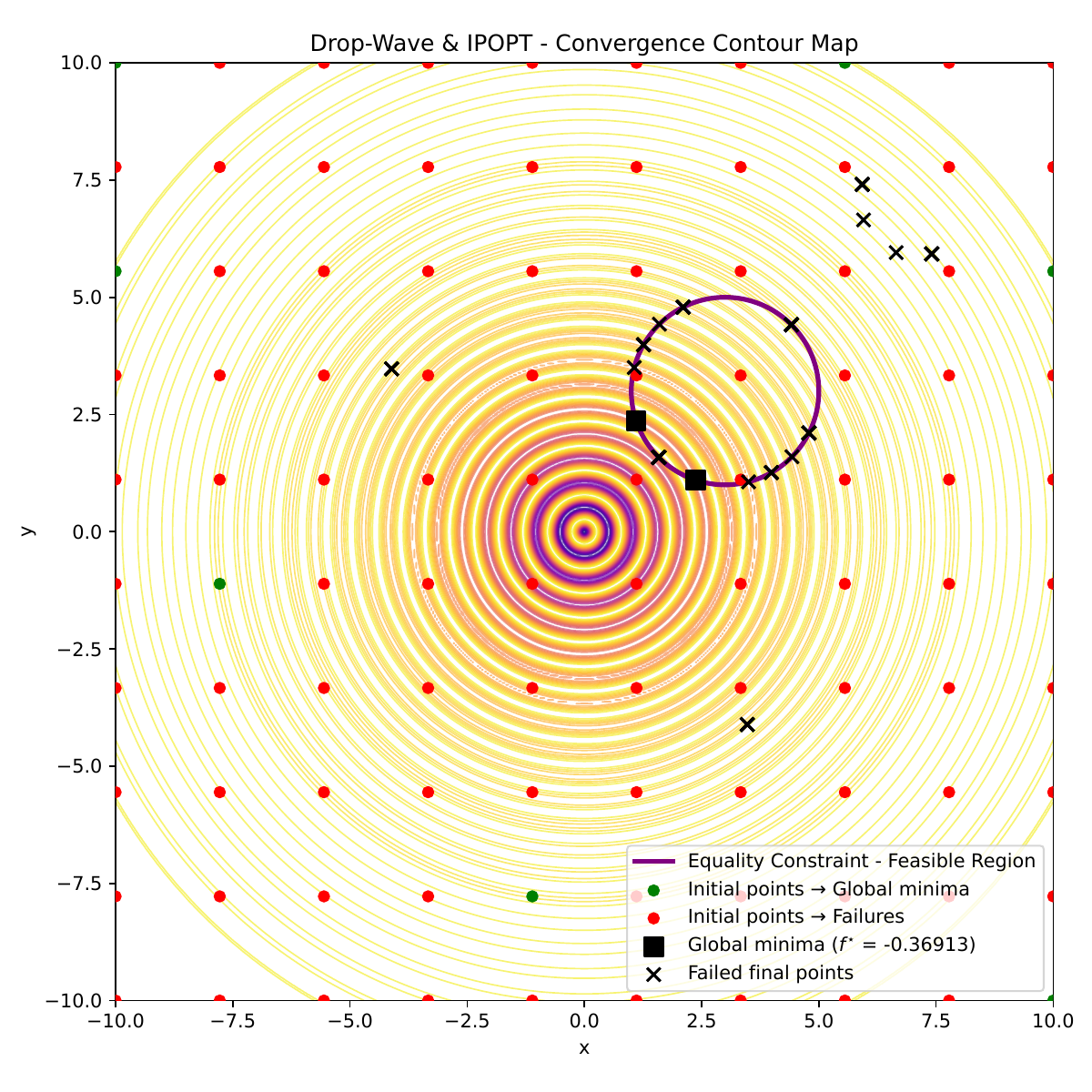}
        \caption{Drop-Wave}
    \end{subfigure}
    \begin{subfigure}{\columnwidth}
        \centering
        \includegraphics[width=\linewidth]{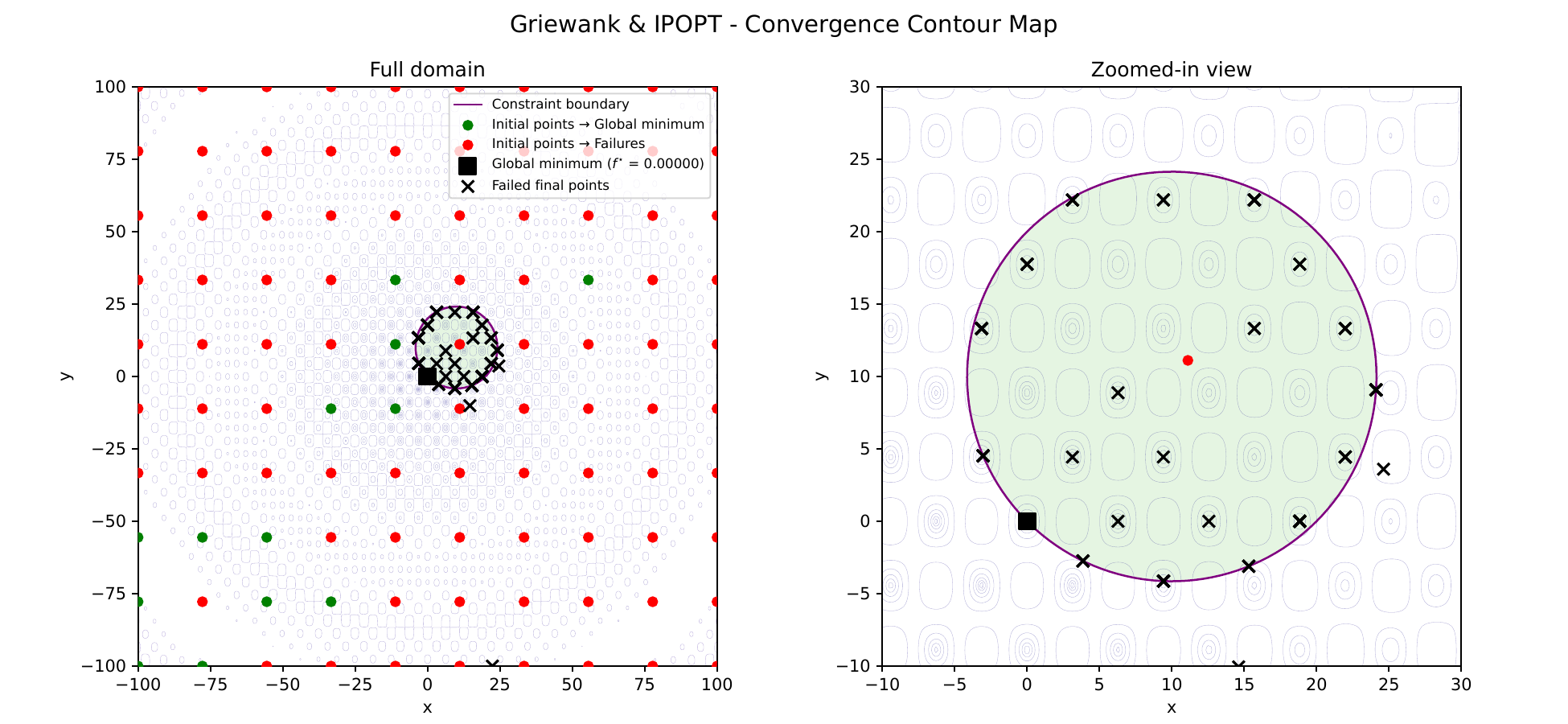}
        \caption{Griewank 2D}
    \end{subfigure}
    \caption{Convergence contour maps for three 2D COPs using IPOPT. Top left: Rosenbrock 2D (Prob.~\ref{prob:rosenbrock}). Top right: Drop-Wave (Prob.~\ref{prob:drop_wave}). Bottom: Griewank 2D (full domain + zoomed in view) (Prob.~\ref{prob:griewank}). Green dots show initial points that reached the global minimum, red dots show those that did not, black crosses mark final points of failed runs, and black squares mark the global minima positions. Constraint boundaries are in purple, with feasible regions shaded light green.}
    \label{fig:convergence_contour_maps}
\end{figure}

\subsection{Comparing all Optimizers together} \label{sec:comparing_all_optimizers}

\begin{figure}
    \centering
    
    \begin{subfigure}{.49\columnwidth}
        \centering
        \includegraphics[width=\linewidth]{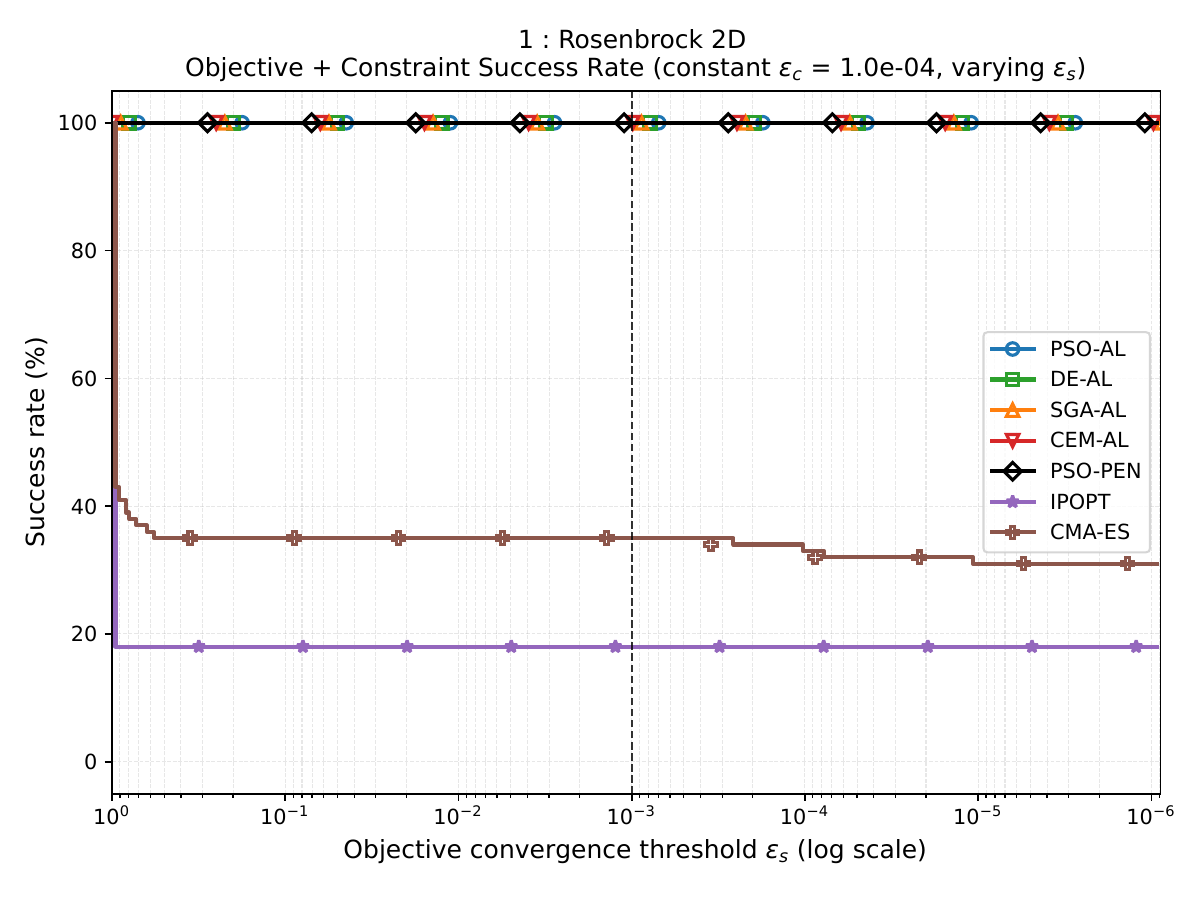}
    \end{subfigure}
    \begin{subfigure}{.49\columnwidth}
        \centering
        \includegraphics[width=\linewidth]{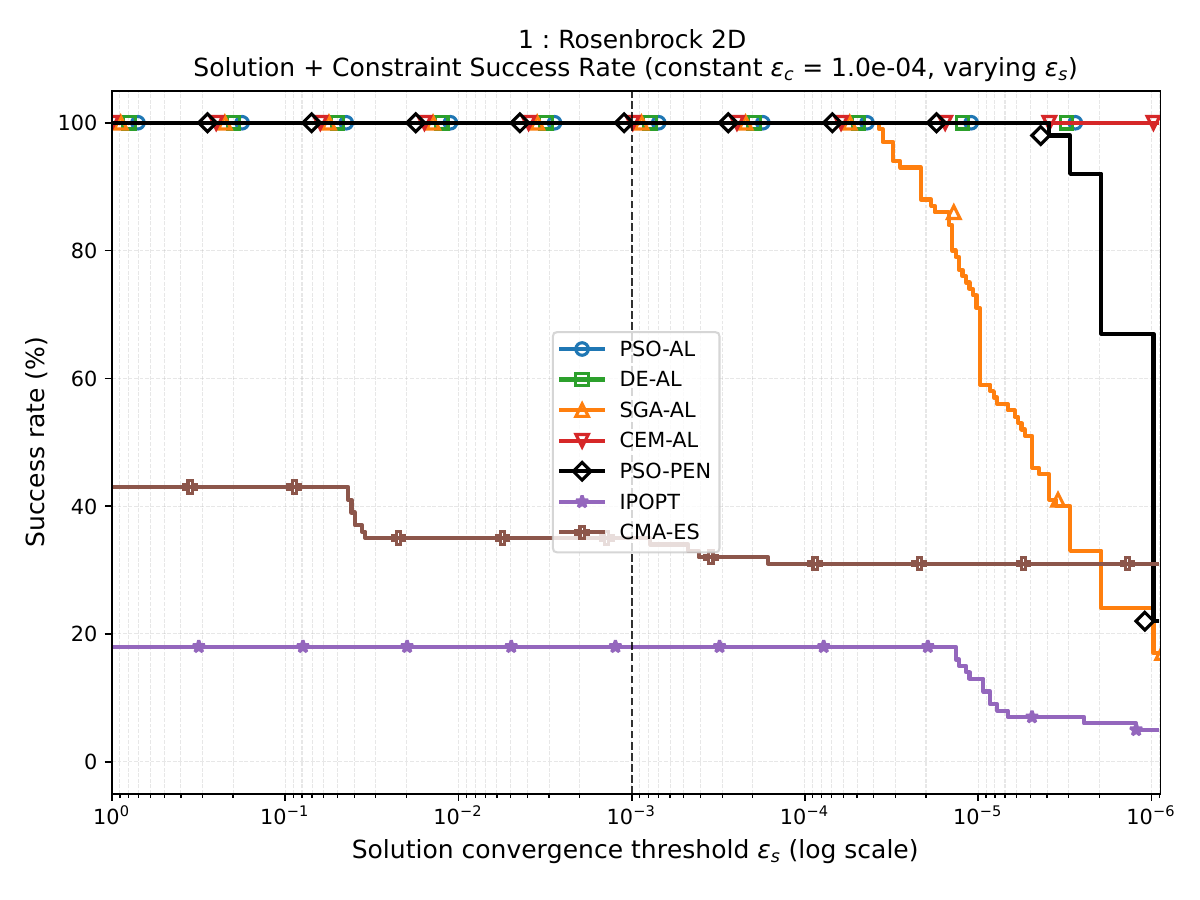}
    \end{subfigure}
    
    \begin{subfigure}{.49\columnwidth}
        \centering
        \includegraphics[width=\linewidth]{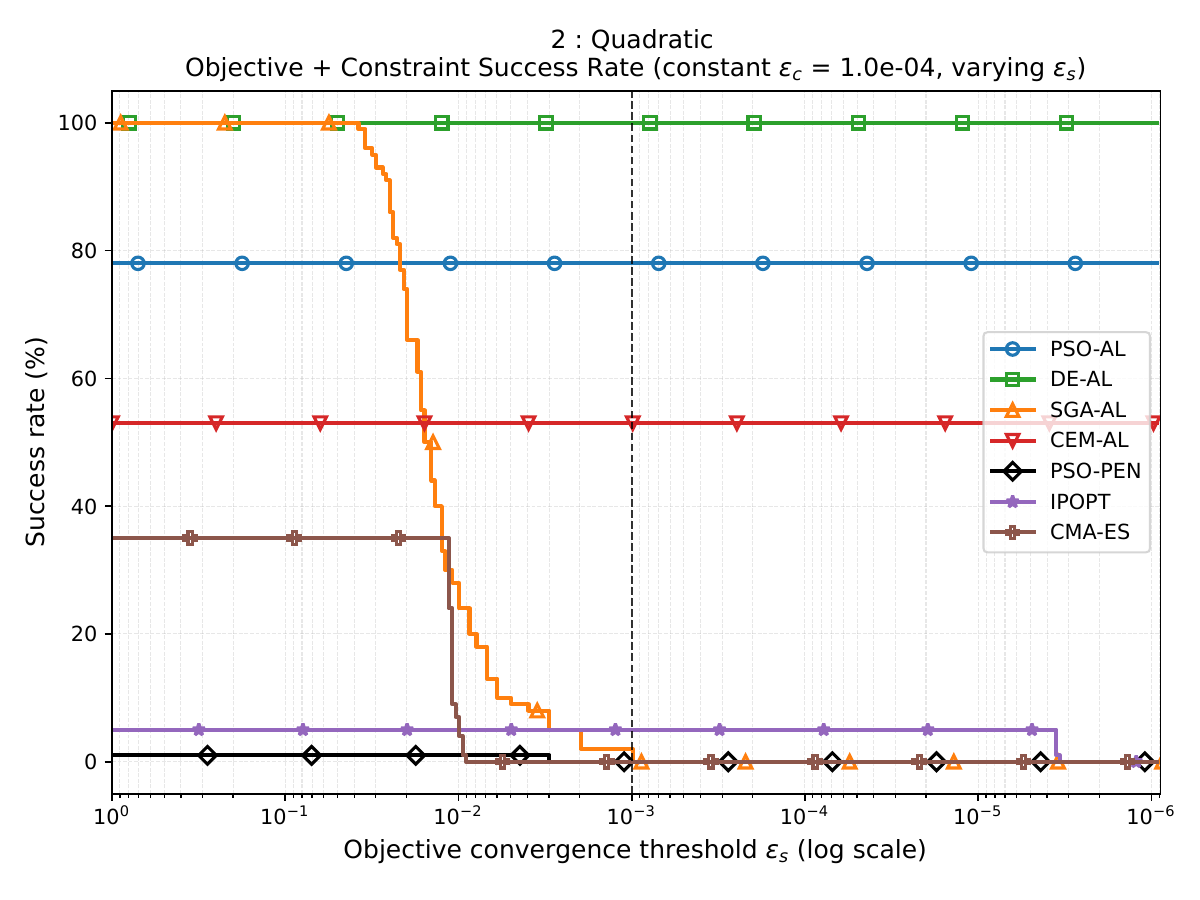}
    \end{subfigure}
    \begin{subfigure}{.49\columnwidth}
        \centering
        \includegraphics[width=\linewidth]{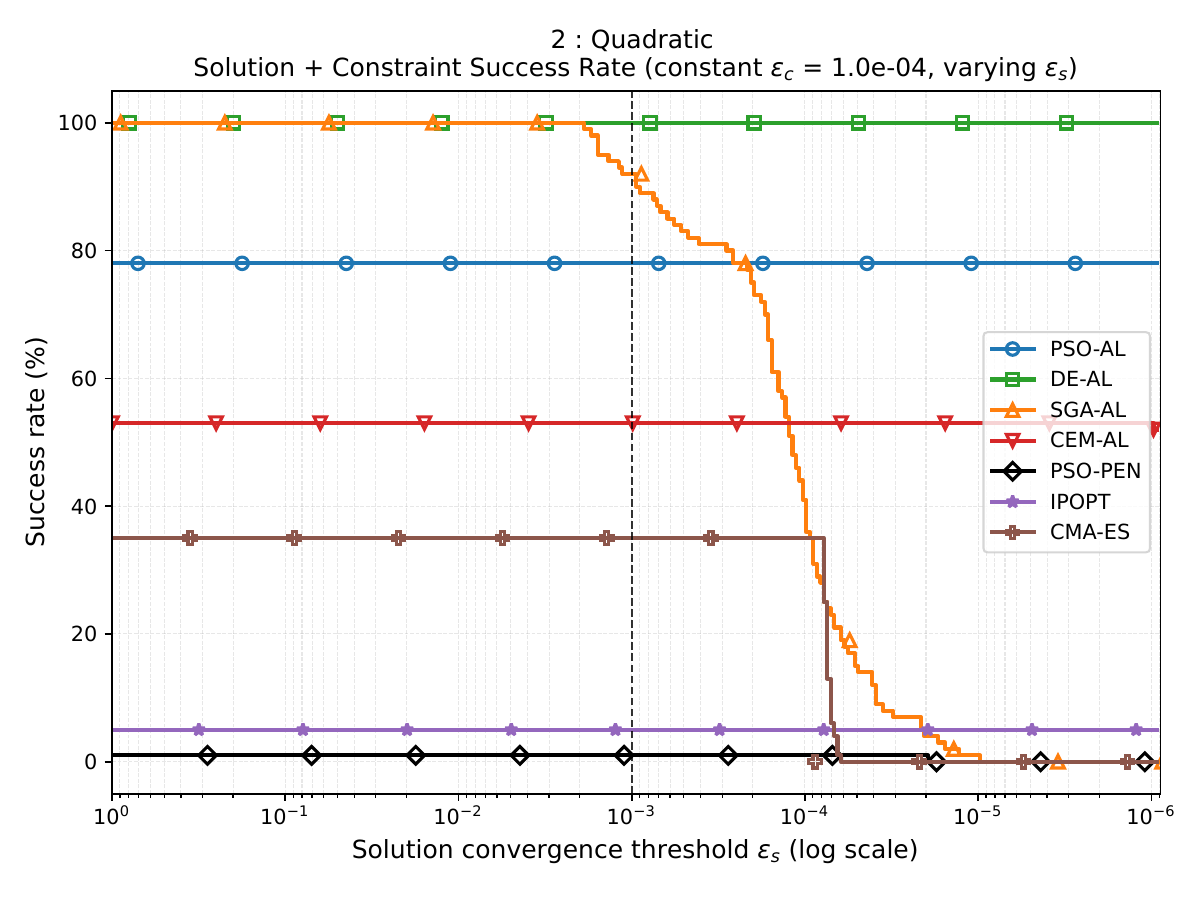}
    \end{subfigure}
    
    \begin{subfigure}{.49\columnwidth}
        \centering
        \includegraphics[width=\linewidth]{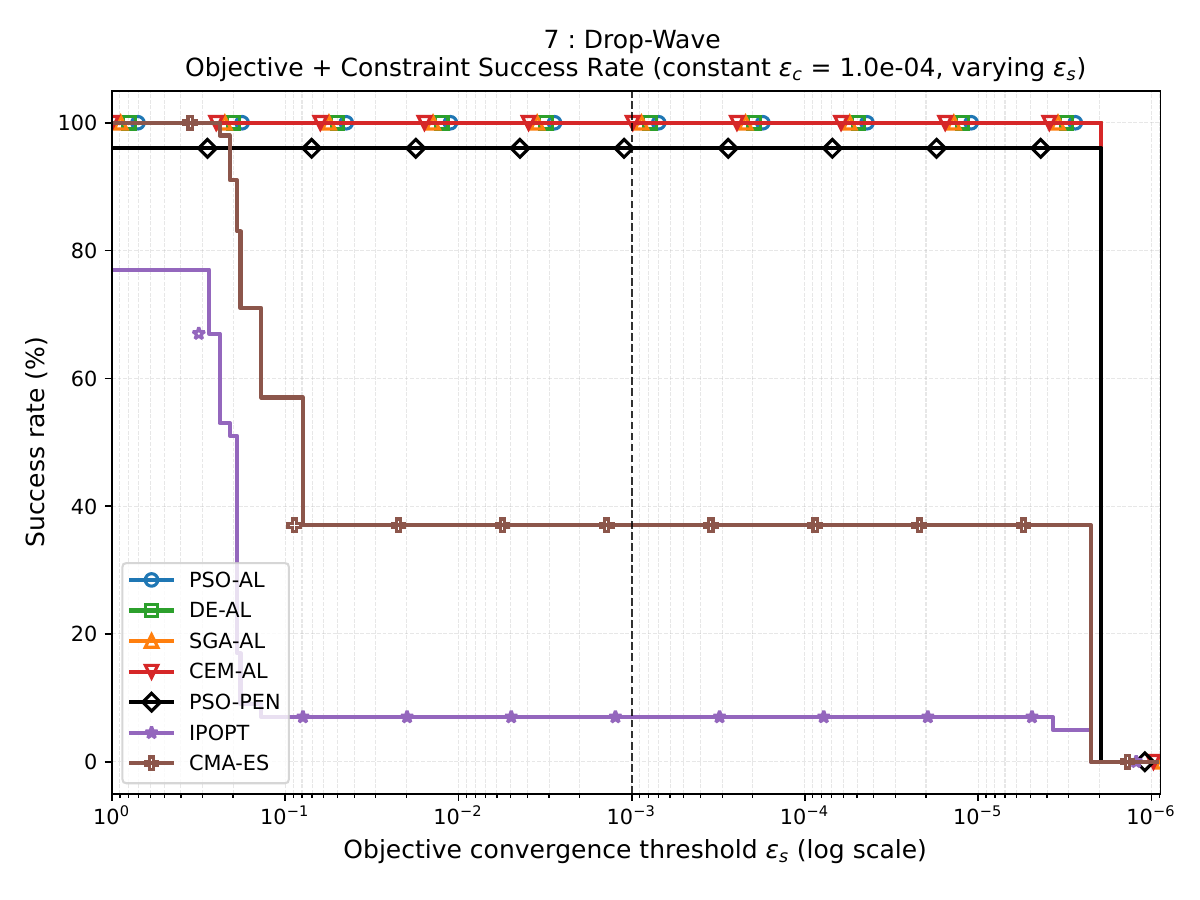}
    \end{subfigure}
    \begin{subfigure}{.49\columnwidth}
        \centering
        \includegraphics[width=\linewidth]{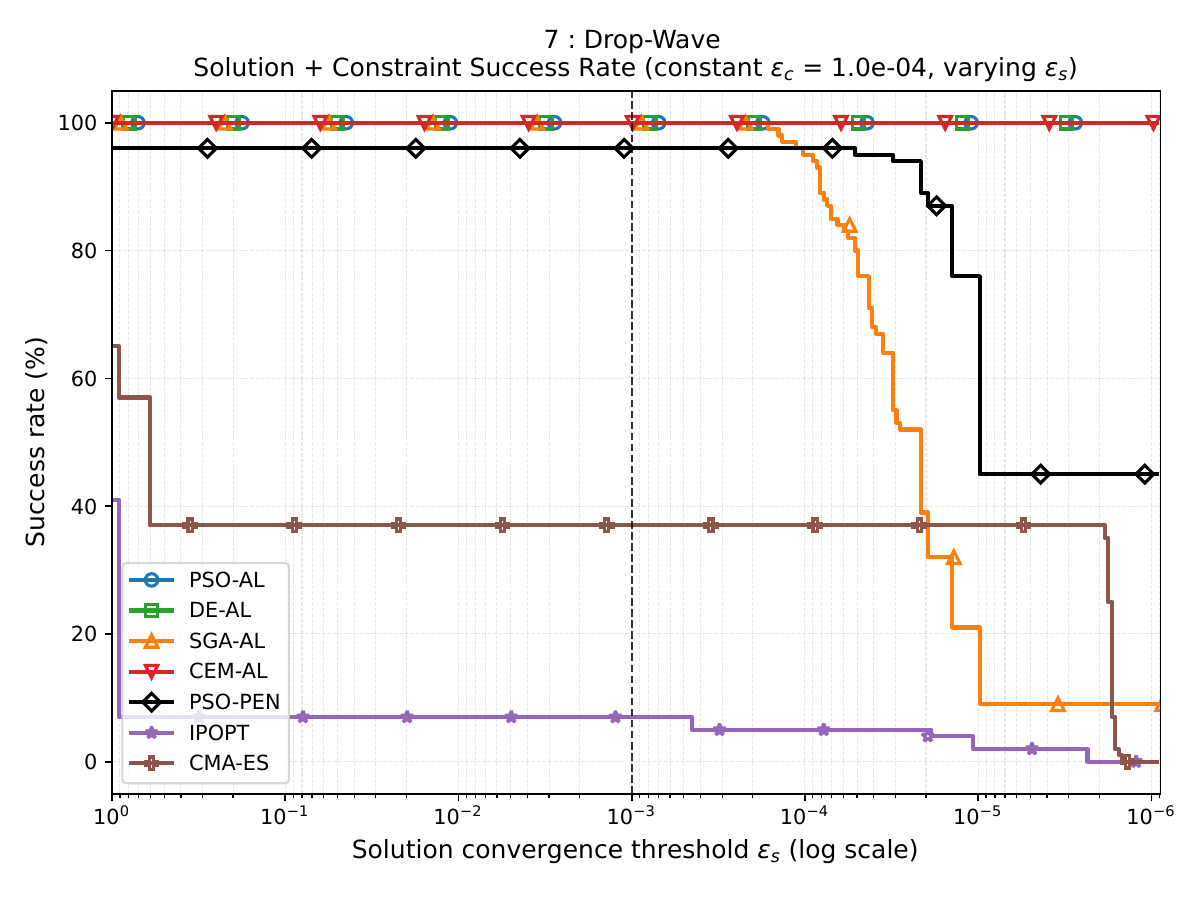}
    \end{subfigure}
    
    \begin{subfigure}{.49\columnwidth}
        \centering
        \includegraphics[width=\linewidth]{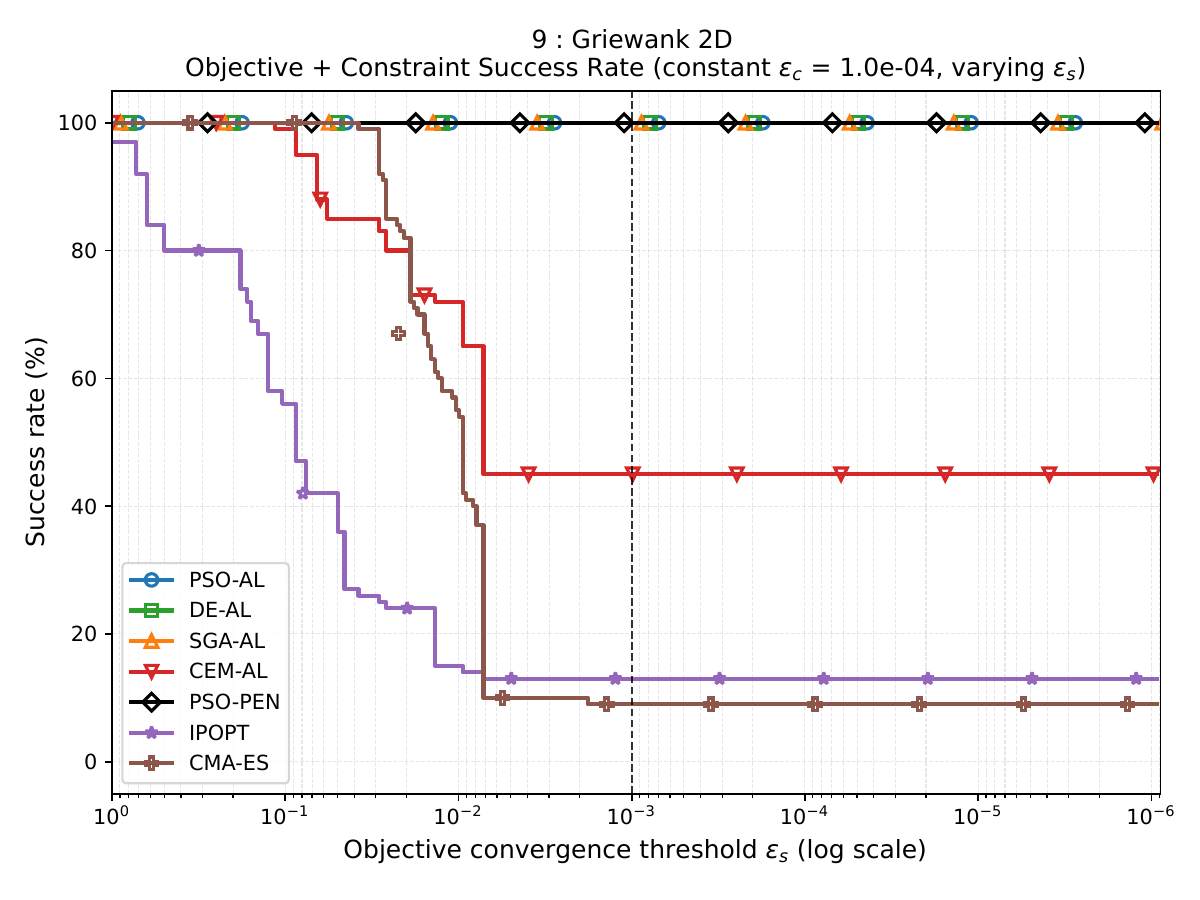}
    \end{subfigure}
    \begin{subfigure}{.49\columnwidth}
        \centering
        \includegraphics[width=\linewidth]{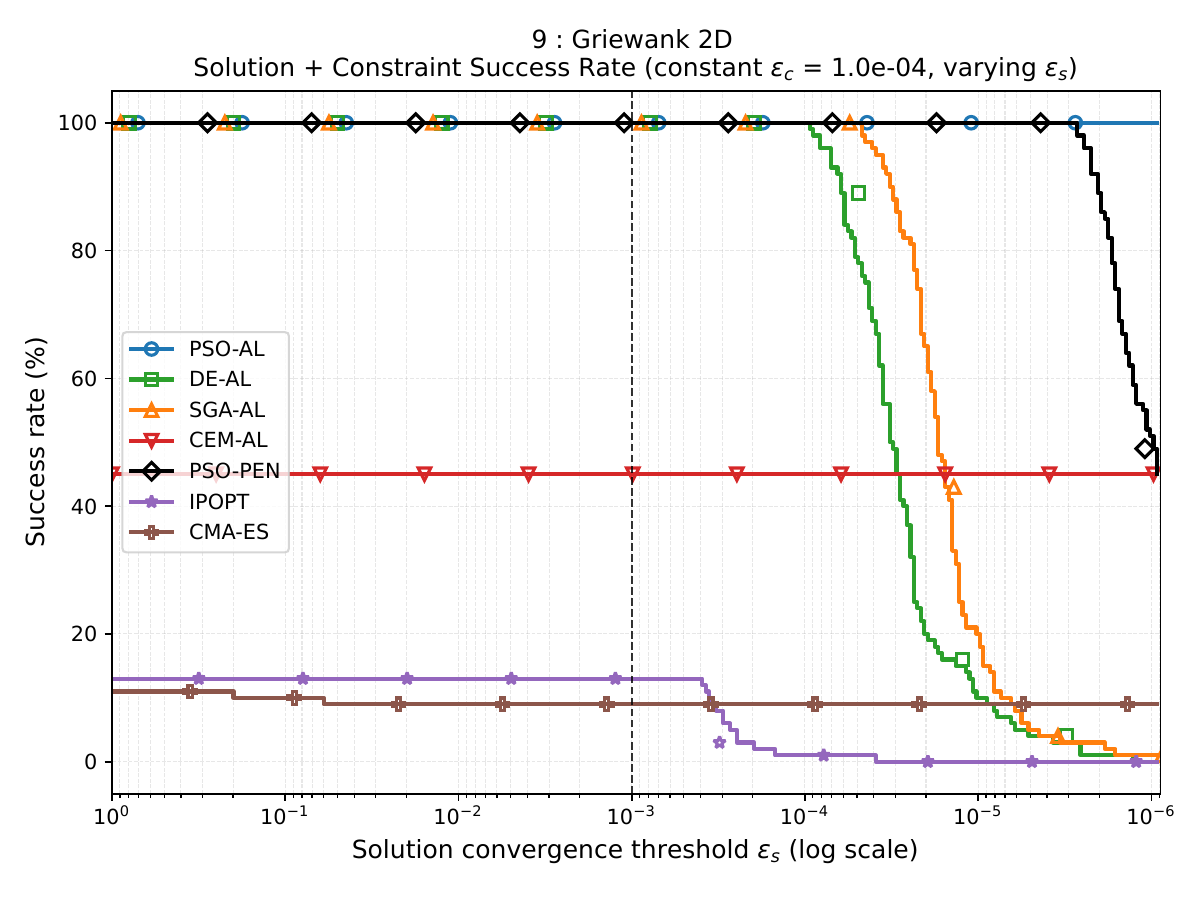}
    \end{subfigure}
    
    \begin{subfigure}{.49\columnwidth}
        \centering
        \includegraphics[width=\linewidth]{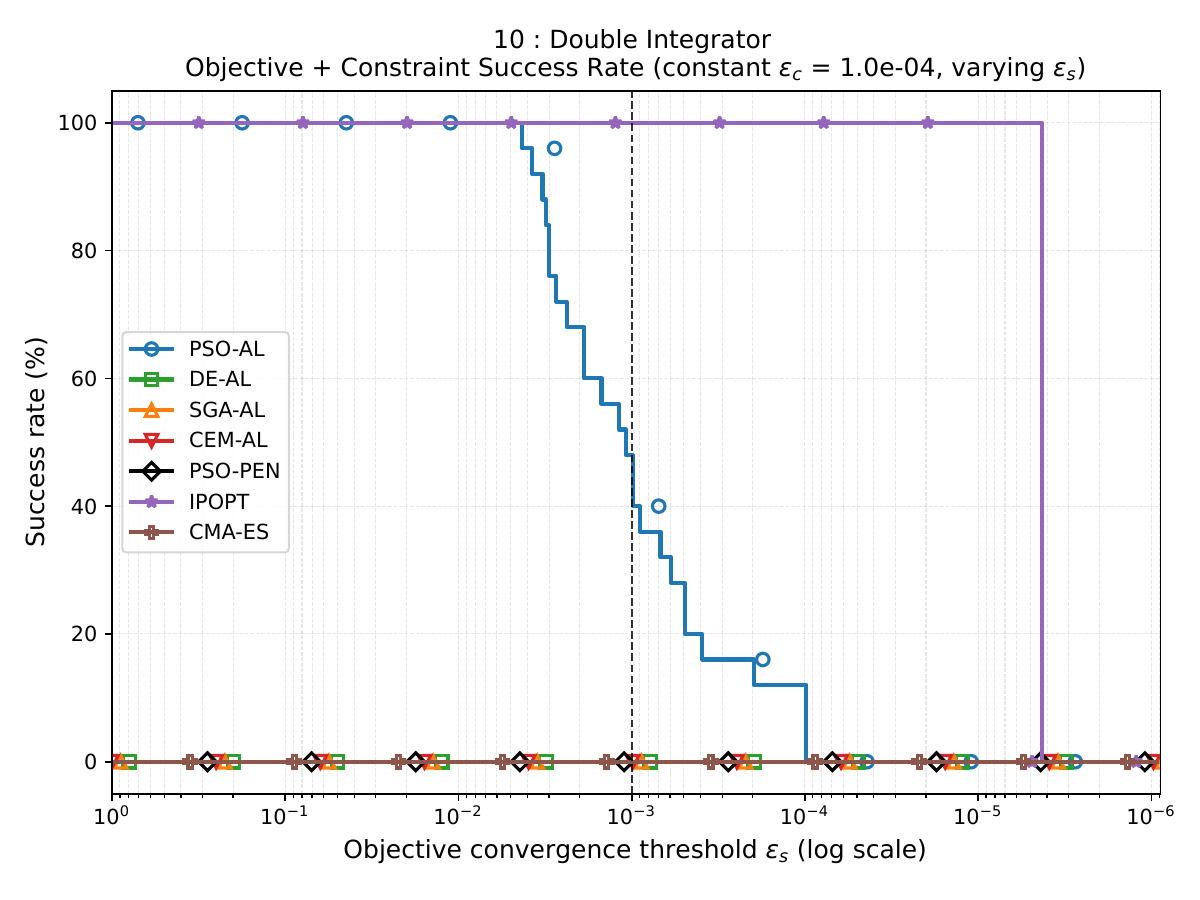}
    \end{subfigure}
    \begin{subfigure}{.49\columnwidth}
        \centering
        \includegraphics[width=\linewidth]{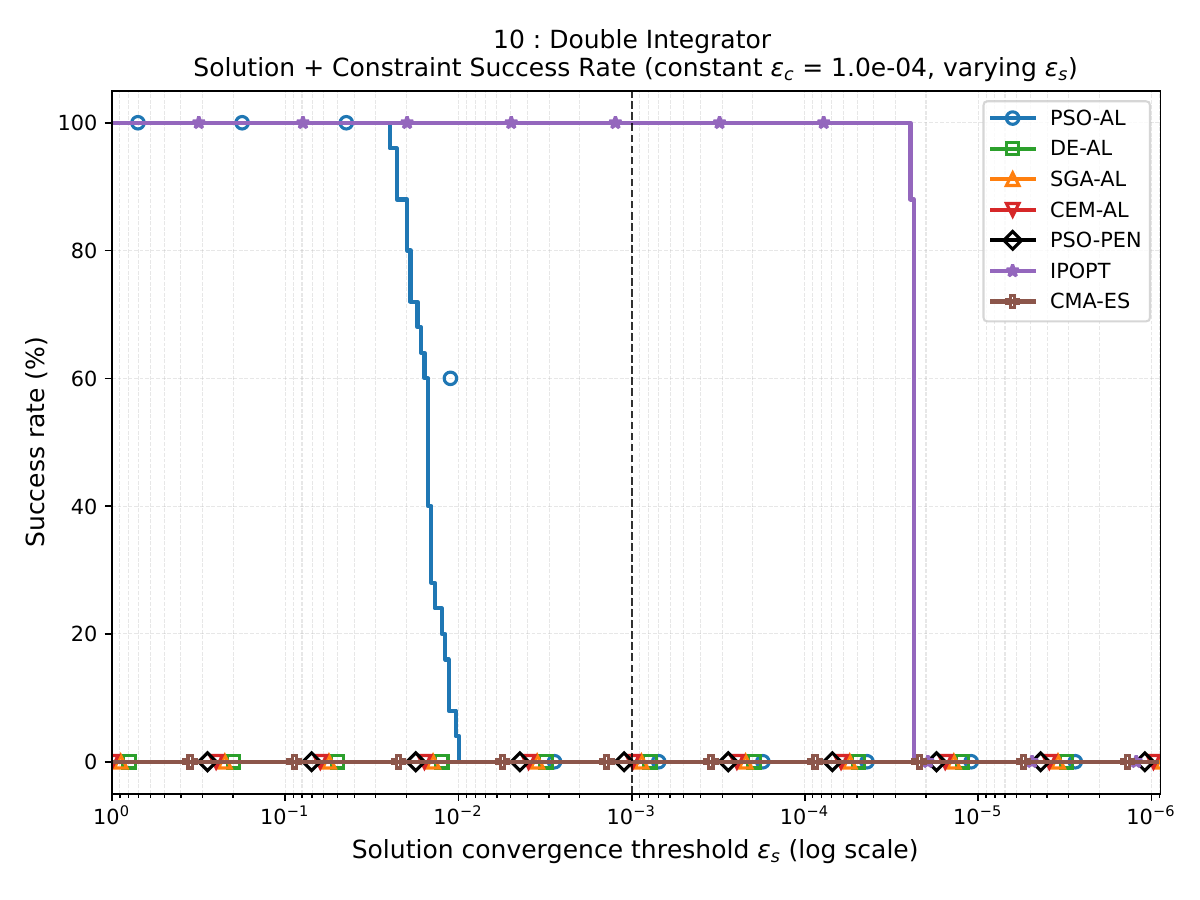}
    \end{subfigure}
    
    \caption{Total comparison success rate curves for Prob.~\ref{prob:rosenbrock},~\ref{prob:quadratic},~\ref{prob:drop_wave},~\ref{prob:griewank},~\ref{prob:double_integrator} (from top to bottom). Left / Right panels show the objective / solution - based success rates, respectively, with fixed constraint violation $\mathcal{V}_{c}(\boldsymbol{x_{f}}) \leq 10^{-4}$. A run is considered successful if the objective / solution - based convergence error is $\lvert f(\boldsymbol{x}_\star) - f(\boldsymbol{x}_f) \rvert \leq \varepsilon_s$ / $\lVert \boldsymbol{x}_\star - \boldsymbol{x}_f \rVert_2 \leq \varepsilon_s$. Each plot presents the success rate curve for each optimizer as a function of the solution threshold $\varepsilon_s$.}
    \label{fig:success_rates_selected_plots}
\end{figure}

Finally, we conduct a more visual total comparison between all tested optimizers, by creating objective and solution - based success rate plots. Fig.~\ref{fig:success_rates_selected_plots} presents the success rates for Prob.~\ref{prob:rosenbrock},~\ref{prob:quadratic},~\ref{prob:drop_wave},~\ref{prob:griewank},~\ref{prob:double_integrator}, while the remaining plots are shown in Appendix~\ref{appendix:success_rate_plots_all} in Fig.~\ref{fig:success_rates_remaining_plots_appendix}. We highlight these particular problems because, taken together, they represent all the major categories of challenging search spaces described in Sec.~\ref{sec:benchmark_problems}, including nonconvex objectives and feasible regions, multimodality, equality-constrained manifolds, and high-dimensionality. The plots summarize optimizer performance across all runs for each COP, showing success rates over a range of solution convergence thresholds $\varepsilon_s$, while maintaining a fixed constraint threshold $\varepsilon_c = 10^{-4}$. Interestingly, despite relying on an objective-based success criterion, the proposed methods attain solution-based convergence in most COPs, indicating that objective convergence typically accompanies convergence of the iterates toward the true solution. The vertical dashed line marks the nominal solution threshold $\varepsilon_s = 10^{-3}$ (the last column of Tab.~\ref{tab:metrics_stats_AL_EAs},~\ref{tab:metrics_stats_diff_optimizers} contains the success rates on this line), considered sufficient for practical convergence, while further improvements beyond $\varepsilon_s \approx 10^{-6}$ are not expected due to the limited precision of the known minima solutions on most benchmark problems (Sec.~\ref{sec:benchmark_problems}).

From the success rate plots, we observe that PSO-AL and DE-AL consistently achieve perfect $100\%$ success rates at the nominal solution threshold $\varepsilon_s = 10^{-3}$ (with the exception of Prob.~\ref{prob:double_integrator}), and remain robust as the tolerance $\varepsilon_s$ decreases. In contrast, SGA-AL and more so CEM-AL exhibit less stable performance, with noticeably lower average success rates, particularly on Prob.~\ref{prob:quadratic} and~\ref{prob:double_integrator}. Regarding the other optimizers, we verify our previous findings that PSO-PEN mainly struggles with high-dimensional problems (e.g. Prob.~\ref{prob:quadratic},~\ref{prob:double_integrator}), just like CMA-ES, while the latter cannot easily converge to the global minima of multimodal problems (such as Prob.~\ref{prob:drop_wave},~\ref{prob:griewank}). IPOPT has generally the worst performance, with success rates below $50\%$ on almost all benchmarks (except for Prob.~\ref{prob:power}, and especially~\ref{prob:double_integrator}, where it outperforms all the other optimizers, since the Double Integrator COP is convex, thus well suited to IPOPT). Interestingly, constraint satisfaction is mostly reliable across all methods and benchmarks (with only PSO-PEN having difficulties satisfying the constraints on Prob.~\ref{prob:quadratic} and~\ref{prob:double_integrator}), indicating that feasibility is not the main limiting factor in these experiments. Instead, the primary differences, as demonstrated in the previous comparisons, lie in solution convergence.

\subsection{Verdict}
Overall, our experimental analysis effectively answers the three experimental questions that we set out to address (Sec.~\ref{sec:experiments}): \textbf{(1)} \methodacronym~is able to solve a wide range of problems with different characteristics while also being able to scale to high-dimensional problems; \textbf{(2)} PSO and DE are the most effective EAs for this AL setting; and \textbf{(3)} \methodacronym~is able to outperform the state-of-the-art IPOPT and CMA-ES optimizers in the complex Prob.~\ref{prob:rosenbrock}-\ref{prob:griewank}, while being competitive, albeit with some limitations, in high-dimensional Prob.~\ref{prob:double_integrator}.

\section{Conclusion} \label{conclusion}
This work demonstrates that the proposed \methodacronym\ framework is a robust and effective approach for constrained, nonlinear, nonconvex, continuous optimization. Across the 10 benchmark COPs, it consistently finds high-quality feasible solutions and, in the AL-EA comparison, PSO-AL and DE-AL emerge as the strongest variants, achieving the smallest errors and near-perfect success rates on almost all problems, with the only notable exception of Prob.~\ref{prob:double_integrator}. By embedding population-based evolutionary algorithms within an Augmented Lagrangian scheme, \methodacronym\ successfully balances feasibility and optimality, even in the presence of multimodality and complex constraint structures.

Among the evolutionary strategies considered, PSO-AL and DE-AL are clearly the most effective unconstrained subproblem solvers, delivering superior solution accuracy, higher success rates, and faster convergence than other evolutionary variants such as SGA-AL and CEM-AL. In particular, these methods typically reach the target precision in a matter of milliseconds, while also reducing constraint violations to near machine precision, despite relying exclusively on derivative-free optimization. The framework remains effective in higher-dimensional settings (Prob.~\ref{prob:double_integrator}), where evolutionary methods alone, without gradient information, are known to fail \cite{chatzilygeroudis2023fast}, nevertheless PSO-AL, functioning inside an AL framework, preserves strong performance. At the same time, the results confirm that reaching machine-precision accuracy without gradient information is especially difficult in high-dimensional problems.

Overall, our findings highlight the advantages of evolutionary methods in navigating nonlinearities, multimodal landscapes, and constrained search spaces, where gradient-based approaches often converge to suboptimal stationary points. The comparison against other optimizers further shows that \methodacronym\ substantially outperforms pure penalty methods like PSO-PEN, which fails on the harder high-dimensional problems, and also outperforms IPOPT and AL-CMA-ES on most benchmarks. IPOPT is fast but unreliable at finding the global minimum, while CMA-ES is more competitive yet still weaker than \methodacronym\ on the multimodal and higher-dimensional problems. Our study not only validates the robustness of \methodacronym\, but also highlights its broader potential for real-world constrained optimization problems. 

Promising directions for future work include exploiting \emph{GPU-based parallelization to improve computational efficiency} \cite{chalumeau2024qdax,lim2023accelerated,hu2019taichi,jakob2022dr} and \emph{combining evolutionary methods with gradient-based methods} \cite{plevris2011hybrid,han2017adaptive,chatzilygeroudis2023fast} to further accelerate convergence and enhance solution accuracy in high-dimensional settings.

\section*{Acknowledgments}
This work was supported by the Hellenic Foundation for Research and Innovation (H.F.R.I.) under the “3rd Call for H.F.R.I. Research Projects to support Post-Doctoral Researchers” (Project Acronym: NOSALRO, Project Number: 7541). The authors have no competing interests.

\bibliographystyle{cas-model2-names}
\bibliography{bibliography}

\appendix

\section{Benchmark Problems Set} \label{appendix:benchmark_problems}

\begin{problem}[: Rosenbrock Function (2D) with Cubic \& Linear Inequality Constraints]
\label{prob:rosenbrock}
\small
\[ \min_{x,y}\ (1 - x)^2 + 100(y - x^2)^2 \]
\[ \text{s.t.:} \quad  g_{1}: (x - 1)^3 - y + 1 \leq 0, \quad g_{2}: x + y - 2 \leq 0, \]
\[ x \in [-1.5, 1.5], \quad y \in [-0.5, 2.5] \]
Global minimum: $f(1, 1) = 0$. \\
Optimal multipliers: $(\mu^{g_{1}}_{\star}, \mu^{g_{2}}_{\star}) = (0, 0)$.
\end{problem}

\begin{problem}[: Quadratic Objective with Quadratic \& Linear Inequality Constraints]
\label{prob:quadratic}
\small
\[ \min_{\boldsymbol{x}}\ -25(x_1 - 2)^2 - (x_2 - 2)^2 - (x_3 - 1)^2 - (x_4 - 4)^2 - (x_5 - 1)^2 - (x_6 - 4)^2 \]
\[ \text{s.t.:} \quad g_{1}: (x_3 - 3)^2 + x_4 \geq 4, \quad g_{2}: (x_5 - 3)^2 + x_6 \geq 4, \]
\[ g_{3}: x_1 - 3x_2 \leq 2, \quad g_{4}: -x_1 + x_2 \leq 2, \]
\[ g_{5}: x_1 + x_2 \leq 6, \quad g_{6}: x_1 + x_2 \geq 2, \]
\[ x_1, x_2 \in [0,100], x_3 \in [1,5], x_4 \in [0,6], x_5 \in [1,5], x_6 \in [0,10] \]
Global minimum: $f(5, 1, 5, 0, 5, 10) = -310$. \\
Optimal multipliers: $(\mu^{g_{1}}_{\star}, \mu^{g_{2}}_{\star}, \mu^{g_{3}}_{\star}, \mu^{g_{4}}_{\star}, \mu^{g_{5}}_{\star}, \mu^{g_{6}}_{\star}) = (0, 0, 38, 0, 112, 0)$.
\end{problem}

\begin{problem}[: Power-Terms Objective with Linear Equality \& Inequality Constraints]
\label{prob:power}
\small
\[ \min_{\boldsymbol{x}}\ x_1^{0.6} + x_2^{0.6} - 6x_1 - 4x_3 + 3x_4 \]
\[ \text{s.t.: } \ h_{1}: x_2 - 3x_1 - 3x_3 = 0, \quad g_{1}: x_1 + 2x_3 \leq 4, \quad g_{2}: x_2 + 2x_4 \leq 4, \]
\[ x_1 \in [0,3], \quad x_2 \in [0,100], \quad x_3 \in [0,100], \quad x_4 \in [0,1] \]
Global minimum: $f(\tfrac{4}{3}, 4, 0, 0) \approx -4.51420$. \\
Optimal multipliers: $(\lambda^{h_{1}}_{\star}, \mu^{g_{1}}_{\star}, \mu^{g_{2}}_{\star}) = (-1.82174, 0, 1.47713)$.
\end{problem}

\begin{problem}[: Linear Objective with Polynomial Inequality Constraints]
\label{prob:linear_polynomial}
\small
\[ \min_{x,y}\ -x - y \]
\[ \text{s.t.: } \ g_{1}: y \leq 2x^4 - 8x^3 + 8x^2 + 2,\ g_{2}: y \leq 4x^4 - 32x^3 + 88x^2 - 96x + 36, \]
\[ x \in [0, 3], \quad y \in [0, 4] \]
Global minimum: $f(2.32952, 3.17849) \approx -5.50801$. \\
Optimal multipliers: $(\mu^{g_{1}}_{\star}, \mu^{g_{2}}_{\star}) \approx (0.28760, 0.71240)$.
\end{problem}

\begin{problem}[: Mishra's Bird Function with Quadratic Inequality Constraint]
\label{prob:mishra_bird}
\small
\[ \min_{x,y}\ \sin(y)e^{(1-\cos x)^2} + \cos(x)e^{(1-\sin y)^2} + (x-y)^2 \]
\[ \text{s.t.: } \quad g_{1}: (x+5)^2+(y+5)^2 \leq 25, \]
\[ x \in [-10, 0], \quad y \in [-6.5, 0] \]
Global minimum: $f(-3.13025, -1.58214) \approx -106.76454$. \\
Optimal multipliers: $\mu^{g_{1}}_{\star} = 0$.
\end{problem}

\begin{problem}[: Constrained Gomez and Levy Function with Trigonometric Inequality Constraint]
\label{prob:gomez_levy}
\small
\[ \min_{x,y}\ 4x^2 - 2.1x^4 + \tfrac{1}{3}x^6 + xy - 4y^2 + 4y^4 \]
\[ \text{s.t.: } \quad g_{1}: -\sin(4\pi x) + 2\sin^2(2\pi y) \leq 1.5, \]
\[ x \in [-1, 0.75], \quad y \in [-1, 1] \]
Global minimum: $f(0.08984, -0.71266) \approx -1.03163$. \\
Optimal multipliers: $\mu^{g_{1}}_{\star} = 0$.
\end{problem}

\begin{problem}[: Drop-Wave Function with Quadratic Equality Constraint]
\label{prob:drop_wave}
\small
\[ \min_{x,y}\ -\frac{1+\cos(12\sqrt{x^2+y^2})}{0.5(x^2+y^2)+2} \]
\[ \text{s.t.: } \quad h_{1}: (x-3)^2+(y-3)^2 = 4, \]
\[ x \in [-10, 10], \quad y \in [-10, 10] \]
Global minimum: $f(2.36699, 1.10282) = f(1.10282, 2.36699) \approx -0.36913$. \\
Optimal multipliers: $\lambda^{h_{1}}_{\star} = 0$.
\end{problem}

\begin{problem}[: Egg-Holder Function with Linear Inequality Constraints]
\label{prob:egg_holder}
\small
\[ \min_{x,y}\ -(y+47)\sin\left(\sqrt{\left|\tfrac{x}{2}+y+47\right|}\right) - x\sin\left(\sqrt{|x-y-47|}\right) \]
\[ \text{s.t.: } \quad g_{1}: 2x + y \geq 700, \quad g_{2}: x - 3y \leq -300, \]
\[ x \in [-512, 512], \quad y \in [-512, 512] \]
Global minimum: $f(512, 404.23180) \approx -959.64066$. \\
Optimal multipliers: $(\mu^{g_{1}}_{\star}, \mu^{g_{2}}_{\star}) = (0, 0)$.
\end{problem}


\begin{problem}[: Griewank Function (2D) with Quadratic Spherical Inequality Constraint]
\label{prob:griewank}
\small
\[ \min_{x,y}\ 1 + \frac{x^2 + y^2}{4000} - \cos\left(x\right) \cos\left(\frac{y}{\sqrt{2}}\right) \]
\[ \text{s.t.: } \quad g_{1}: (x - 10)^2 + (y - 10)^2 \leq 200, \]
\[ x \in [-100, 100], \quad y \in [-100, 100] \]
Global minimum: $f(0, 0) = 0$. \\
Optimal multipliers: $\mu^{g_{1}}_{\star} = 0$.
\end{problem}


\begin{problem}[: 1D Double Integrator Optimal Control Problem (box mass accelerating on a line, under variable controlled force)]
\label{prob:double_integrator}
\small
\[ \min_{\boldsymbol{\mathbf{x}}_{1:K-1},\; \boldsymbol{u}_{0:K-1}}
  \mathcal{J} = \tfrac{1}{2}\sum_{k=0}^{K-1} (\boldsymbol{\mathbf{x}}_k - \boldsymbol{\mathbf{x}}_{target})^\top(\boldsymbol{\mathbf{x}}_k - \boldsymbol{\mathbf{x}}_{target}) + \tfrac{1}{2}\sum_{k=0}^{K-1} 0.1 u_k^2, \]
\[\text{s.t.: } \quad \boldsymbol{\mathbf{x}}_k = \begin{bmatrix} x_k \\ v_k \end{bmatrix}, \quad \boldsymbol{x}_{k+1} = \boldsymbol{x}_k + dt \begin{bmatrix} v_k \\ \tfrac{1}{m} u_k \end{bmatrix}, \quad k = 0,\ldots,K-1,\]
\[\boldsymbol{\mathbf{x}}_{K} = \boldsymbol{\mathbf{x}}_{target}, \quad x_k,\; v_k,\; u_k \in [-10,10].\]
{
\centering
\includegraphics[width=\linewidth]{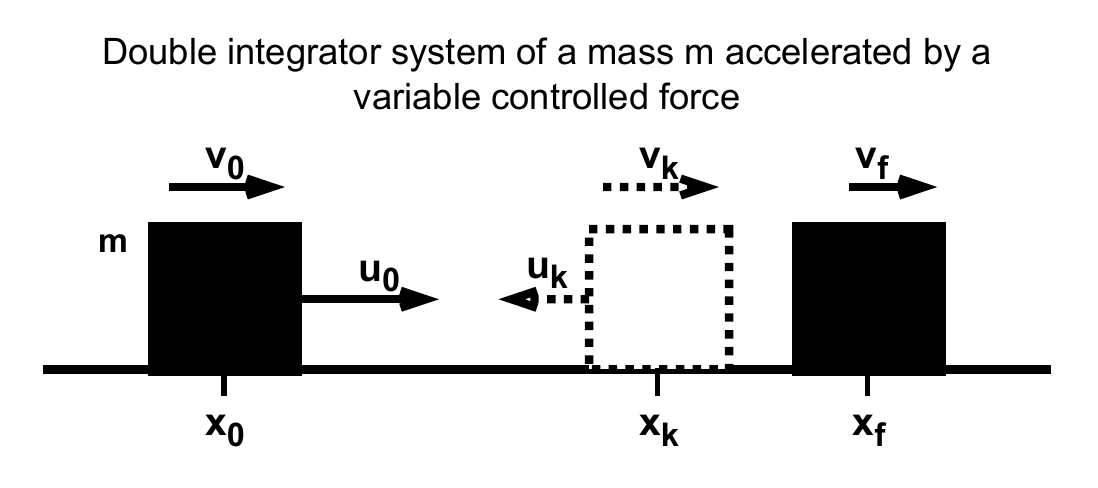}
\par
}
\noindent Here, $\boldsymbol{\mathbf{x}}_k = [x_k, v_k]^\top \in \mathbb{R}^2$ is the state and $\boldsymbol{u}_k \in \mathbb{R}$ is the control input at node k, with known initial state $\boldsymbol{\mathbf{x}}_0 = \boldsymbol{\mathbf{x}}_{start}$ and fixed terminal state $\boldsymbol{\mathbf{x}}_K = \boldsymbol{\mathbf{x}}_{target}$. We set $K=21$ nodes, so we get $61$ decision variables, and $42$ equality constraints. This is a finite LQR problem (with bound constraints), it has one unique global minimizer, and we can compute it via iterative methods like multiple shooting DDP \cite{wensing2023optimization,jallet2025proxddp}. \\
Global minimum: $f(\boldsymbol{x}_{\star}) \approx 23.87240$ ($\boldsymbol{\mathbf{x}}_{start}=[0, 0]^\top$, $\boldsymbol{\mathbf{x}}_{target}=[1, 2]^\top$, the full optimal $\boldsymbol{x}_{\star} \in \mathbb{R}^{61}$ is omitted for brevity) \\
$\boldsymbol{x}_\star \approx \begin{bmatrix} 0,\ 0.18254,\ 0.00913,\ \ldots,\ 2.55968,\ 3.00173,\ 3.52111 \end{bmatrix}^\top$. \\
Optimal multipliers: the full optimal $\boldsymbol{\lambda}_{\star} \in \mathbb{R}^{42}$ is similarly omitted \\
$\boldsymbol{\lambda}_\star \approx \begin{bmatrix} 11.33964,\ -7.30164,\ \ldots,\ 24.29632,\ -7.04222 \end{bmatrix}^\top$.
\end{problem}

\section{Regularity Conditions for the Benchmark Constrained Optimization Problems} \label{appendix:regularity_conditions}

We can establish that every COP in our benchmark set satisfies at least one standard regularity condition, also referred to as a constraint qualification. Specifically, LICQ (Linear Independence Constraint Qualification) holds for all problems, with the exception of the Quadratic problem~\ref{prob:quadratic}, for which we instead verify MFCQ (Mangasarian-Fromovitz Constraint Qualification) in the case of inequality constraints only. For illustration, we provide here the LICQ verification for the Rosenbrock problem~\ref{prob:rosenbrock} and the MFCQ verfication for the Quadratic problem~\ref{prob:quadratic}.

\begin{definition}[LICQ]
We consider the constrained optimization problem \eqref{eq:constrained_optimization_problem}. We let $\mathcal{A_I}(\boldsymbol{x}_\star) = \{i : g_i(\boldsymbol{x}_\star)=0 \}$ denote the active inequality set at the minimizer $\boldsymbol{x}_\star$. LICQ holds at $\boldsymbol{x}_\star$ if the set of vectors below is linearly independent:
\[
    \bigl\{ \nabla_{\boldsymbol{x}} g_i(\boldsymbol{x}_\star) : i \in \mathcal{I}(\boldsymbol{x}_\star) \bigr\}
    \cup
    \bigl\{ \nabla_{\boldsymbol{x}} h_j(\boldsymbol{x}_\star) : j=1,\dots,l \bigr\}
\]
\end{definition}

\begin{definition}[MFCQ for inequality constraints only]
We again consider problem \eqref{eq:constrained_optimization_problem}, but without equality constraints. We let $\mathcal{A_I}(\boldsymbol{x}_\star) = \{i : g_i(\boldsymbol{x}_\star)=0 \}$ be the active inequality set at the minimizer $\boldsymbol{x}_\star$.  
MFCQ holds at $\boldsymbol{x}_\star$ if there exists $\boldsymbol{d} \in \mathbb{R}^n$ such that the directional derivatives below are negative:
\[
    D_{\boldsymbol{d}}g_i(\boldsymbol{x}_\star) = \nabla^\top_{\boldsymbol{x}} g_i(\boldsymbol{x}_\star) \cdot \boldsymbol{d} < 0
    \qquad \text{for all } i \in \mathcal{I}(\boldsymbol{x}_\star)
\]
\end{definition}



\begin{proof} \textbf{\textit{LICQ check for Rosenbrock problem~\ref{prob:rosenbrock}}}

We write the inequality constraints in the standard form $g_i \leq 0$:
\[
    g_1(x,y)=(x-1)^3-y+1\le 0, \quad g_2(x,y)=x+y-2\le 0
\]
At the global minimizer $\boldsymbol{x}_{\star} = \begin{bmatrix} 1 & 1 \end{bmatrix}^\top$ (interior point of the search domain, meaning that the box bounds for the decision variables are inactive at $\boldsymbol{x}_{\star}$), both inequality constraints are active ($g_1(\boldsymbol{x}_{\star})=g_2(\boldsymbol{x}_{\star})=0$). Their gradients at $\boldsymbol{x}_{\star}$ are:
\[
    \nabla g_1(\boldsymbol{x}_{\star}) =
    \begin{bmatrix} 3(x-1)^2 \\ -1
    \end{bmatrix}_{\boldsymbol{x}_{\star}}
    =
    \begin{bmatrix}
    0 \\ -1
    \end{bmatrix},
    \quad
    \nabla g_2(\boldsymbol{x}_{\star})=
    \begin{bmatrix}
    1 \\ 1
    \end{bmatrix}
\]
The equality relation $a \cdot \nabla_{\boldsymbol{x}_{\star}} g_1(\boldsymbol{x}_{\star}) + b \cdot \nabla_{\boldsymbol{x}_{\star}} g_2(\boldsymbol{x}_{\star})=0$ is true only for $a=b=0$. Hence the active gradients are linearly independent, so LICQ holds at the minimizer $\boldsymbol{x}_{\star}$.
\end{proof}

\begin{proof} \textbf{\textit{MFCQ check for Quadratic problem~\ref{prob:quadratic}}}

The global minimizer is at $\boldsymbol{x}_\star = \begin{bmatrix} 5 & 1 & 5 & 0 & 5 & 10 \end{bmatrix}^\top$, where the following constraints $g_{a,1},g_{a,2},\ldots,g_{a,7}$ (written in standard form $g_{a,i} \leq 0$, along with their gradients at $\boldsymbol{x}_\star$) are active, while the box bounds on $x_1$ and $x_2$ are inactive:

\vspace{-10pt}
\[
    \begin{aligned}
        g_{a,1}(\boldsymbol{x}) &= 4-(x_3-3)^2-x_4, &\nabla_{\boldsymbol{x}} g_{a,1} (\boldsymbol{x}_{\star}) &= \begin{bmatrix} 0, 0, -4, -1, 0, 0 \end{bmatrix}^\top\\
        g_{a,2}(\boldsymbol{x}) &= x_1-3x_2-2, &\nabla_{\boldsymbol{x}} g_{a,2} (\boldsymbol{x}_{\star}) &= \begin{bmatrix} 1, -3, 0, 0, 0, 0 \end{bmatrix}^\top\\
        g_{a,3}(\boldsymbol{x}) &= x_1+x_2-6, &\nabla_{\boldsymbol{x}} g_{a,3} (\boldsymbol{x}_{\star}) &= \begin{bmatrix} 1, 1, 0, 0, 0, 0 \end{bmatrix}^\top\\
        g_{a,4}(\boldsymbol{x}) &= x_3-5, &\nabla_{\boldsymbol{x}} g_{a,4} (\boldsymbol{x}_{\star}) &= \begin{bmatrix} 0, 0, 1, 0, 0, 0 \end{bmatrix}^\top\\
        g_{a,5}(\boldsymbol{x}) &= -x_4, &\nabla_{\boldsymbol{x}} g_{a,5} (\boldsymbol{x}_{\star}) &= \begin{bmatrix} 0, 0, 0, -1, 0, 0 \end{bmatrix}^\top\\
        g_{a,6}(\boldsymbol{x}) &= x_5-5, &\nabla_{\boldsymbol{x}} g_{a,6} (\boldsymbol{x}_{\star}) &= \begin{bmatrix} 0, 0, 0, 0, 1, 0 \end{bmatrix}^\top\\
        g_{a,7}(\boldsymbol{x}) &= x_6-10, &\nabla_{\boldsymbol{x}} g_{a,7} (\boldsymbol{x}_{\star}) &= \begin{bmatrix} 0, 0, 0, 0, 0, 1 \end{bmatrix}^\top
    \end{aligned}
\]

\noindent We notice that the active gradients are linearly dependent:
\[
    \nabla_{\boldsymbol{x}} g_{a,1} (\boldsymbol{x}_\star) = -4\,\nabla_{\boldsymbol{x}} g_{a,4} (\boldsymbol{x}_\star) + \nabla_{\boldsymbol{x}} g_{a,5} (\boldsymbol{x}_\star)
\]
Therefore, LICQ does not hold at $\boldsymbol{x}_\star$.

However, we can show that the regularity condition MFCQ is satisfied. We construct the direction:
\[
    \boldsymbol{d} = \begin{bmatrix} -1, 0, -1, 5, -1, -1 \end{bmatrix}^\top
\]
Then, for every active inequality we can verify that:
\[
    \nabla^\top_{\boldsymbol{x}} g_{a,k}(\boldsymbol{x}_\star) \cdot \boldsymbol{d} < 0, \quad k=1,\ldots,7
\]
Hence, there exists a direction that strictly decreases all active inequalities, so MFCQ holds at the minimizer $\boldsymbol{x}_{\star}$.
\end{proof}

\section{Success Rate Plots} \label{appendix:success_rate_plots_all}
In Fig.~\ref{fig:success_rates_remaining_plots_appendix}, we present the remaining set of success rate plots that were not shown in Sec.~\ref{sec:results}, for Prob.~\ref{prob:power},~\ref{prob:linear_polynomial},~\ref{prob:mishra_bird},~\ref{prob:gomez_levy},~\ref{prob:egg_holder}.

\begin{figure}
    \centering
    
    \begin{subfigure}{.49\columnwidth}
        \centering
        \includegraphics[width=\linewidth]{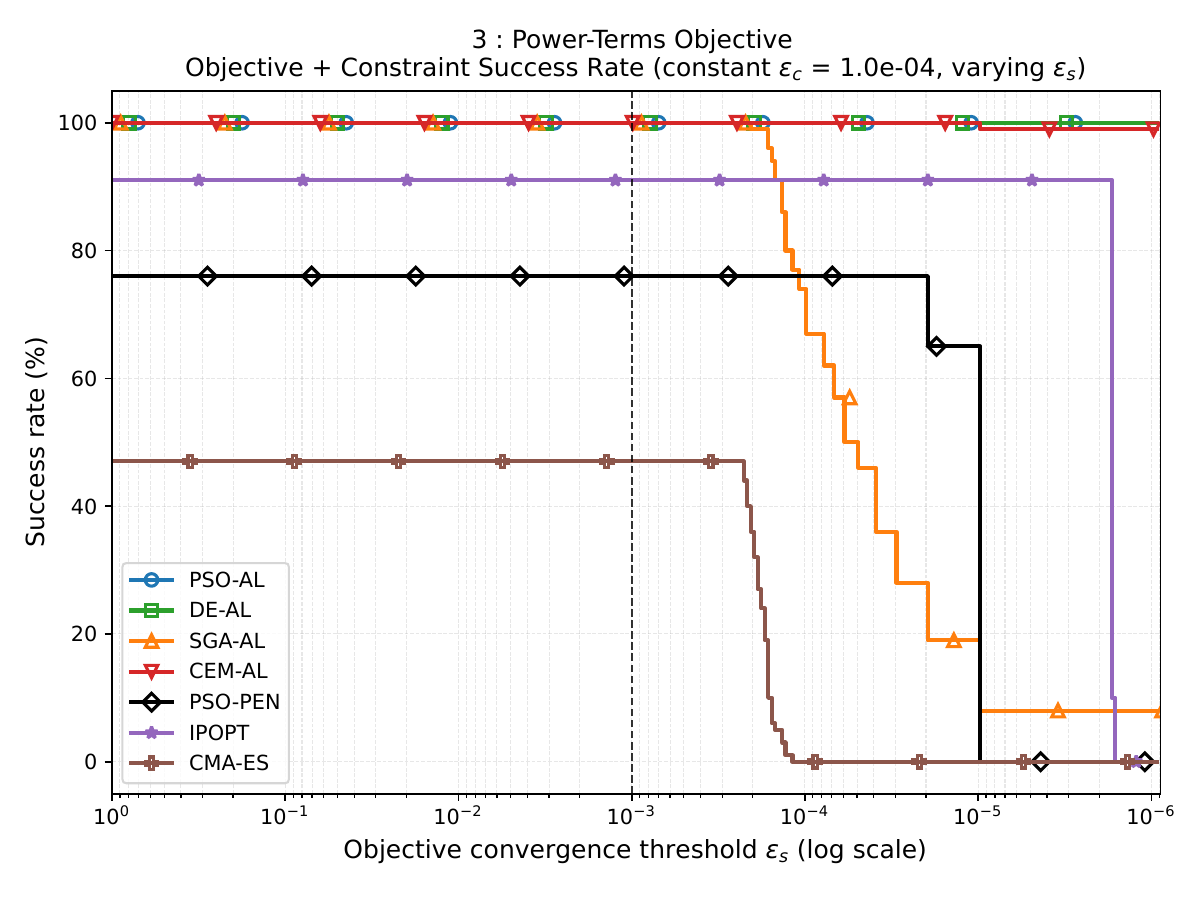}
    \end{subfigure}
    \begin{subfigure}{.49\columnwidth}
        \centering
        \includegraphics[width=\linewidth]{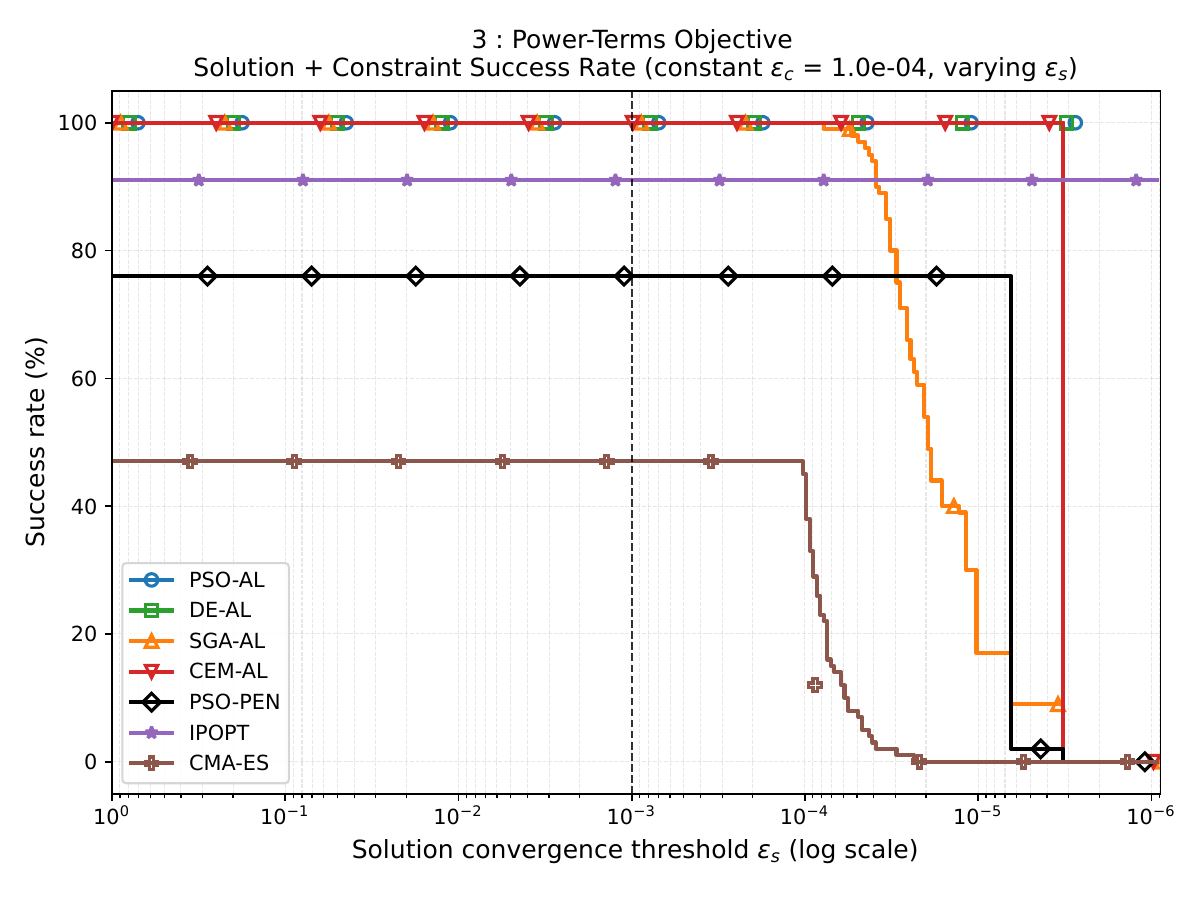}
    \end{subfigure}
    
    \begin{subfigure}{.49\columnwidth}
        \centering
        \includegraphics[width=\linewidth]{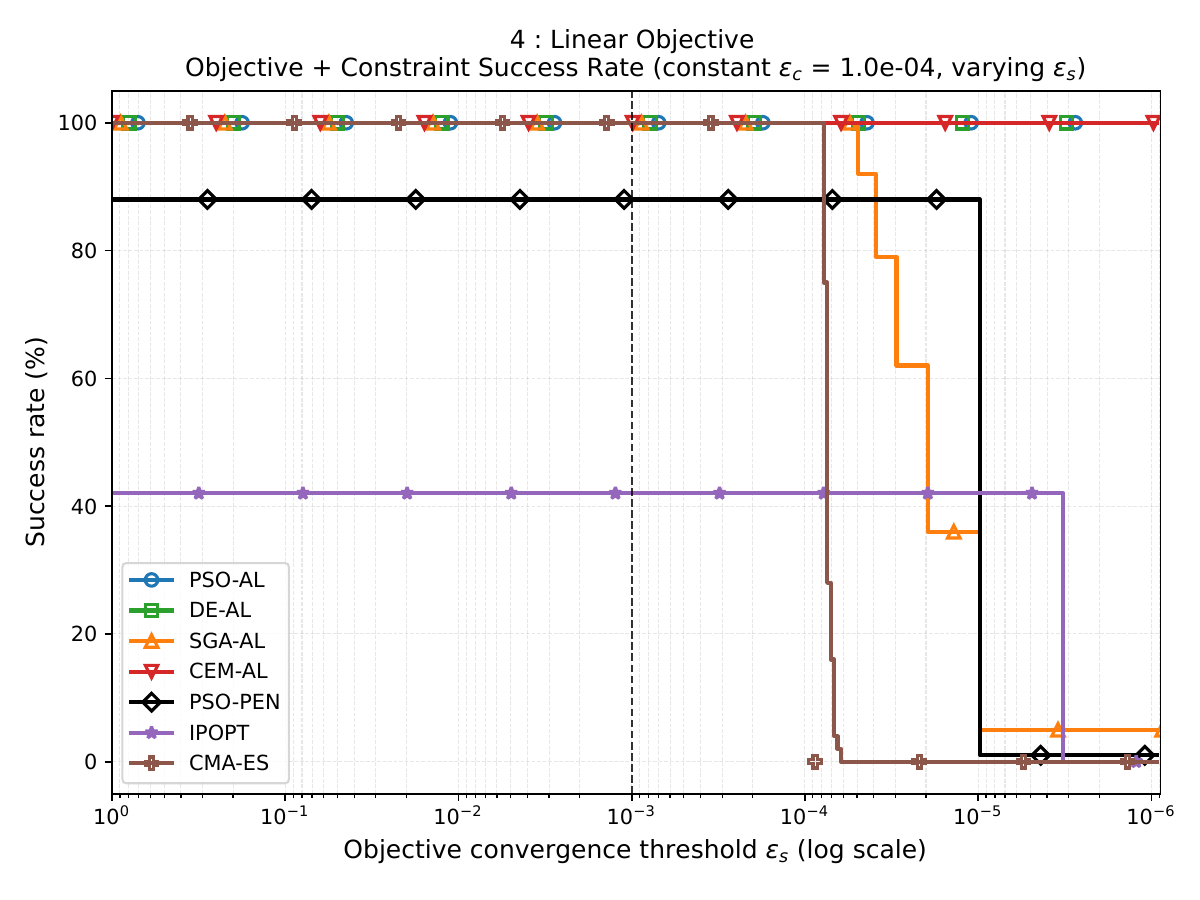}
    \end{subfigure}
    \begin{subfigure}{.49\columnwidth}
        \centering
        \includegraphics[width=\linewidth]{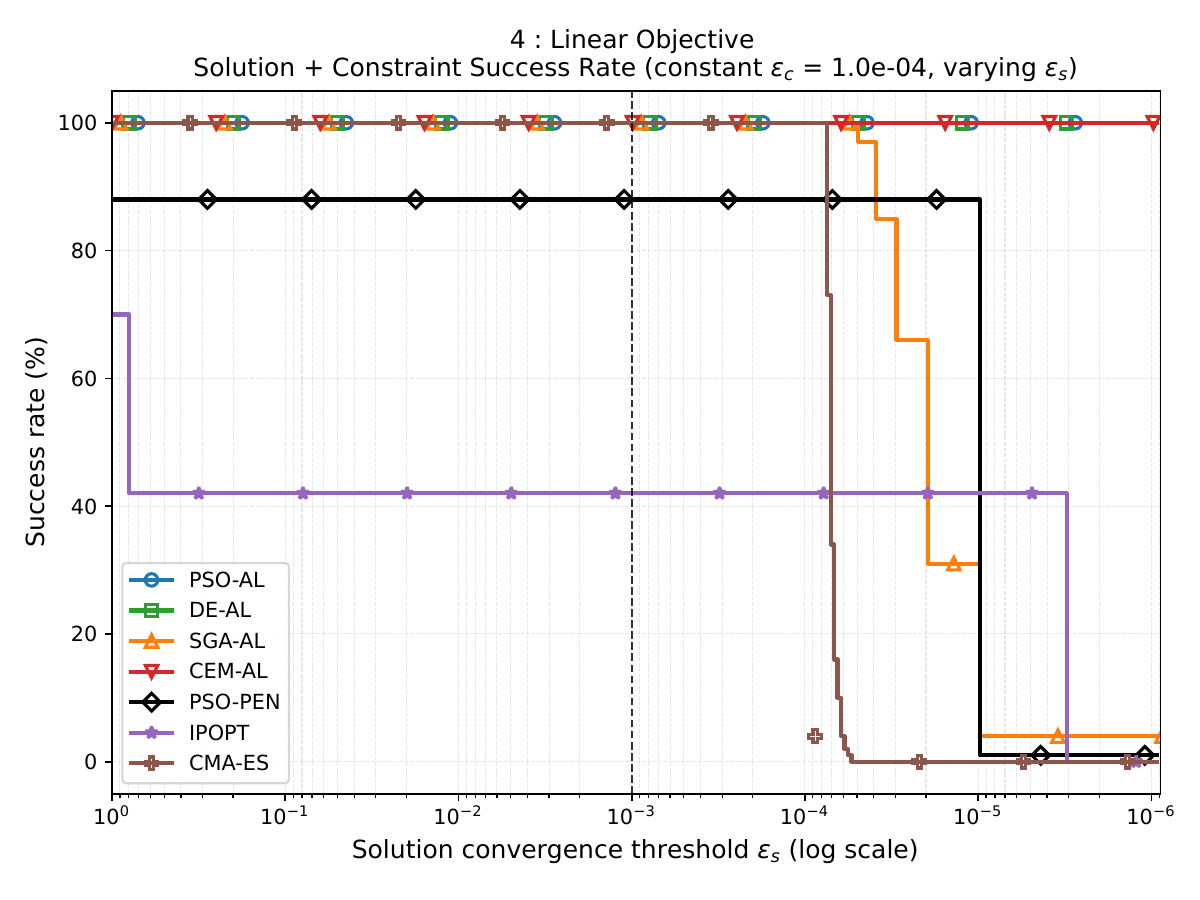}
    \end{subfigure}
    
    \begin{subfigure}{.49\columnwidth}
        \centering
        \includegraphics[width=\linewidth]{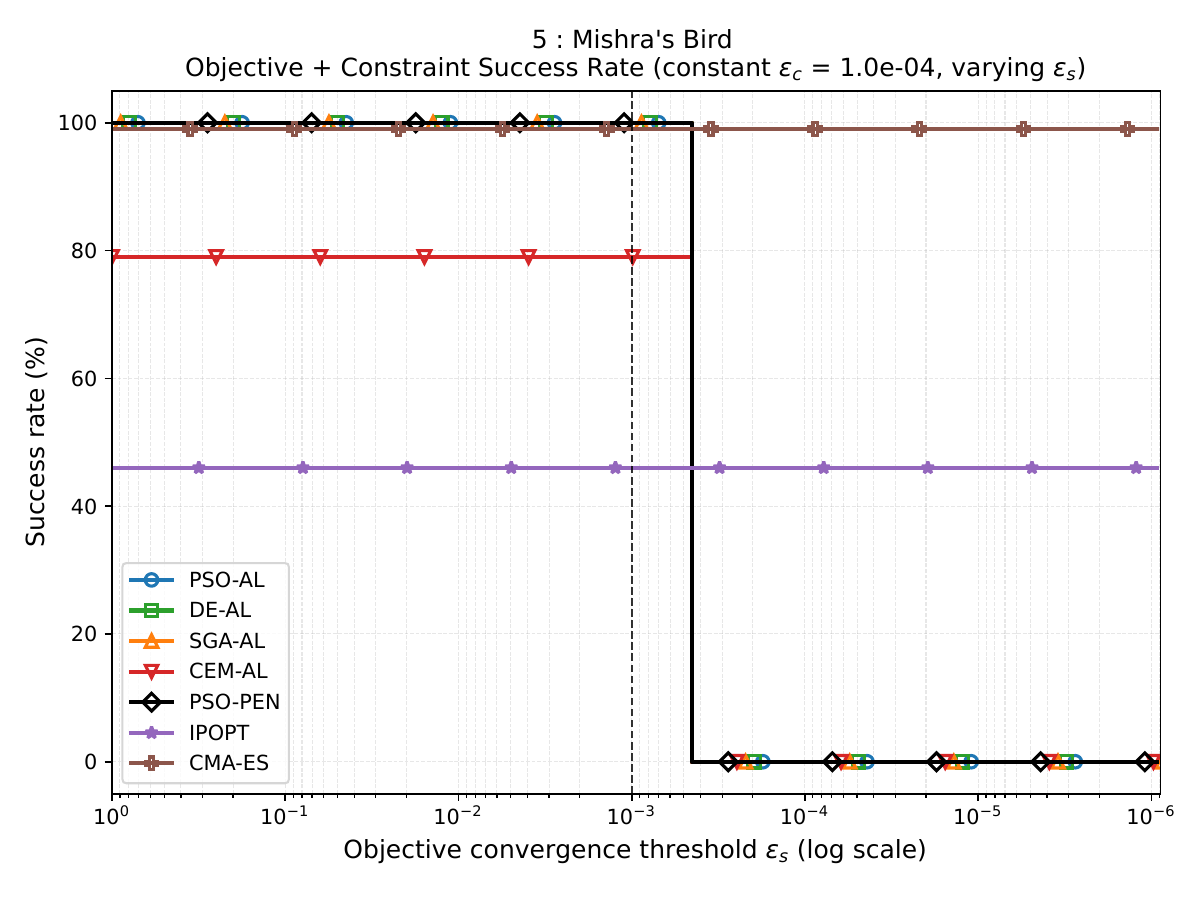}
    \end{subfigure}
    \begin{subfigure}{.49\columnwidth}
        \centering
        \includegraphics[width=\linewidth]{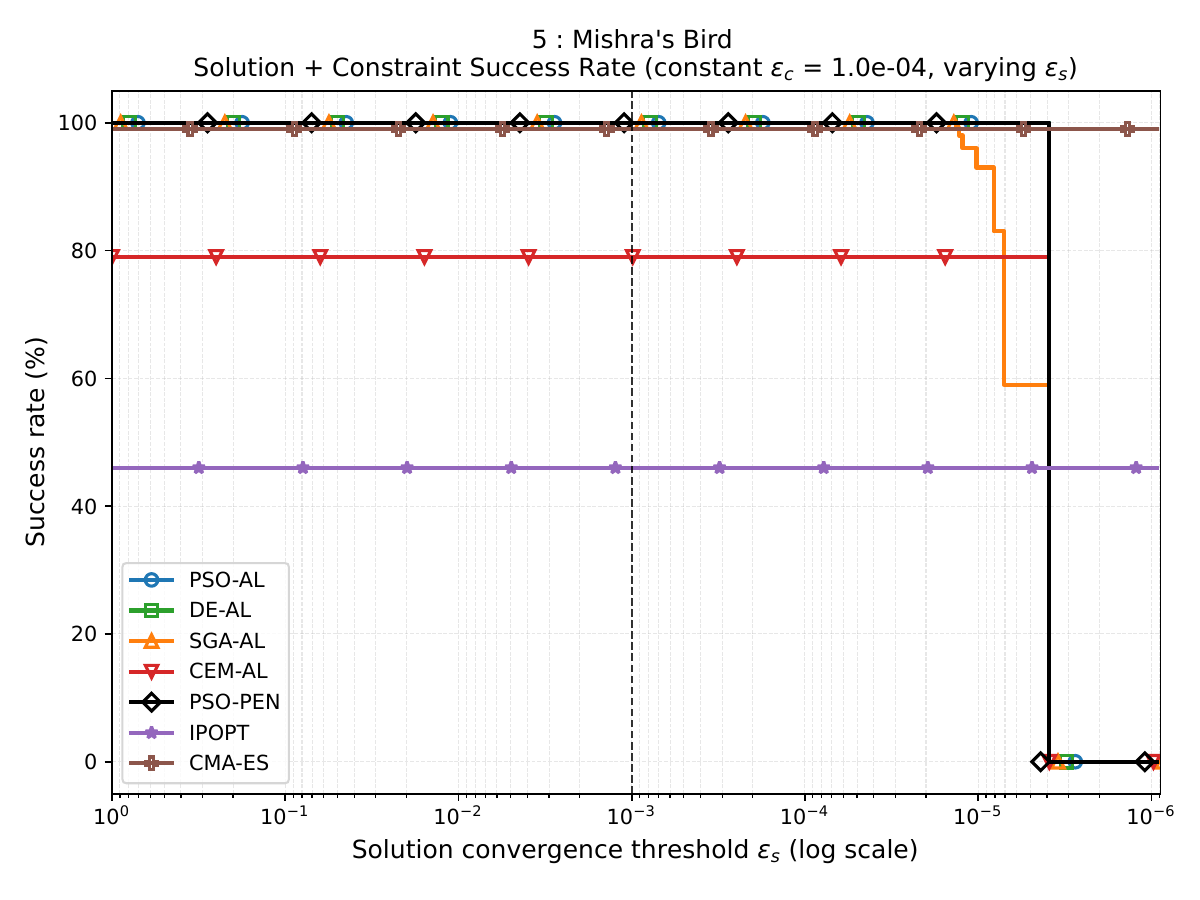}
    \end{subfigure}
    
    \begin{subfigure}{.49\columnwidth}
        \centering
        \includegraphics[width=\linewidth]{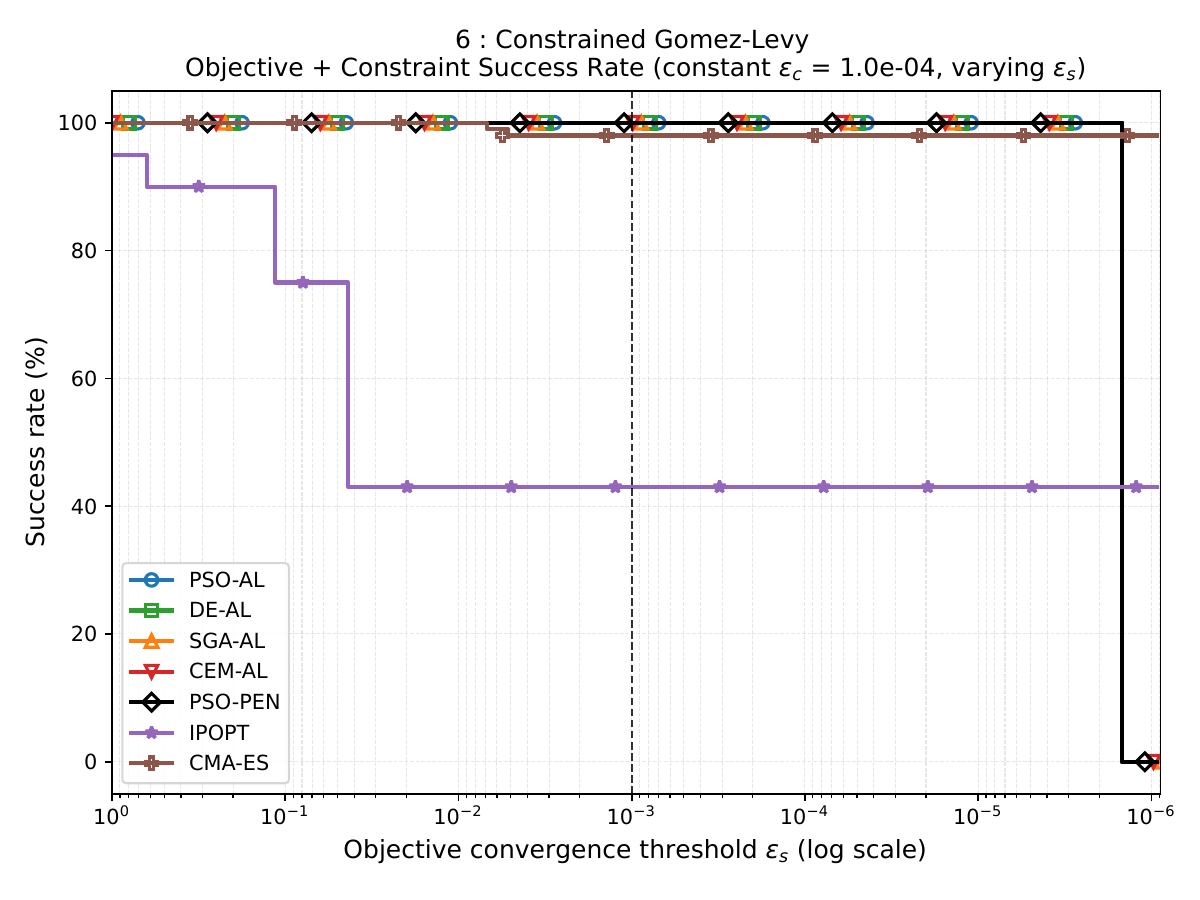}
    \end{subfigure}
    \begin{subfigure}{.49\columnwidth}
        \centering
        \includegraphics[width=\linewidth]{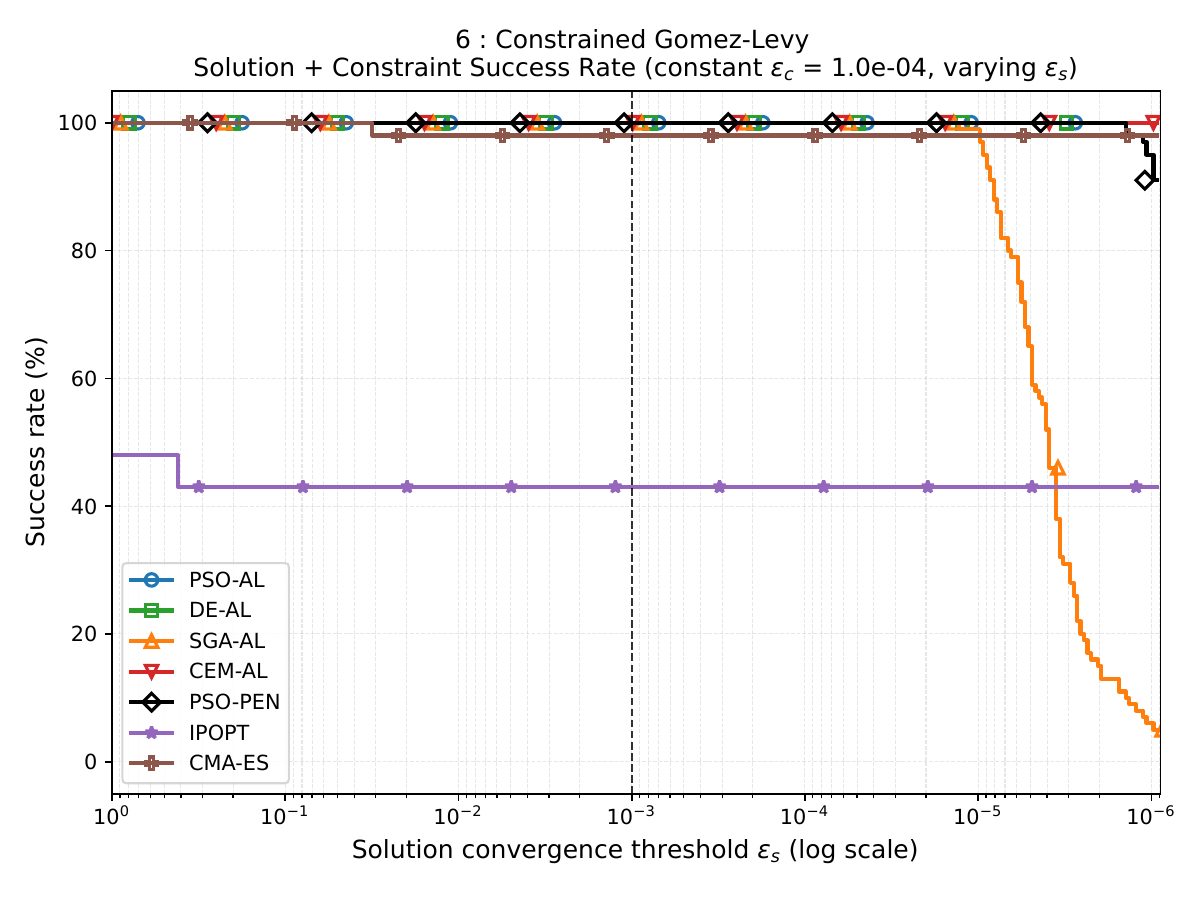}
    \end{subfigure}
    
    \begin{subfigure}{.49\columnwidth}
        \centering
        \includegraphics[width=\linewidth]{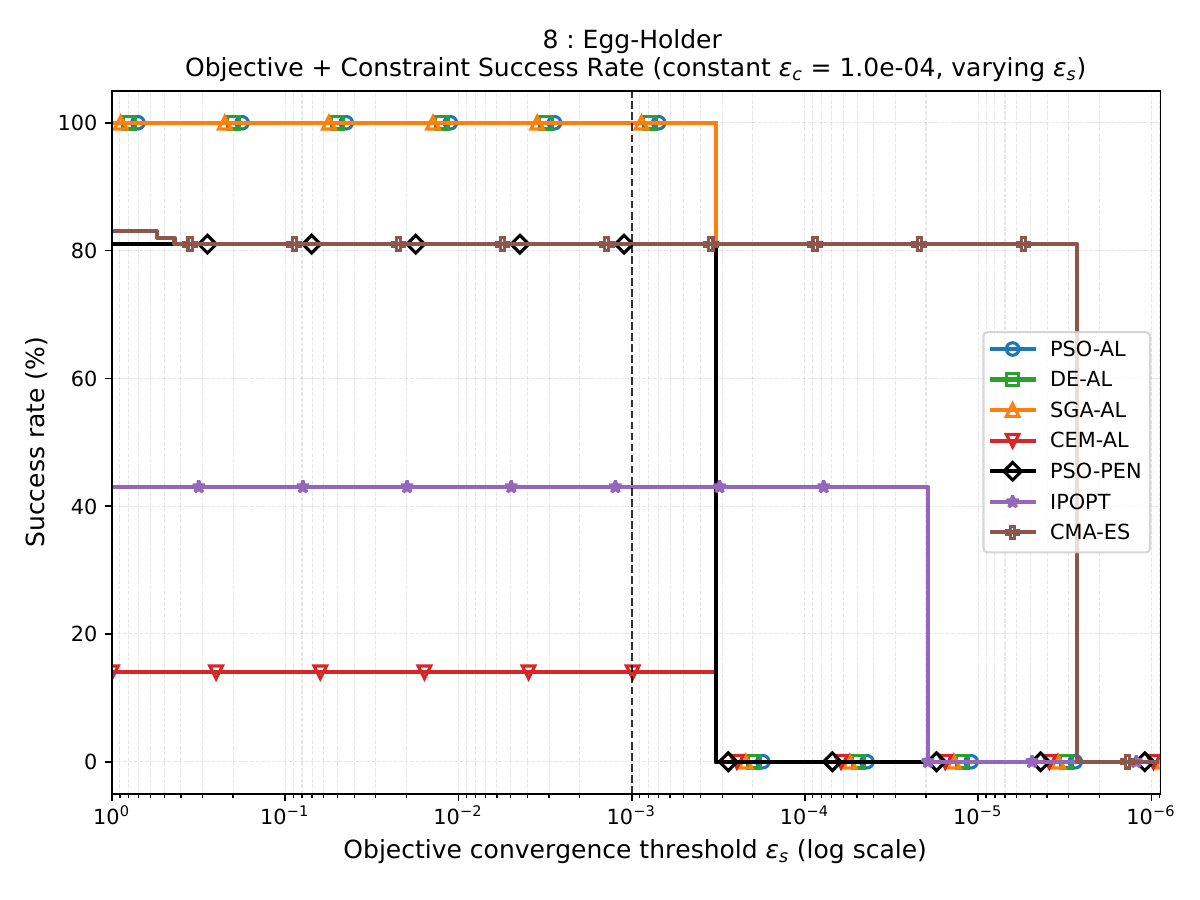}
    \end{subfigure}
    \begin{subfigure}{.49\columnwidth}
        \centering
        \includegraphics[width=\linewidth]{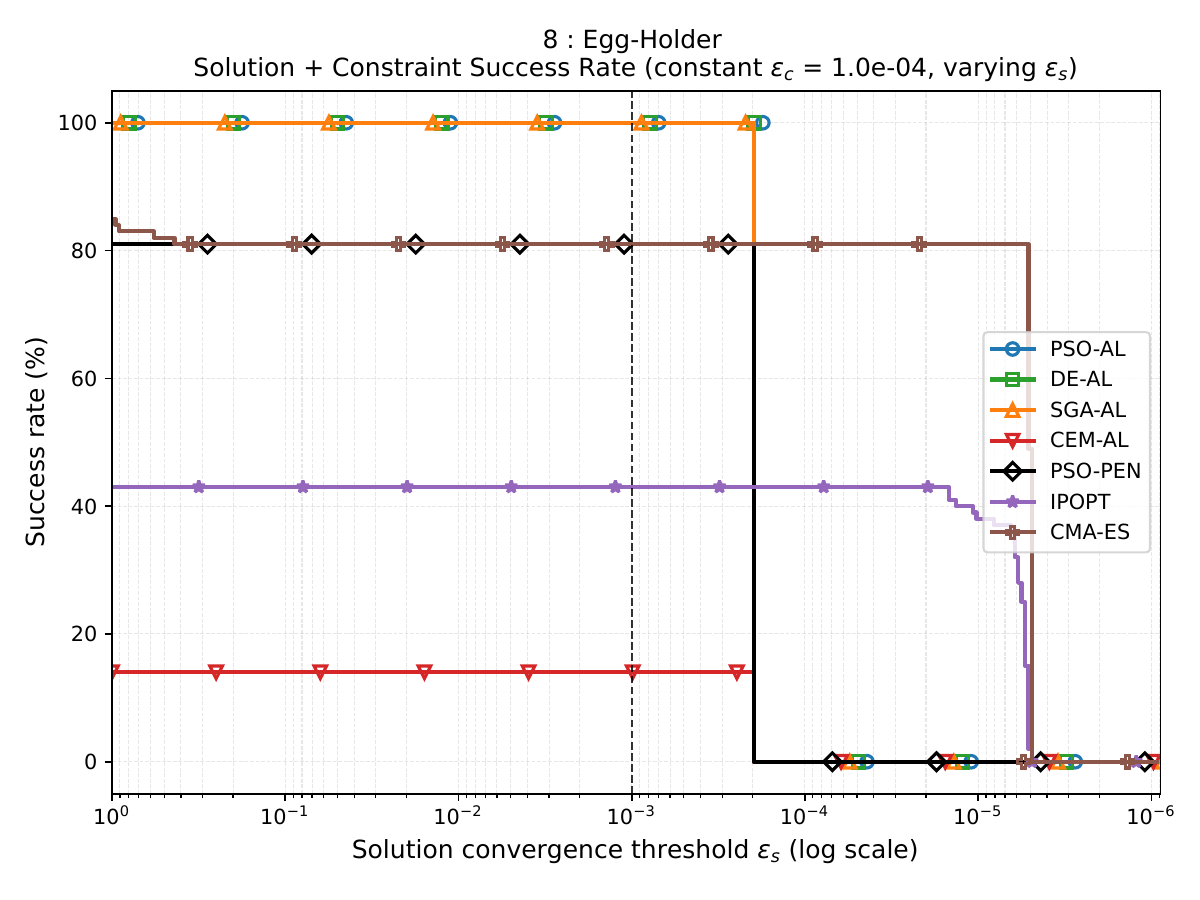}
    \end{subfigure}
    \caption{Total comparison success rate curves for the remaining Prob.~\ref{prob:power},~\ref{prob:linear_polynomial},~\ref{prob:mishra_bird},~\ref{prob:gomez_levy},~\ref{prob:egg_holder} (from top to bottom). Left / Right panels show the objective / solution - based success rates, respectively, with fixed constraint violation $\mathcal{V}_{c}(\boldsymbol{x_{f}}) \leq 10^{-4}$. A run is considered successful if the objective / solution - based convergence error is $\lvert f(\boldsymbol{x}_\star) - f(\boldsymbol{x}_f) \rvert \leq \varepsilon_s$ / $\lVert \boldsymbol{x}_\star - \boldsymbol{x}_f \rVert_2 \leq \varepsilon_s$. Each plot presents the success rate curve for each optimizer as a function of the solution threshold $\varepsilon_s$.}
    \label{fig:success_rates_remaining_plots_appendix}
\end{figure}

\end{document}